\documentclass[10pt]{article} 

\usepackage{etoolbox}
\usepackage{comment}
\newtoggle{neurips}
\togglefalse{neurips}

\newcommand{\neurips}[1]{\iftoggle{neurips}{#1}{}}
\newcommand{\arxiv}[1]{\iftoggle{neurips}{}{#1}}

\arxiv{
\usepackage[letterpaper, left=.95in, right=.95in, top=.95in,bottom=.95in]{geometry}
  \usepackage{parskip}
  \usepackage[colorlinks=true, linkcolor=blue!70!black, citecolor=blue!70!black,urlcolor=black,breaklinks=true]{hyperref}
  \usepackage[dvipsnames]{xcolor}
}

\neurips{
  \usepackage{parskip}
    \usepackage[colorlinks=true, linkcolor=blue!70!black, citecolor=blue!70!black,urlcolor=black,breaklinks=true]{hyperref}
}

\PassOptionsToPackage{hypertexnames=false}{hyperref}  

\usepackage{amsmath}
\usepackage{microtype}
\usepackage{hhline}

\usepackage{amsthm}
\usepackage{mathtools}
\usepackage{bbm}
\usepackage{amsfonts}
\usepackage{amssymb}
\usepackage[nameinlink,capitalize]{cleveref}

\usepackage{xargs}

\makeatletter
\newcommand{\neutralize}[1]{\expandafter\let\csname c@#1\endcsname\count@}
\makeatother

\usepackage{algorithm}

\arxiv{
\usepackage{natbib}
\bibliographystyle{plainnat}
\bibpunct{(}{)}{;}{a}{,}{,}
}

\usepackage{xpatch}

\usepackage{thm-restate}
\declaretheorem[name=Theorem,parent=section]{theorem}
\declaretheorem[name=Lemma,parent=section]{lemma}
\declaretheorem[name=Assumption, parent=section]{assumption}
\declaretheorem[name=Definition, parent=section]{definition}
\declaretheorem[name=Condition, parent=section]{condition}
\declaretheorem[name=Corollary, parent=section]{corollary}

\declaretheorem[name=Proposition, parent=section]{proposition}

\renewcommand{\eqref}[1]{\texorpdfstring{\hyperref[#1]{(\ref*{#1})}}{(\ref*{#1})}}
\crefformat{equation}{#2Eq.\,(#1)#3}
\Crefformat{equation}{#2Eq.\,(#1)#3}
\Crefformat{figure}{#2Figure~#1#3}
\Crefformat{assumption}{#2Assumption~#1#3}
\Crefname{assumption}{Assumption}{Assumptions}
\crefname{fact}{Fact}{Facts}
\Crefformat{figure}{#2Figure #1#3}
\Crefformat{assumption}{#2Assumption #1#3}

\usepackage{crossreftools}
\makeatletter
\renewenvironment{proof}[1][Proof]%
{%
  \par\noindent{\bfseries\upshape {#1.}\ }%
}%
{\qed\newline}
\makeatother

\xpatchcmd{\proof}{\itshape}{\normalfont\proofnameformat}{}{}
\newcommand{\proofnameformat}{\bfseries}

\usepackage[most]{tcolorbox}

\tcbset {
  base/.style={
    arc=0mm, 
    bottomtitle=0.5mm,
    boxrule=0mm,
    colbacktitle=!10!white, 
    coltitle=black, 
    fonttitle=\bfseries, 
    left=2.5mm,
    leftrule=1mm,
    right=3.5mm,
    title={#1},
    toptitle=0.75mm, 
  }
}

\newtcolorbox{mainbox}[1]{
  colframe=blue!10!black,
  colbacktitle=blue!50!black!30!white,
  colback=blue!2!white,
  enhanced,
  fonttitle=\bfseries,
  attach boxed title to top left={yshift=-2.5mm},
  boxed title style={size=small,colframe=blue!40!black,colback=blue!40!black},
  title={\small\textcolor{white}{\textsc{#1}}}
}

\newtcolorbox{minbox}[1]{
  colframe=blue!10!black,
  colbacktitle=blue!50!black!30!white,
  colback=blue!2!white,
  enhanced,
  fonttitle=\bfseries,
}

\usepackage[utf8]{inputenc} 
\usepackage[T1]{fontenc}    
\usepackage{url}            
\usepackage{booktabs}       
\usepackage{amsfonts}       
\usepackage{nicefrac}       
\usepackage{microtype}      
\usepackage{makecell}
\usepackage{enumitem}
\usepackage{breakcites}
\usepackage{mathrsfs}
\usepackage{booktabs,longtable}
\usepackage[normalem]{ulem}

\usepackage{algorithm}
\usepackage{verbatim}
\usepackage[noend]{algpseudocode}

\makeatletter
\let\OldStatex\Statex
\renewcommand{\Statex}[1][3]{%
  \setlength\@tempdima{\algorithmicindent}%
  \OldStatex\hskip\dimexpr#1\@tempdima\relax}
\makeatother

\usepackage{multicol}
\usepackage{colortbl}
\usepackage{setspace}
\usepackage{transparent}
\usepackage{upgreek}

\usepackage{inconsolata}
\usepackage[scaled=.90]{helvet}
\usepackage{xspace}

\usepackage{graphicx}
\neurips{\usepackage{subfigure}}

\usepackage{amsmath,amssymb}
\usetikzlibrary{arrows.meta,calc}

\DeclarePairedDelimiter{\brk}{[}{]}
\DeclarePairedDelimiter{\crl}{\{}{\}}
\DeclarePairedDelimiter{\prn}{(}{)}

\let\Pr\undefined

\DeclareMathOperator{\Pr}{Pr}

\DeclareMathOperator*{\minimize}{minimize} 

\DeclareMathOperator*{\argmin}{arg\,min} 

\def\ddefloop#1{\ifx\ddefloop#1\else\ddef{#1}\expandafter\ddefloop\fi}
\def\ddef#1{\expandafter\def\csname bb#1\endcsname{\ensuremath{\mathbb{#1}}}}
\ddefloop ABCDEFGHIJKLMNOPQRSTUVWXYZ\ddefloop
\def\ddefloop#1{\ifx\ddefloop#1\else\ddef{#1}\expandafter\ddefloop\fi}
\def\ddef#1{\expandafter\def\csname b#1\endcsname{\ensuremath{\mathbf{#1}}}}
\ddefloop ABCDEFGHIJKLMNOPQRSTUVWXYZ\ddefloop
\def\ddef#1{\expandafter\def\csname sf#1\endcsname{\ensuremath{\mathsf{#1}}}}
\ddefloop ABCDEFGHIJKLMNOPQRSTUVWXYZ\ddefloop
\def\ddef#1{\expandafter\def\csname c#1\endcsname{\ensuremath{\mathcal{#1}}}}
\ddefloop ABCDEFGHIJKLMNOPQRSTUVWXYZ\ddefloop
\def\ddef#1{\expandafter\def\csname h#1\endcsname{\ensuremath{\widehat{#1}}}}
\ddefloop ABCDEFGHIJKLMNOPQRSTUVWXYZ\ddefloop
\def\ddef#1{\expandafter\def\csname hc#1\endcsname{\ensuremath{\widehat{\mathcal{#1}}}}}
\ddefloop ABCDEFGHIJKLMNOPQRSTUVWXYZ\ddefloop
\def\ddef#1{\expandafter\def\csname t#1\endcsname{\ensuremath{\widetilde{#1}}}}
\ddefloop ABCDEFGHIJKLMNOPQRSTUVWXYZ\ddefloop
\def\ddef#1{\expandafter\def\csname tc#1\endcsname{\ensuremath{\widetilde{\mathcal{#1}}}}}
\ddefloop ABCDEFGHIJKLMNOPQRSTUVWXYZ\ddefloop
\def\ddefloop#1{\ifx\ddefloop#1\else\ddef{#1}\expandafter\ddefloop\fi}
\def\ddef#1{\expandafter\def\csname scr#1\endcsname{\ensuremath{\mathscr{#1}}}}
\ddefloop ABCDEFGHIJKLMNOPQRSTUVWXYZ\ddefloop

\newcommand{\pipre}{\pi_{\mathrm{pre}}}
\newcommand{\pisft}{\pi_{\mathrm{SFT}}}
\newcommand{\pidata}{\pi_{\mathrm{D}}}
\newcommand{\Xcal}{\mathcal{X}}
\newcommand{\Ycal}{\mathcal{Y}}
\newcommand{\Pcov}{\mathrm{Cov}}
\newcommand{\passat}[1]{pass@$#1$}
\newcommand{\tailsft}{\textsc{TailSFT}}
\let\underbar\undefined

\makeatletter
\let\save@mathaccent\mathaccent
\newcommand*\if@single[3]{%
  \setbox0\hbox{${\mathaccent"0362{#1}}^H$}%
  \setbox2\hbox{${\mathaccent"0362{\kern0pt#1}}^H$}%
  \ifdim\ht0=\ht2 #3\else #2\fi
  }
\newcommand*\rel@kern[1]{\kern#1\dimexpr\macc@kerna}
\newcommand*\widebar[1]{\@ifnextchar^{{\wide@bar{#1}{0}}}{\wide@bar{#1}{1}}}
\newcommand*\underbar[1]{\@ifnextchar_{{\under@bar{#1}{0}}}{\under@bar{#1}{1}}}
\newcommand*\wide@bar[2]{\if@single{#1}{\wide@bar@{#1}{#2}{1}}{\wide@bar@{#1}{#2}{2}}}
\newcommand*\under@bar[2]{\if@single{#1}{\under@bar@{#1}{#2}{1}}{\under@bar@{#1}{#2}{2}}}
\newcommand*\wide@bar@[3]{%
  \begingroup
  \def\mathaccent##1##2{%
    \let\mathaccent\save@mathaccent
    \if#32 \let\macc@nucleus\first@char \fi
    \setbox\z@\hbox{$\macc@style{\macc@nucleus}_{}$}%
    \setbox\tw@\hbox{$\macc@style{\macc@nucleus}{}_{}$}%
    \dimen@\wd\tw@
    \advance\dimen@-\wd\z@
    \divide\dimen@ 3
    \@tempdima\wd\tw@
    \advance\@tempdima-\scriptspace
    \divide\@tempdima 10
    \advance\dimen@-\@tempdima
    \ifdim\dimen@>\z@ \dimen@0pt\fi
    \rel@kern{0.6}\kern-\dimen@
    \if#31
      \overline{\rel@kern{-0.6}\kern\dimen@\macc@nucleus\rel@kern{0.4}\kern\dimen@}%
      \advance\dimen@0.4\dimexpr\macc@kerna
      \let\final@kern#2%
      \ifdim\dimen@<\z@ \let\final@kern1\fi
      \if\final@kern1 \kern-\dimen@\fi
    \else
      \overline{\rel@kern{-0.6}\kern\dimen@#1}%
    \fi
  }%
  \macc@depth\@ne
  \let\math@bgroup\@empty \let\math@egroup\macc@set@skewchar
  \mathsurround\z@ \frozen@everymath{\mathgroup\macc@group\relax}%
  \macc@set@skewchar\relax
  \let\mathaccentV\macc@nested@a
  \if#31
    \macc@nested@a\relax111{#1}%
  \else
    \def\gobble@till@marker##1\endmarker{}%
    \futurelet\first@char\gobble@till@marker#1\endmarker
    \ifcat\noexpand\first@char A\else
      \def\first@char{}%
    \fi
    \macc@nested@a\relax111{\first@char}%
  \fi
  \endgroup
}
\newcommand*\under@bar@[3]{%
  \begingroup
  \def\mathaccent##1##2{%
    \let\mathaccent\save@mathaccent
    \if#32 \let\macc@nucleus\first@char \fi
    \setbox\z@\hbox{$\macc@style{\macc@nucleus}_{}$}%
    \setbox\tw@\hbox{$\macc@style{\macc@nucleus}{}_{}$}%
    \dimen@\wd\tw@
    \advance\dimen@-\wd\z@
    \divide\dimen@ 3
    \@tempdima\wd\tw@
    \advance\@tempdima-\scriptspace
    \divide\@tempdima 10
    \advance\dimen@-\@tempdima
    \ifdim\dimen@>\z@ \dimen@0pt\fi
    \rel@kern{0.6}\kern-\dimen@
    \if#31
      \underline{\rel@kern{-0.6}\kern\dimen@\macc@nucleus\rel@kern{0.4}\kern\dimen@}%
      \advance\dimen@0.4\dimexpr\macc@kerna
      \let\final@kern#2%
      \ifdim\dimen@<\z@ \let\final@kern1\fi
      \if\final@kern1 \kern-\dimen@\fi
    \else
      \underline{\rel@kern{-0.6}\kern\dimen@#1}%
    \fi
  }%
  \macc@depth\@ne
  \let\math@bgroup\@empty \let\math@egroup\macc@set@skewchar
  \mathsurround\z@ \frozen@everymath{\mathgroup\macc@group\relax}%
  \macc@set@skewchar\relax
  \let\mathaccentV\macc@nested@a
  \if#31
    \macc@nested@a\relax111{#1}%
  \else
    \def\gobble@till@marker##1\endmarker{}%
    \futurelet\first@char\gobble@till@marker#1\endmarker
    \ifcat\noexpand\first@char A\else
      \def\first@char{}%
    \fi
    \macc@nested@a\relax111{\first@char}%
  \fi
  \endgroup
}
\makeatother

\defcitealias{olmo2026olmo3}{Olmo, 2026}
\defcitealias{team2026kimi}{Kimi, 2026}
\defcitealias{microsoft2025mai}{MAI, 2025}

 \usepackage{color-edits}
 \addauthor{df}{ForestGreen}
 \addauthor{ja}{red}
 \addauthor{ak}{BurntOrange}
 \addauthor{sm}{purple}
 \addauthor{sj}{cyan}

\let\oldparagraph\paragraph

\renewcommand{\paragraph}[1]{\oldparagraph{#1.}}

\makeatletter
\g@addto@macro\appendix{%
  \crefalias{section}{appendixsection}%
  \crefalias{subsection}{appendixsubsection}%
  \crefalias{subsubsection}{appendixsubsubsection}%
}
\makeatother
\crefname{appendixsection}{Appendix}{Appendices}
\Crefname{appendixsection}{Appendix}{Appendices}
\crefname{appendixsubsection}{Appendix}{Appendices}
\Crefname{appendixsubsection}{Appendix}{Appendices}
\crefname{appendixsubsubsection}{Appendix}{Appendices}
\Crefname{appendixsubsubsection}{Appendix}{Appendices}

\title{TailSFT: Filtered Fine-Tuning Improves Post-Training Performance}

\arxiv{
  \author{}
  \date{}
}

\begin{document}
\maketitle

\arxiv{ 
\begingroup
\renewcommand{\thefootnote}{\fnsymbol{footnote}}
\vspace{-4em}
\begin{center}
\large
\setlength{\tabcolsep}{5pt}
\begin{tabular}{ccccc}
\makecell{Sadhika Malladi\footnotemark[2]} &
\makecell{Samy Jelassi\footnotemark[3]}
&
\makecell{Dylan J. Foster\footnotemark[3]}
&
\makecell{Jordan T. Ash\footnotemark[3]}
&
\makecell{Akshay Krishnamurthy\footnotemark[3]}
\end{tabular}\vspace{1em}
\end{center}
\vspace{2em}

\footnotetext[2]{University of California San Diego, \texttt{sadhika.malladi98@gmail.com}. Work partially completed while at Microsoft Research NYC.}
\footnotetext[3]{Microsoft Research NYC, \texttt{\{samyjelassi,dylanfoster,ash.jordan,akshaykr\}@microsoft.com}.}
\endgroup
}

\begin{abstract}

Reinforcement learning post-training drives reasoning and agentic capabilities in modern AI systems, yet a growing body of work shows that it is most effective when used to fine-tune an already capable base model. 
We question whether existing pipelines yield models that are most suitable for reinforcement learning. 
Building on prior work highlighting the role of \emph{coverage} and \passat{K} as predictors of post-RL performance, we design a simple modification to supervised fine-tuning, \emph{TailSFT}, which filters out already fit sequences during training, thereby focusing learning on under-modeled regions, or the tail, of the data distribution. 
We justify and validate the design choices in \tailsft{}, particularly the specific filtering criteria, through a combination of controlled experiments and theoretical analysis. On OLMo-3 7B, \tailsft{} often improves \passat{16} performance on math and coding evaluations, with gains up to $17\%$ absolute, while incurring minimal computational overhead. These higher-coverage checkpoints consistently translate to up to $4\%$ absolute \passat{1} gains in subsequent GRPO runs, demonstrating that \tailsft{} checkpoints serve as better initializations for RL. We further introduce a lightweight diagnostic for identifying settings where \tailsft{} is most likely to help. More broadly, our results motivate a principled, stage-aware approach to model development, in which intermediate checkpoints are judged by how effectively they support subsequent training.\looseness=-1




\end{abstract}


\begin{figure}[H]
    \centering
    \includegraphics[width=0.8\textwidth]{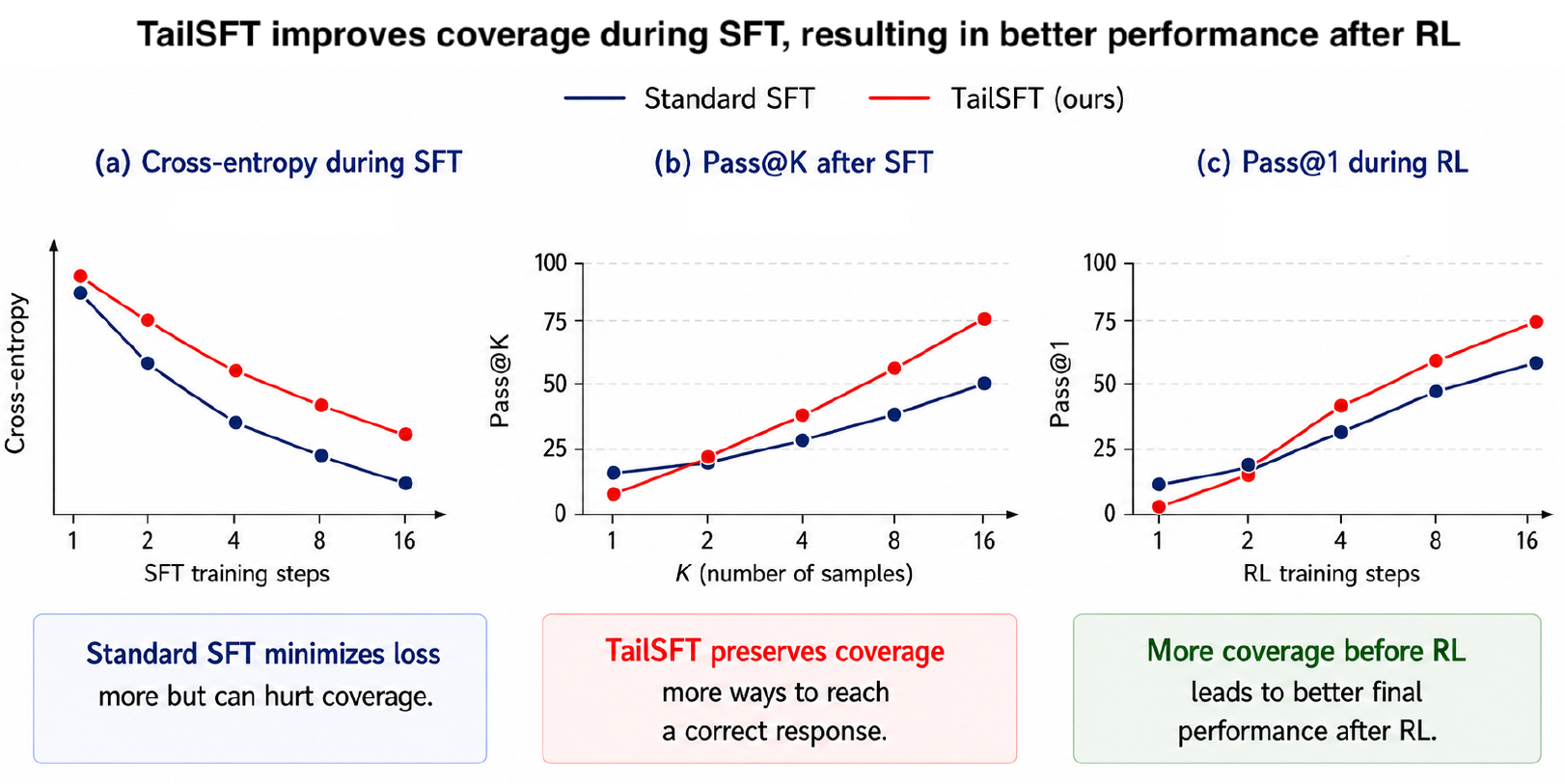}
    \caption{\textbf{\tailsft{} yields better coverage, possibly at the cost of worse cross-entropy, which translates to stronger post-RL performance.} Standard SFT achieves lower cross-entropy, while \tailsft{} preserves higher \passat{K} at large $K$. Starting from this higher-coverage checkpoint, the same RL procedure yields higher \passat{1}. Curves are illustrative.}
    \label{fig:overview}
    \vspace{0.5cm}
\end{figure}

\section{Introduction}
\label{sec:intro}

Language models are commonly trained in multiple phases, each defined by a different dataset and objective. A typical pipeline includes (1) next-token prediction on a general-purpose corpus, (2) supervised fine-tuning (SFT) for instruction following and domain adaptation, and (3) reinforcement learning (RL) for complex reasoning tasks. A tacit assumption in such pipelines is that improving the objective at one stage also produces a better initialization for the next. Recent work suggests that this assumption need not hold: models that perform better according to the local objective of one stage can nevertheless be worse starting points for subsequent training~\citep{zhang2025interplay,wang2025octothinker,springer2025overtraining,chen2026rethinking}. These findings motivate a stage-aware view of model development, in which an intermediate model is judged not only by its standalone performance, but also by how effectively it supports the training that follows.

This issue is particularly important at the transition from SFT to RL. RL has produced substantial gains in mathematics, code, and other reasoning domains~(\citealp{shao2024deepseekmath,deepseekai2025r1}; \citetalias{microsoft2025mai}; \citealp{zhao2026absolute}; \citetalias{team2026kimi}), but requires substantial computation to generate and evaluate on-policy responses. In RL with verifiable rewards (RLVR), the usefulness of these rollouts depends strongly on the initialization. If rewarding responses are difficult to sample from the SFT model, many rollout groups contain little direct positive learning signal; conversely, if rewarding responses are reachable under repeated sampling, RL has more useful behavior to reinforce. Thus, designing SFT for subsequent RL requires a criterion that captures not merely single-sample accuracy, but whether useful responses remain accessible within the rollout budget available to RL.

The \emph{coverage principle} formalizes this desideratum by relating a model's coverage profile to what repeated sampling can recover~\citep{chen2025coverageprinciple} (\Cref{def:coverage}). Empirically, \passat{k} provides a convenient measure of this property, in that it quantifies how often $k$ samples contain at least one reward-bearing response. At values of $k$ comparable to the sampling budget used during RL, large-$k$ \passat{k} therefore measures how frequently the initialization exposes RL to a useful response.

Standard SFT, however, is not explicitly designed to preserve or improve this form of coverage. Cross-entropy continues to reward increases in the likelihood of every demonstration, including responses that the model can readily produce already. Further fitting these responses can shift probability mass away from potentially useful responses represented by the initial model. Consequently, cross-entropy and coverage need not be aligned. A model with worse cross-entropy can have better coverage, and vice versa~\citep{chen2025coverageprinciple,chen2026rethinking}. On the other hand, directly optimizing the coverage profile is difficult, in part because the criterion is non-differentiable and depends on information about the target response distribution that is generally unavailable (\Cref{def:coverage}). This raises the question we study in this work:

\begin{center}
    \emph{Can we modify supervised fine-tuning so that it better preserves the useful response coverage needed for subsequent reinforcement learning?}
\end{center}


\paragraph{Contributions} 
To better prepare models for RL post-training, we propose \tailsft{}, an SFT algorithm that prioritizes coverage (or high \passat{k} at large $k$) over low cross-entropy loss (\Cref{alg:tailsft}).
\tailsft{} is a sequence-level filtering method 
designed to direct training effort towards learning the under-modeled tail of the response distribution. The algorithm is lightweight and serves as a straightforward drop-in objective for SFT.
\begin{enumerate}
    \item We isolate the effect of filtering in a controlled graph navigation task (\Cref{sec:synthetic}) and show that filtering already-fit examples improves coverage, achieving higher \passat{k} at large $k$ despite worse cross-entropy and \passat{1} (\Cref{fig:synthetic_graph,fig:synthetic_main}). \item We design \tailsft{} to filter relative to the initial policy and show that this criterion provably ensures better coverage than standard cross-entropy or more na\"{i}ve filtering (\Cref{thm:offset-clipping-informal}). 
    \item We apply \tailsft{} to OLMo-3 7B on standard math and coding tasks and obtain substantial gains in coverage and post-RL performance. \tailsft{} improves \passat{16}, in absolute terms, by up to 16.8\% on coding and 3.1\% on math (\Cref{tab:sft-main}), and improves final \passat{1} after GRPO by up to 3.9\% (\Cref{tab:grpo-results}).
    \item We introduce a lightweight coverage-ratio diagnostic (\Cref{def:coverage-ratio}), computed from the base model and a standard SFT run, that provides a sufficient condition for \tailsft{} to improve coverage (\Cref{fig:lm-clip-predictor}).
\end{enumerate}

\paragraph{Broader outlook}
Our results illustrate a broader source of suboptimality in multi-stage model training, where the objective that is appropriate for producing a strong model at one stage may not be the objective that produces the best initialization for the next. For the SFT-to-RL transition studied here, this distinction means that high large-$k$ \passat{k} can be more valuable than lower cross-entropy, even when the two criteria are in tension. Recent work has documented related forms of misalignment across pre-training, fine-tuning, and RL~\citep{zhang2025interplay,wang2025octothinker,springer2025overtraining,chen2025coverageprinciple,chen2026rethinking,watts2026sharpness}. \tailsft{} complements these observations with a practical intervention: rather than only diagnosing whether an intermediate model will support later training, it modifies SFT itself to produce a more useful initialization for the RL stage. Based on our results, we believe there is significant potential in adopting a holistic approach to language modeling.

\section{Preliminaries on Coverage}\label{sec:coverage}
We formalize coverage as a property of the SFT initialization that captures how readily reward-bearing responses can be sampled. We then relate coverage to \passat{K}, which provides an empirical measure of the learning signal available to RL at the start of post-training.

\paragraph{Notation and setting}
Consider a prompt distribution $\mu$ over a prompt space $\mathcal{X}$, a response space $\mathcal{Y}$, and a language model $\pi(\cdot\mid x)\in\Delta(\mathcal{Y})$. We focus on verifiable tasks with binary reward $R:\mathcal{X}\times\mathcal{Y}\rightarrow\{0,1\}$, where $R(x,y)=1$ indicates that response $y$ is correct for prompt $x$. In SFT, we observe $(x_i,y_i)$ with $x_i\sim\mu$ and $y_i\sim\pidata(\cdot\mid x_i)$, where $\pidata$ assigns probability to reward-bearing responses. Standard SFT minimizes the sequence-level cross-entropy $\ell_\pi(x,y)=-\log\pi(y\mid x)$, and the resulting model $\pisft$ initializes RL.

\paragraph{The coverage profile}
Consider a single prompt $x$, and let $p_\pi(x):=\Pr_{y\sim\pi(\cdot\mid x)}[R(x,y)=1]$. A batch of $K$ independent rollouts contains at least one rewarding response with probability $1-(1-p_\pi(x))^K$. If all $K$ rollouts receive zero reward, no gradient is received for their corresponding prompt. Thus, useful responses must have sufficient probability under the initialization to have a chance of being produced during RL. Across a dataset of many prompts, the model may benefit via positive transfer from easier to harder prompts, but the initial probability of rewarding responses, $p_\pi(x)$, determines the learning signal available before such transfer.

The coverage profile formalizes this property relative to the ground-truth, SFT-data policy~\citep{chen2025coverageprinciple}. We assume that $\pidata$ assigns non-negligible probability to reward-bearing responses, so coverage of $\pidata$ is relevant to downstream RL.

\begin{definition}[Coverage Profile]
Given a data-generating policy $\pidata$ and model $\pi$, the coverage profile at scale $N$ is
\[
\operatorname{Cov}_N(\pidata\Vert\pi)
:=
\Pr_{x\sim\mu,\,y\sim\pidata(\cdot\mid x)}
\left[
\frac{\pidata(y\mid x)}{\pi(y\mid x)}\geq N
\right].
\]
\label{def:coverage}
\end{definition}

The coverage profile measures the probability mass under $\pidata$ that $\pi$ underweights by a factor of at least $N$; smaller values indicate better coverage. The scale $N$ determines the sampling budget needed to reach this mass. \citet{chen2025coverageprinciple} show that, if $\operatorname{Cov}_N(\pidata\Vert\pi)\leq\tfrac{1}{2}$, then $K\geq 2N\log(1/\varepsilon)$ samples suffice for Best-of-$K$ to achieve expected reward within $\operatorname{Cov}_N(\pidata\Vert\pi)+\varepsilon$ of $\pidata$, with a corresponding worst-case necessity result with a comparable relationship between $N$ and $K$. Thus, coverage at scale $N$ characterizes which useful responses are accessible with a sampling budget on the order of $N$.

Since $\pidata(y\mid x)$ is generally unknown, the coverage profile cannot be measured directly.
However, for binary rewards, \passat{K}$(\pi):=\mathbb{E}_{x\sim\mu}[1-(1-p_\pi(x))^K]$ is exactly the expected reward of Best-of-$K$ sampling~\citep{brown2024monkeys}.  We therefore use \passat{K} at large $K$ as an empirical measure of coverage. At a $K$ comparable to the RL rollout budget, it also measures how often the policy exposes RL to reward-bearing responses.

\section{\tailsft{}: Supervised Fine-Tuning for Coverage}
\label{sec:tailsft}

In this section, we develop \tailsft{}, a sequence-level filtering method designed to direct SFT updates toward improving coverage. The method combines two ideas: stopping updates on responses that are already sufficiently well modeled via filtering, and determining which responses to filter relative to the initial policy. We first introduce the algorithm, then isolate the effect of filtering in a controlled graph-navigation task, and finally analyze why the initial policy provides a useful reference for filtering.

\subsection{The \tailsft{} Algorithm}
\label{sec:thresholding}

Coverage at scale $N$ depends on whether useful responses receive enough probability to be reached via repeated sampling. Once a response is already well covered at that scale, increasing its likelihood further does not reduce $\operatorname{Cov}_N$. Improving coverage instead requires increasing the probability of responses that remain under-modeled. Standard SFT makes no such distinction and continues to increase the likelihood of every training response.

\tailsft{} reduces training on responses that have already improved substantially and concentrates subsequent updates on the remaining tail. Since $\pidata(y\mid x)$ is unknown, we cannot determine directly which responses are already covered. Instead, we compare each response's current loss to its loss under the initial policy and filter the responses whose losses have decreased the most.

\begin{algorithm}[t]
\caption{\tailsft{}}
\label{alg:tailsft}
\begin{algorithmic}[1]
\Require Initial policy $\pi_0$, SFT data $\mathcal D$, filtering schedule $\{\gamma_t\}$

\State Record $\ell_i^0 \gets \ell_i(\pi_0)$ for every $(x_i,y_i)\in\mathcal D$

\For{each training step $t$}
    \State Sample a batch $\mathcal B_t$
    \State Compute $\ell_i^t \gets \ell_i(\pi_t)$ for each $(x_i,y_i)\in\mathcal B_t$
    \State Let $\mathcal F_t\subseteq\mathcal B_t$ be the $\gamma_t$ fraction of sequences with the smallest $\ell_i^t-\ell_i^0$
    \State Set $\pi_{t+1}$ by taking a gradient step from $\pi_t$ on
    \Statex \hspace{\algorithmicindent}
    $\displaystyle
    \mathcal L_t(\pi)
    =
    \frac{
        \sum_{i\in\mathcal B_t\setminus\mathcal F_t}
        -\log \pi(y_i\mid x_i)
    }{
        \sum_{i\in\mathcal B_t\setminus\mathcal F_t}
        |y_i|
    }$
\EndFor
\end{algorithmic}
\end{algorithm}

As described in \Cref{alg:tailsft}, \tailsft{} redirects training away from responses whose losses have decreased the most and toward responses that remain under-modeled. Filtering is applied at the sequence level. The length-normalized loss is used only to determine which sequences are filtered; optimization uses the standard token-averaged cross-entropy over target tokens in the retained sequences. The filtering fraction $\gamma_t$ may be fixed or vary over training.

Filtering itself does not require using the initial policy as a reference. We therefore consider two alternatives alongside \tailsft{}: \emph{absolute filtering}, which stops training on an example once its loss falls below a fixed threshold, and \emph{quantile filtering}, which filters the lowest-loss fraction of each batch. These alternatives allow us to separate the benefit of filtering itself from the benefit of measuring progress relative to the initial policy $\pi_0$.

\subsection{Warm-Up: Graph Navigation}
\label{sec:synthetic}

We study whether halting updates on an example once it is fit by the model can improve coverage, even at the cost of the standard SFT objective and pass@1. To do so, this section considers a controlled setting in which the distribution of rewarding responses is known exactly. Following~\citet{chen2025coverageprinciple}, each prompt $x$ specifies a layered directed graph with source $s$, target $t$, and exactly eight valid $s\!\to\!t$ paths. Pretraining pairs each graph with a uniformly sampled valid path, while the SFT data deterministically selects and rewards a single path (see~\cref{fig:synthetic_graph}). Thus, pretraining teaches generic graph navigation, while SFT teaches the task-specific path-selection rule. Since the target policy is deterministic, \passat{K} is exactly the probability that the rewarded path is contained in $K$ samples, yielding a finite-sample estimate of coverage. Full details of the graph construction and data-generating process appear in~\cref{app:synthetic}.

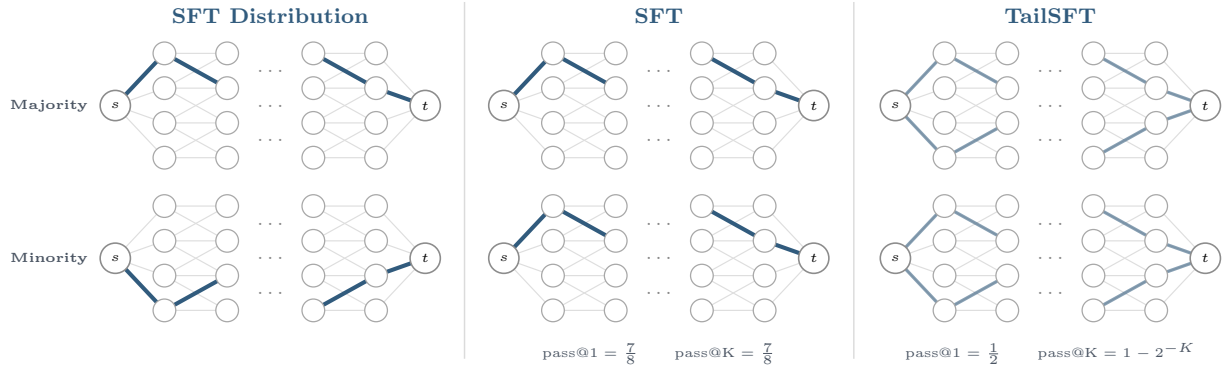
\begin{figure}
    \centering
    \renewcommand{\familydefault}{\sfdefault}

\definecolor{navy}{HTML}{315B7D}
\definecolor{tailred}{HTML}{D9534F}
\definecolor{gain}{HTML}{2F8F68}
\definecolor{softgreen}{HTML}{EAF6F0}
\definecolor{softred}{HTML}{FBEDEC}
\definecolor{softblue}{HTML}{EAF2F8}
\definecolor{textgray}{HTML}{4B5563}

\begin{tikzpicture}[
  x=1cm,y=1cm,font=\sffamily,
  baseedge/.style={draw=black!13,line width=0.44pt},
  node/.style={circle,draw=black!34,fill=white,line width=0.48pt,
    minimum size=3.0mm,inner sep=0pt},
  endpoint/.style={circle,draw=black!45,fill=white,line width=0.6pt,
    minimum size=4.0mm,inner sep=0pt,font=\fontsize{5.1}{5.7}\selectfont},
  certainpath/.style={draw=navy,line width=1.55pt,line cap=round,line join=round},
  halfpath/.style={draw=navy,line width=1.25pt,opacity=0.58,
    line cap=round,line join=round},
]

\newcommand{\basegraph}[2]{%
  \coordinate (s) at ({#1},{#2+0.91});
  \coordinate (t) at ({#1+4.10},{#2+0.91});
  \foreach \c/\dx in {1/0.65,2/1.48,3/2.62,4/3.45}{
    \foreach \r in {0,...,3}{
      \coordinate (n\c\r) at
        ({#1+\dx},{#2+0.22+0.46*\r});
    }
  }
  \coordinate (dots1) at ({#1+2.05},{#2+0.45});
  \coordinate (dots2) at ({#1+2.05},{#2+0.91});
  \coordinate (dots3) at ({#1+2.05},{#2+1.37});
  \foreach \r in {0,...,3}{
    \draw[baseedge] (s) -- (n1\r);
    \draw[baseedge] (n4\r) -- (t);
  }
  \foreach \c/\d in {1/2,3/4}{
    \foreach \r in {0,...,3}{
      \draw[baseedge] (n\c\r) -- (n\d\r);
    }
    \foreach \r/\q in {0/1,1/0,1/2,2/1,2/3,3/2}{
      \draw[baseedge] (n\c\r) -- (n\d\q);
    }
  }
}

\newcommand{\drawnodes}{%
  \foreach \c in {1,...,4}{
    \foreach \r in {0,...,3}{\node[node] at (n\c\r) {};}
  }
  \node[endpoint] at (s) {$s$};
  \node[endpoint] at (t) {$t$};
  \foreach \d in {1,2,3}{
    \node[font=\scriptsize,text=black!48,fill=white,
      inner xsep=1.2pt,inner ysep=0pt] at (dots\d) {$\hdots$};
  }
}

\newcommand{\drawpathA}[1]{%
  \draw[#1] (s) -- (n13) -- (n22);
  \draw[#1] (n33) -- (n42) -- (t);
}
\newcommand{\drawpathB}[1]{%
  \draw[#1] (s) -- (n10) -- (n21);
  \draw[#1] (n30) -- (n41) -- (t);
}


\draw[black!14,line width=0.55pt] (5.40,0.06) -- (5.40,4.78);
\draw[black!14,line width=0.55pt] (10.55,0.06) -- (10.55,4.78);

\node[font=\bfseries\footnotesize,text=navy] at (2.83,4.60)
  {SFT Distribution};

\node[font=\bfseries\footnotesize,text=navy] at (7.97,4.60)
  {SFT};

\node[font=\bfseries\footnotesize,text=navy] at (13.15,4.60)
  {TailSFT};

\node[anchor=east,font=\tiny\bfseries,text=textgray] at (0.54,3.41) {Majority};
\node[anchor=east,font=\tiny\bfseries,text=textgray] at (0.54,1.39) {Minority};

\begin{scope}
  \basegraph{0.78}{2.50}
  \drawpathA{certainpath}
  \drawnodes
\end{scope}

\begin{scope}
  \basegraph{0.78}{0.48}
  \drawpathB{certainpath}
  \drawnodes
\end{scope}

\begin{scope}
  \basegraph{5.92}{2.50}
  \drawpathA{certainpath}
  \drawnodes
\end{scope}

\begin{scope}
  \basegraph{5.92}{0.48}
  \drawpathA{certainpath}
  \drawnodes
\end{scope}

\node[font=\tiny,text=textgray,align=center] at (7.97,0.13)
  {\(\mathrm{pass@1}=\tfrac78\qquad \mathrm{pass@K}=\tfrac78\)};

\begin{scope}
  \basegraph{11.10}{2.50}
  \drawpathA{halfpath}
  \drawpathB{halfpath}
  \drawnodes
\end{scope}

\begin{scope}
  \basegraph{11.10}{0.48}
  \drawpathA{halfpath}
  \drawpathB{halfpath}
  \drawnodes
\end{scope}

\node[font=\tiny,text=textgray,align=center] at (13.15,0.13)
  {\(\mathrm{pass@1}=\tfrac12\qquad
     \mathrm{pass@K}=1-2^{-K}\)};


\end{tikzpicture}
    \vspace{-0.5cm}
    \caption{\textbf{\tailsft{} trades single-sample accuracy for better coverage under repeated sampling.} The SFT distribution contains majority and minority shards with different rewarding paths. Standard SFT over-anchors on the majority shard, achieving high \passat{1} but no improvement as $K$ grows. A filtered objective, illustrated here by \tailsft{}, retains probability on both paths, yielding lower \passat{1} but substantially better \passat{K} scaling.}
    \label{fig:synthetic_graph}
\end{figure}

Using a GPT-2-style architecture and a shared pretraining setup, we vary only the SFT objective. We compare standard SFT with absolute filtering, quantile filtering, and \tailsft{}, which uses offset quantile filtering (as in \Cref{alg:tailsft}). For each filtering variant, we sweep its single filtering hyperparameter and report the setting with the highest final \passat{8}; complete training details and sweeps appear in~\cref{app:synthetic}.

The results in~\cref{fig:synthetic_main} reveal a clear divergence between the standard SFT objective and coverage. Standard SFT achieves the best cross-entropy and \passat{1}, yet all three filtering methods achieve substantially better \passat{8}, which we use as a surrogate for coverage. Thus, improving cross-entropy and single-sample accuracy can come at the expense of the probability that a rewarded response appears within a finite rollout budget. Filtering mitigates this failure mode by reducing further updates on responses that are already well modeled.

\begin{figure}[t]
    \centering
    \includegraphics[width=\textwidth]{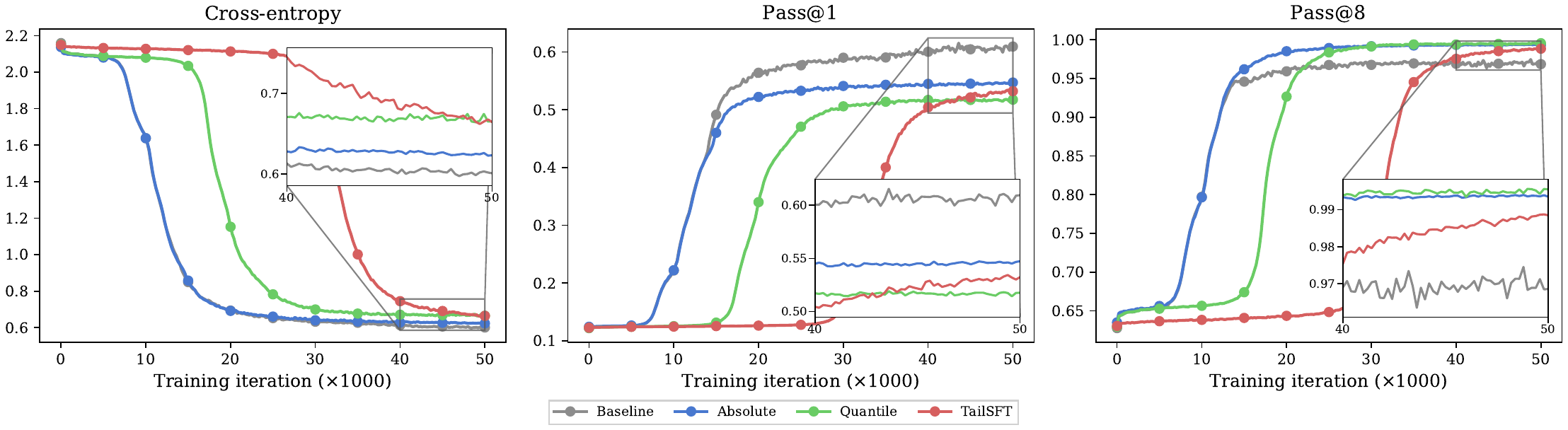}
    \vspace{-0.5cm}
    \caption{\textbf{Filtering improves \passat{K} at the expense of cross entropy and \passat{1}.} Results on the graph-navigation task. Standard SFT achieves the lowest cross-entropy and highest \passat{1}, while all three filtering methods achieve substantially higher \passat{8}. See~\cref{app:synthetic} for details. }
    \label{fig:synthetic_main}
\end{figure}

The graph task therefore provides clean evidence for the first design choice underlying \tailsft{}: filtering already-fit examples can improve coverage. It is, however, too simple a setting to determine whether the filtering criterion should depend on the initial policy. In fact, \tailsft{} is slightly weaker than other filtering variants in this setting, due to homogeneity across prompts. Under the ideal pretrained policy, every prompt has the same possible loss reduction (i.e., $\log 8$), so the reference loss provides no information for distinguishing which examples have more room to improve. Imperfections in pretraining therefore introduce noise into the offset score without providing useful signal. We discuss this effect further in~\cref{app:synthetic}.

\subsection{Theoretical Analysis}
\label{sec:tailsft-theory}

We next turn to a setting where the initial policy carries information about the response distribution, and ask whether  filtering rules can preserve it. To study this question, we turn to theoretical analysis.  
Let $S$ denote the set of rewarding responses, and suppose the target policy is obtained by discarding all responses outside $S$ and renormalizing the initial policy:
\[
\pi^\star(y)\propto \pi_0(y)\mathbf{1}\{y\in S\}.
\]
Thus, among rewarding responses, the target preserves the relative preferences already present in $\pi_0$. The initial model may place too much total probability on unrewarding responses, but it still contains information about how probability should be distributed among the rewarding ones.

We compare two ways to decide when an example should stop contributing to the SFT objective. Let $\ell_\pi(x,y)=-\log \pi(y\mid x)$ denote the sequence-level SFT loss. With an absolute filtering constant $\alpha$, we use
\[
\ell_{\mathrm{abs}}(\pi;x,y)
=
\bigl[\ell_\pi(x,y)+\log(\alpha)\bigr]_+.
\]
An example stops contributing once $\pi(y\mid x) \geq \alpha$. Because the same threshold is applied to every response, this criterion ignores how likely the response was under the initial policy.

\tailsft{} instead measures progress relative to the initial policy $\pi_0$. The corresponding offset loss is
\[
\ell_{\mathrm{off}}(\pi;x,y)
=
\left[
\ell_\pi(x,y)-\ell_{\pi_0}(x,y)+\log\beta
\right]_+,
\]
which stops updating a response once
\[
\pi(y\mid x)\geq\beta\pi_0(y\mid x).
\]
The stopping point therefore adapts to the probability that $\pi_0$ already assigns to each response rather than imposing a common final threshold.

\begin{theorem}[Informal]
Let $\pi_{\mathrm{ERM}}$, $\pi_{\mathrm{ABS},\alpha}$, and $\pi_{\mathrm{OFF},\beta}$ denote the policies obtained using standard SFT, absolute filtering, and offset filtering, respectively.
For any initial policy, rewarding set, and SFT sample, offset filtering can be tuned to achieve coverage no worse than standard SFT or the best absolute threshold:
\[
\inf_{\beta}
\operatorname{Cov}_N(\pi^\star\Vert\pi_{\mathrm{OFF},\beta})
\leq
\min\left\{
\operatorname{Cov}_N(\pi^\star\Vert\pi_{\mathrm{ERM}}),
\inf_{\alpha}
\operatorname{Cov}_N(\pi^\star\Vert\pi_{\mathrm{ABS},\alpha})
\right\}.
\]
The inequality can be strict, and absolute filtering can be worse than standard SFT.
\label{thm:offset-clipping-informal}
\end{theorem}

The theorem formalizes the role of the initial policy as more than merely an initialization. In this setting, $\pi_0$ contains useful information about how probability should be distributed among rewarding responses. Standard SFT can overwrite this structure by fitting the empirical frequencies of a finite SFT sample, while absolute filtering pushes observed responses toward a common probability threshold regardless of where they started. Offset filtering instead uses each response's initial probability to determine when further training is unnecessary. By preserving more of the useful structure already present in $\pi_0$, it can achieve strictly better coverage. The full statement and proof appear in~\cref{app:theory}.

\paragraph{Connections to prior work}
\tailsft{} is most closely related to direct coverage optimization and loss-based data filtering. \citet{chen2026rethinking} introduce the direct coverage optimization loss
\[
\ell_{\mathrm{DCO}}(\pi;x,y)
=
-\log\left(1-\left(1-\pi(y\mid x)\right)^K\right),
\]
which directly optimizes \passat{K} when $y$ is the final solution. For large $K$, its gradient becomes small once $\pi(y\mid x)$ is sufficiently large, making it a soft analogue for filtering without reference to the initial policy. \Cref{thm:offset-clipping-informal} shows that, in our stylized setting, offset filtering can achieve coverage no worse than the best absolute threshold and can be strictly better.
\tailsft{} is also related to RHO-LOSS and subsequent variants, though these methods were designed and tested primarily to improve pretraining efficiency~\citep{mindermann2022prioritized,thirukovalluru2024sequence,lin2024rho1,zhao2026bounded}. These approaches prioritize examples or tokens according to their potential loss reduction. \tailsft{}, by contrast, is designed to improve coverage for a subsequent RL stage and applies filtering at the sequence level. See~\cref{app:related_work} for further discussion.




\section{Language Model Experiments}
\label{sec:lm}
The results in~\Cref{sec:tailsft} suggest that \tailsft{} can preserve reward-bearing responses more effectively than standard SFT. We now evaluate whether this behavior carries over to pretrained language models. We first compare standard SFT and \tailsft{} on math and code tasks. The variation across these results then allows us to study when preserving information from the base model is most useful. Finally, we follow matched SFT checkpoints through GRPO to test whether higher coverage before RL translates into stronger post-RL performance. Throughout, we use \passat{16} as an empirical measure of coverage.

\subsection{SFT Results}
\label{sec:lm-sft-results}

\begin{table}[t]
\centering
\footnotesize
\setlength{\tabcolsep}{5pt}
\renewcommand{\arraystretch}{1.2}
\providecommand{\dgain}[1]{\textcolor{ForestGreen}{#1}}
\providecommand{\dloss}[1]{\textcolor{BrickRed}{#1}}
\providecommand{\dnull}[1]{\textcolor{yellow!65!black}{#1}}
\newcommand{\val}[2]{$#1_{\,\pm#2}$}
\newcommand{\valn}[1]{$#1$}
\begin{tabular}{@{}ll rrr rrr@{}}
\toprule
 & & \multicolumn{3}{c}{\textbf{pass@1}} & \multicolumn{3}{c}{\textbf{pass@16}} \\
\cmidrule(lr){3-5}\cmidrule(lr){6-8}
SFT data & Benchmark & Standard & \tailsft{} & $\Delta$ & Standard & \tailsft{} & $\Delta$ \\
\midrule
OMI & AIME 2022--2025 ($n{=}120$) & \val{3.65}{0.37} & \val{4.03}{0.02} & \dgain{$+0.37$} & \val{15.24}{2.02} & \val{18.31}{0.57} & \dgain{$+3.07$} \\
           & MATH-500 Level 5 ($n{=}134$) & \val{24.80}{0.40} & \val{25.16}{0.40} & \dnull{$+0.36$} & \val{66.42}{0.00} & \val{69.15}{0.86} & \dgain{$+2.74$} \\
           & OMEGA-500 ($n{=}500$) & \val{6.58}{0.50} & \val{6.51}{0.26} & \dnull{$-0.06$} & \val{32.80}{2.25} & \val{32.60}{1.56} & \dnull{$-0.20$} \\
\midrule
BigCode   & MBPP+ ($n{=}378$)          & \val{53.04}{0.93} & \val{51.36}{0.14} & \dloss{$-1.68$} & \val{74.69}{1.10} & \val{78.84}{1.06} & \dgain{$+4.14$} \\
          & HumanEval+ ($n{=}164$)     & \val{45.06}{0.57} & \val{46.33}{1.07} & \dgain{$+1.27$} & \val{76.63}{0.70} & \val{80.08}{0.93} & \dgain{$+3.46$} \\
          & CruxEval-I ($n{=}800$)     & \val{29.79}{0.48} & \val{30.43}{0.15} & \dgain{$+0.64$} & \val{61.71}{2.27} & \val{65.79}{0.26} & \dgain{$+4.08$} \\
          & CruxEval-O ($n{=}800$)     & \val{4.16}{0.44}  & \val{13.18}{0.30} & \dgain{$+9.02$} & \val{24.21}{2.38} & \val{41.00}{1.35} & \dgain{$+16.79$} \\
          & LiveCodeBench ($n{=}612$)  & \val{10.74}{0.46} & \val{11.49}{0.38} & \dgain{$+0.75$} & \val{30.39}{0.86} & \val{33.12}{0.34} & \dgain{$+2.72$} \\
\addlinespace[2pt]
Magicoder & MBPP+ ($n{=}378$)          & \val{55.35}{0.47} & \val{54.60}{0.55} & \dloss{$-0.75$} & \val{75.31}{0.93} & \val{78.66}{1.25} & \dgain{$+3.35$} \\
          & HumanEval+ ($n{=}164$)     & \val{48.49}{0.58} & \val{47.69}{0.19} & \dloss{$-0.80$} & \val{77.85}{0.70} & \val{78.66}{1.83} & \dnull{$+0.81$} \\
          & CruxEval-I ($n{=}800$)     & \val{28.50}{0.13} & \val{26.48}{0.45} & \dloss{$-2.02$} & \val{59.75}{0.62} & \val{68.08}{0.94} & \dgain{$+8.33$} \\
          & CruxEval-O ($n{=}800$)     & \val{12.86}{0.19} & \val{16.16}{0.53} & \dgain{$+3.30$} & \val{38.08}{0.94} & \val{47.92}{1.66} & \dgain{$+9.83$} \\
          & LiveCodeBench ($n{=}612$)  & \val{14.84}{0.11} & \val{14.32}{0.38} & \dloss{$-0.51$} & \val{31.81}{0.34} & \val{33.22}{0.34} & \dgain{$+1.42$} \\
\addlinespace[2pt]
OCI       & MBPP+ ($n{=}378$) & \val{58.26}{0.88} & \val{55.81}{0.06} & \dloss{$-2.44$} & \val{81.22}{0.75} & \val{82.36}{0.81} & \dgain{$+1.15$} \\
          & HumanEval+ ($n{=}164$) & \val{54.76}{1.16} & \val{51.94}{1.83} & \dloss{$-2.82$} & \val{85.98}{1.61} & \val{83.23}{2.16} & \dloss{$-2.74$} \\
          & CruxEval-I ($n{=}800$)     & \val{29.46}{0.46} & \val{30.17}{0.64} & \dnull{$+0.71$} & \val{68.08}{1.77} & \val{71.58}{1.91} & \dgain{$+3.50$} \\
          & CruxEval-O ($n{=}800$)     & \val{9.30}{0.49}  & \val{10.84}{0.37} & \dgain{$+1.54$} & \val{40.12}{0.00} & \val{44.46}{0.47} & \dgain{$+4.33$} \\
          & LiveCodeBench ($n{=}612$)  & \val{12.24}{0.36} & \val{11.49}{0.84} & \dnull{$-0.75$} & \val{34.59}{0.90} & \val{33.50}{0.16} & \dloss{$-1.09$} \\
\bottomrule
\end{tabular}
\caption{
\textbf{\tailsft{} improves coverage over standard SFT.}
Across 18 model--benchmark pairs, \tailsft{} often improves \passat{16}, while changes in \passat{1} are mixed.
Values are percentages reported as mean $\pm$ standard deviation over three seeds; $\Delta$ denotes \tailsft{} minus Standard SFT in percentage points, shown in green when positive, red when negative, and yellow when within one standard deviation.
We later establish a sufficient condition (\Cref{def:coverage-ratio}) that accurately predicts the failure modes of \tailsft{} exhibited in this table (\Cref{fig:lm-clip-predictor}).
MATH-500 Level 5 refers to the subset of the benchmark with the highest difficulty.}
\label{tab:sft-main}
\end{table}

\paragraph{Setup}
We fine-tune the OLMo-3 7B base model~\citepalias{olmo2026olmo3} separately on domain-specific math and code instruction data, giving 18 dataset--benchmark pairs. For math, we train on a $350$k-example subset of OpenMathInstruct-2 (OMI, \citet{toshniwal2024openmathinstruct}), decontaminated against MATH-500. We evaluate with the OLMES protocol~\citep{gu2025olmes} on the 2022--2025 AIME problems, the hardest level of MATH-500, and OMEGA-500. For code, we train separately on Magicoder~\citep{wei2024magicoder}, BigCode Self-OSS-Instruct~\citep{wei2024selfcodealign}, and OpenCodeInstruct (OCI, \citet{ahmad2025opencodeinstruct}). We evaluate on MBPP+, HumanEval+, CruxEval-I/O, and LiveCodeBench. We focus on benchmarks where performance is not saturated.

We report \passat{16} using the task-specific OLMES sampling protocol and average results over three seed sets. Results are reported in \Cref{tab:sft-main} and \Cref{app:sft} contains experimental details.

\paragraph{\tailsft{} improves coverage in most settings}
Across the 18 dataset--benchmark pairs in~\Cref{tab:sft-main}, \tailsft{} improves \passat{16} in 15 cases, while changes in \passat{1} are mixed. OMEGA-500 is essentially unchanged, with an absolute difference of $-0.20\%$. The largest gain occurs on CruxEval-O after SFT on BigCode, where \passat{16} increases by an absolute $16.79\%$. With Magicoder, \passat{16} increases by $8.33\%$ on CruxEval-I and $9.83\%$ on CruxEval-O. On AIME, it increases by $3.07\%$. On MBPP+, absolute gains are $4.14\%$, $3.35\%$, and $1.15\%$ for BigCode, Magicoder, and OCI. Overall, improvements are substantially more consistent at large $K$ than at $K=1$, in line with the coverage motivation for \tailsft{}.

\subsection{Coverage Ratio Diagnostic}
\label{sec:predicting}

\begin{figure}[t]
\centering
\includegraphics[width=0.9\textwidth]{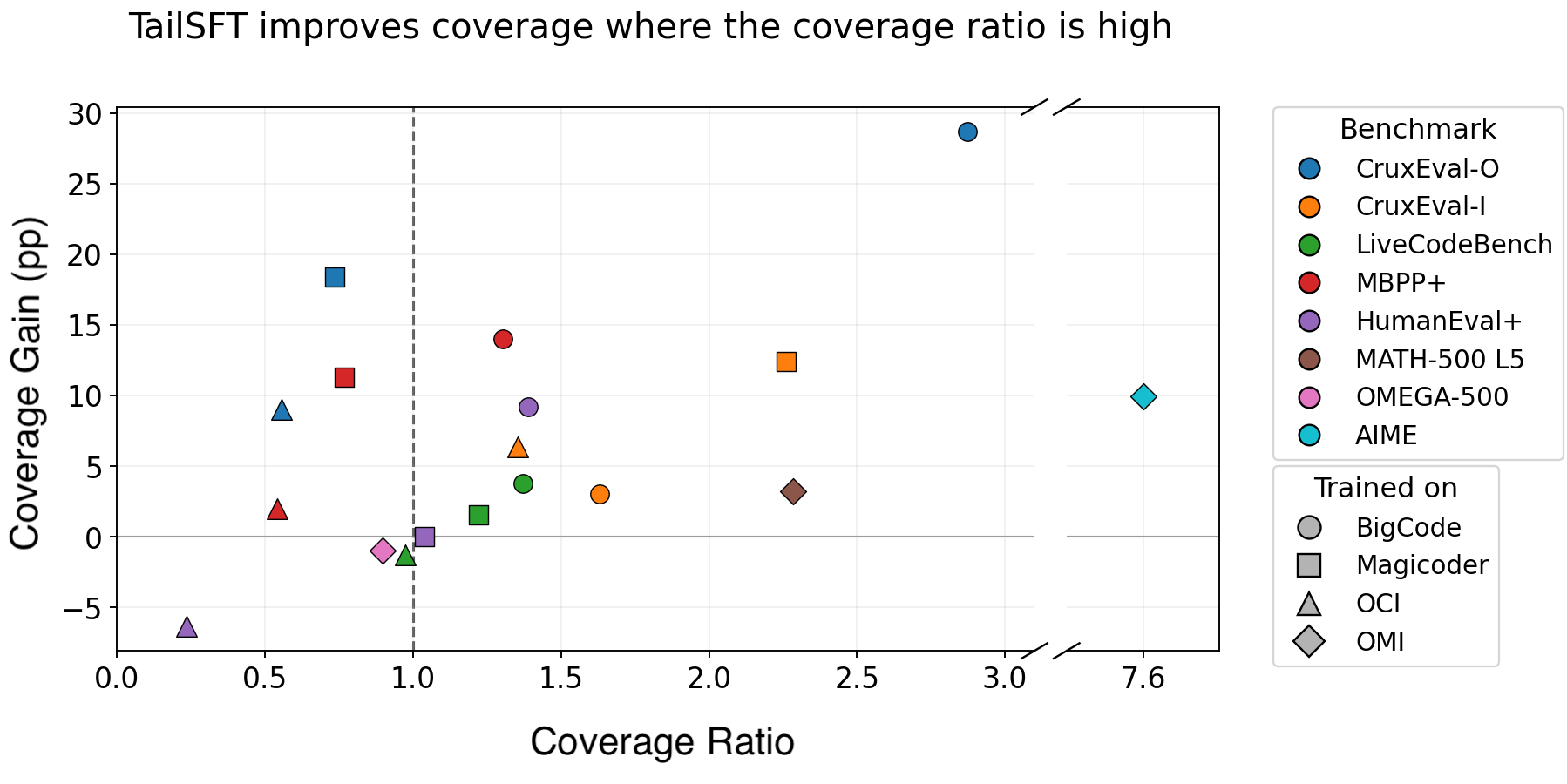}
\vspace{-0.4cm}
\caption{\textbf{When standard SFT loses more coverage than it gains, \tailsft{} preserves or improves coverage.} The vertical line marks $\rho_{16}=1$, where the estimated coverage lost and gained by standard SFT are equal. Of the 11 evaluated settings with $\rho_{16}>1$, 10 have positive coverage gains from \tailsft{} and one is essentially unchanged. Several settings with $\rho_{16}<1$ also improve. Coverage gain and ratio are computed on the base-reachable set (\Cref{def:coverage-ratio}).}
\label{fig:lm-clip-predictor}
\vspace{-0.4cm}
\end{figure}

The variation across these settings motivates a closer look at when \tailsft{} should perform well. In the graph-navigation task, standard SFT increased probability on the majority shard while reducing coverage of a rewarded path that was already represented by the initial policy. 
Based on this, we ask whether SFT in the language model setting exhibits a measurable loss of coverage relative to the base model, and whether this coverage loss is predictive of the benefit of \tailsft{}.

We measure how standard SFT changes coverage on problems that are already reachable from the base model. Restricting attention to problems that are neither effectively unsolved nor already saturated under the base model, we use \passat{1} to estimate the \passat{16} that would result from independent sampling. We then aggregate the decreases and increases induced by standard SFT---their ratio measures whether standard SFT has lost more base-reachable coverage than it has gained.

\begin{definition}[Coverage ratio]
\label{def:coverage-ratio}
For problem $i$, model $\pi$, and
$K\in\{1,16\}$, let $P_{i,K}(\pi)$ denote the empirical \passat{K}, averaged
over seeds. We convert \passat{1} to an estimated \passat{16} value using
$f_{16}(p):=1-(1-p)^{16}$.
The base-reachable set is
$\mathcal R_0:=\{i:0.05<P_{i,16}(\pi_0)<0.95\}$. The coverage lost and gained
by standard SFT on this set are
\[
\begin{aligned}
L
&:=
\sum_{i\in\mathcal R_0}
\Bigl[
    f_{16}\bigl(P_{i,1}(\pi_0)\bigr)
    -
    f_{16}\bigl(P_{i,1}(\pi_{\mathrm{SFT}})\bigr)
\Bigr]_+,
\\
G
&:=
\sum_{i\in\mathcal R_0}
\Bigl[
    f_{16}\bigl(P_{i,1}(\pi_{\mathrm{SFT}})\bigr)
    -
    f_{16}\bigl(P_{i,1}(\pi_0)\bigr)
\Bigr]_+.
\end{aligned}
\]
The coverage ratio is $\rho_{16}:=L/G$. If $G=0$, we set
$\rho_{16}:=\infty$.
\end{definition}

Computing $\rho_{16}$ requires only the base model and a single standard SFT run. A value $\rho_{16}>1$ means that, on the base-reachable set, the estimated coverage lost during standard SFT exceeds the estimated coverage gained. This is the setting most directly targeted by \tailsft{}, which reduces training on responses that have already improved and thereby limits probability shifting away from responses represented by the base model.

\paragraph{The coverage ratio identifies a regime where \tailsft{} consistently preserves or improves coverage}
In~\Cref{fig:lm-clip-predictor}, all 11 dataset--benchmark pairs satisfying $\rho_{16}>1$ have nonnegative coverage gains (improvement in \passat{16} on the base-reachable set $\mathcal{R}_0$) from \tailsft{}. Ten improve and one is essentially unchanged, with gains reaching an absolute $28.69\%$ for BigCode evaluated on CruxEval-O. The condition is not necessary for improvement: several settings with $\rho_{16}<1$ also benefit from \tailsft{}. Thus, the coverage ratio provides a conservative way to identify settings in which \tailsft{} preserves or improves coverage. \Cref{app:coverage-ratio-details} contains further details.

\subsection{GRPO Results}

The coverage principle predicts that a higher-coverage initialization should expose RL to rewarding responses more often~\citep{chen2025coverageprinciple}. This section therefore tests whether the additional coverage resulting from \tailsft{} translates into better post-RL performance.

\paragraph{Setup}
We initialize GRPO from the standard SFT and \tailsft{} checkpoints evaluated in~\Cref{tab:sft-main}. For math reasoning, we train on the MATH training split with MATH-500 held out and evaluate on MATH-500 Level 5 and AIME. For code, we train on the MBPP+ training split and evaluate on its test split. The GRPO procedure and hyperparameters are held fixed within each comparison. The main comparisons use $4$ rollouts per prompt and an actor learning rate of $2\times10^{-5}$. Evaluation follows the OLMES protocol over three seed sets.\looseness=-1 

\paragraph{Improved coverage yields better post-RL performance}
Initializing from \tailsft{} improves post-RL \passat{1} in every matched comparison in~\Cref{tab:grpo-results}, with absolute improvements ranging between $1.21\%$ and $3.93\%$. Our coding experiments make the role of initialization especially salient---each \tailsft{} checkpoint has lower \passat{1} than its standard SFT counterpart before RL, but has higher \passat{1} afterward. Broader initial coverage indeed results in higher single-sample accuracy after GRPO. These results are consistent with the relationship between large-$K$ performance and post-RL outcomes observed by~\citet{kang2026quagmires}.

\begin{table}[t]
\centering
\footnotesize
\setlength{\tabcolsep}{5pt}
\renewcommand{\arraystretch}{1.2}
\providecommand{\dgain}[1]{\textcolor{ForestGreen}{#1}}
\providecommand{\dnull}[1]{\textcolor{yellow!65!black}{#1}}
\newcommand{\val}[2]{$#1_{\,\pm#2}$}
\begin{tabular}{@{}ll rrr rrr@{}}
\toprule
 & & \multicolumn{3}{c}{\textbf{pass@1}} & \multicolumn{3}{c}{\textbf{pass@16}} \\
\cmidrule(lr){3-5}\cmidrule(lr){6-8}
SFT data & Benchmark & Standard & TailSFT & $\Delta$ & Standard & TailSFT & $\Delta$ \\
\midrule
\multicolumn{8}{@{}l}{\textit{GRPO on MATH}}\\
\addlinespace[1.5pt]
OMI       & MATH-500 Level 5 ($n{=}134$) & \val{57.70}{0.44} & \val{60.26}{1.02} & \dgain{$+2.56$} & \val{83.83}{1.72} & \val{87.06}{0.86} & \dgain{$+3.23$} \\
          & AIME 2022--2025 ($n{=}120$) & \val{14.40}{0.26} & \val{15.61}{0.44} & \dgain{$+1.21$} & \val{32.76}{0.88} & \val{36.06}{0.98} & \dgain{$+3.30$} \\
\midrule
\multicolumn{8}{@{}l}{\textit{GRPO on MBPP+}}\\
\addlinespace[1.5pt]
BigCode   & MBPP+ ($n{=}378$)           & \val{69.57}{0.29} & \val{73.50}{0.70} & \dgain{$+3.93$} & \val{75.93}{0.00} & \val{78.66}{0.15} & \dgain{$+2.73$} \\
Magicoder & MBPP+ ($n{=}378$)           & \val{70.52}{0.25} & \val{73.24}{0.09} & \dgain{$+2.72$} & \val{78.22}{0.61} & \val{80.60}{0.55} & \dgain{$+2.38$} \\
OCI       & MBPP+ ($n{=}378$) & \val{74.67}{0.08} & \val{76.30}{0.05} & \dgain{$+1.62$} & \val{84.22}{0.40} & \val{84.04}{0.15} & \dnull{$-0.18$} \\
\bottomrule
\end{tabular}
\caption{
\textbf{Initializing from the \tailsft{} model improves post-GRPO performance across math and code.}
GRPO is initialized from the Standard SFT and \tailsft{} checkpoints evaluated in \cref{tab:sft-main}, and rows are grouped by the corpus used for GRPO training: the MATH train split, excluding MATH-500, for the math rows, and the MBPP+ train split for the code rows.
The SFT data column indicates the corpus used to train the initialization for each GRPO run.
Values are percentages reported as mean $\pm$ standard deviation over three seeds, and $\Delta$ is improvement from \tailsft{}.
Experiment details are in \Cref{app:grpo}.
}
\label{tab:grpo-results}
\end{table}

\paragraph{Improved coverage drives faster learning}
\Cref{fig:grpo-reward-trajectories} shows that the difference between the two initializations appears early in training, with \tailsft{} runs beginning with lower reward but the gap quickly closing. In some settings, early reward increases up to $2.5\times$ faster than for the corresponding standard SFT run. After RL, \tailsft{} retains higher \passat{16} in four of the five matched comparisons, and is approximately tied in the fifth. The coverage preserved during SFT is therefore available to RL early in training and ultimately translates into higher \passat{1}.




\section{Discussion}
\label{sec:discussion}
We present \tailsft{} as a drop-in replacement for standard SFT that filters already-fit sequences to improve coverage after the SFT stage, which produces improvements in post-RL performance. \tailsft{} is derived using a combination of theoretical analysis and controlled empirics, and significantly outperforms standard SFT in our language modeling experiments. 
More broadly, TailSFT uses the coverage principle to target the interaction between the SFT and RL stages of language model training~\citep{chen2025coverageprinciple}. We believe our results motivate a shift in how language models should be optimized, toward targeting the full training trajectory rather than fitting each stage independently.



\bibliography{bibliography}

@article{deepseekai2025r1,
  title   = {{DeepSeek-R1} incentivizes reasoning in {LLM}s through reinforcement learning},
  author  = {Guo, Daya and Yang, Dejian and Zhang, Haowei and Song, Junxiao and Wang, Peiyi and Zhu, Qihao and Xu, Runxin and Zhang, Ruoyu and Ma, Shirong and Bi, Xiao and others},
  journal = {Nature},
  year    = {2025}
}

@inproceedings{yue2025beyondbase,
  title     = {Does Reinforcement Learning Really Incentivize Reasoning Capacity in {LLM}s Beyond the Base Model?},
  author    = {Yue, Yang and Chen, Zhiqi and Lu, Rui and Zhao, Andrew and Wang, Zhaokai and Yue, Yang and Song, Shiji and Huang, Gao},
  booktitle = {Advances in Neural Information Processing Systems},
  year      = {2025}
}

@inproceedings{wu2025invisibleleash,
  title     = {The Invisible Leash: Why {RLVR} May Not Escape Its Origin},
  author    = {Wu, Fang and Choi, Yejin},
  booktitle = {AI for Math Workshop},
  year      = {2025}
}

@inproceedings{zhao2025echochamber,
  title     = {Echo Chamber: {RL} Post-training Amplifies Behaviors Learned in Pretraining},
  author    = {Zhao, Rosie and Meterez, Alexandru and Kakade, Sham M. and Pehlevan, Cengiz and Jelassi, Samy and Malach, Eran},
  booktitle = {Conference on Language Modeling},
  year      = {2025}
}

@inproceedings{wen2025correctreasoning,
  title     = {Reinforcement Learning with Verifiable Rewards Implicitly Incentivizes Correct Reasoning in Base {LLM}s},
  author    = {Wen, Xumeng and Liu, Zihan and Zheng, Shun and Ye, Shengyu and Wu, Zhirong and Wang, Yang and Xu, Zhijian and Liang, Xiao and Li, Junjie and Miao, Ziming and Bian, Jiang and Yang, Mao},
  booktitle = {International Conference on Learning Representations},
  year      = {2026}
}

@inproceedings{chen2025coverageprinciple,
  title     = {The Coverage Principle: How Pre-Training Enables Post-Training},
  author    = {Chen, Fan and Huang, Audrey and Golowich, Noah and Malladi, Sadhika and Block, Adam and Ash, Jordan and Krishnamurthy, Akshay and Foster, Dylan},
  booktitle = {International Conference on Learning Representations},
  year      = {2026}
}

@inproceedings{foster2025goodfoundation,
  title     = {Is a Good Foundation Necessary for Efficient Reinforcement Learning? The Computational Role of the Base Model in Exploration},
  author    = {Foster, Dylan J. and Mhammedi, Zakaria and Rohatgi, Dhruv},
  booktitle = {Conference on Learning Theory},
  year      = {2025}
}

@inproceedings{zhang2025interplay,
  title     = {On the Interplay of Pre-Training, Mid-Training, and {RL} on Reasoning Language Models},
  author    = {Zhang, Charlie and Neubig, Graham and Yue, Xiang},
  booktitle = {International Conference on Machine Learning},
  year      = {2026}
}

@inproceedings{wang2025octothinker,
  title     = {{OctoThinker}: Mid-Training Incentivizes Reinforcement Learning Scaling},
  author    = {Wang, Zengzhi and Zhou, Fan and Li, Xuefeng and Liu, Pengfei},
  booktitle = {AI for Math Workshop},
  year      = {2025}
}

@inproceedings{springer2025overtraining,
  title     = {Overtrained Language Models Are Harder to Fine-Tune},
  author    = {Springer, Jacob Mitchell and Goyal, Sachin and Wen, Kaiyue and Kumar, Tanishq and Yue, Xiang and Malladi, Sadhika and Neubig, Graham and Raghunathan, Aditi},
  booktitle = {International Conference on Machine Learning},
  year      = {2025}
}

@misc{shao2024deepseekmath,
  title        = {{DeepSeekMath}: Pushing the Limits of Mathematical Reasoning in Open Language Models},
  author       = {Shao, Zhihong and Wang, Peiyi and Zhu, Qihao and Xu, Runxin and Song, Junxiao and Bi, Xiao and Zhang, Haowei and Zhang, Mingchuan and Li, Y. K. and Wu, Y. and Guo, Daya},
  howpublished = {arXiv:2402.03300},
  year         = {2024}
}

@misc{qin2025curatedsft,
  title        = {Supervised Fine Tuning on Curated Data is Reinforcement Learning (and can be improved)},
  author       = {Qin, Chongli and Springenberg, Jost Tobias},
  howpublished = {arXiv:2507.12856},
  year         = {2025}
}

@inproceedings{lin2024rho1,
  title     = {Not All Tokens Are What You Need for Pretraining},
  author    = {Lin, Zhenghao and Gou, Zhibin and Gong, Yeyun and Liu, Xiao and Shen, Yelong and Xu, Ruochen and Lin, Chen and Yang, Yujiu and Jiao, Jian and Duan, Nan and Chen, Weizhu},
  booktitle = {Advances in Neural Information Processing Systems},
  year      = {2024}
}

@inproceedings{kang2026quagmires,
  title     = {Quagmires in {SFT-RL} Post-Training: When High {SFT} Scores Mislead and What to Use Instead},
  author    = {Kang, Feiyang and Kuchnik, Michael and Padthe, Karthik and Vlastelica, Marin and Jia, Ruoxi and Wu, Carole-Jean and Ardalani, Newsha},
  booktitle = {International Conference on Learning Representations},
  year      = {2026}
}

@inproceedings{watts2026sharpness,
  title     = {Sharpness-Aware Pretraining Mitigates Catastrophic Forgetting},
  author    = {Watts, Ishaan and Li, Catherine and Goyal, Sachin and Springer, Jacob Mitchell and Raghunathan, Aditi},
  booktitle = {International Conference on Machine Learning},
  year      = {2026}
}

@inproceedings{yoshihara2025twostage,
  title     = {A Practical Two-Stage Recipe for Mathematical {LLM}s: Maximizing Accuracy with {SFT} and Efficiency with Reinforcement Learning},
  author    = {Yoshihara, Hiroshi and Yamaguchi, Taiki and Inoue, Yuichi},
  booktitle = {AI for Math Workshop},
  year      = {2025}
}

@misc{niu2026nondecoupling,
  title        = {On the Non-decoupling of Supervised Fine-tuning and Reinforcement Learning in Post-training},
  author       = {Niu, Xueyan and Bai, Bo and Han, Wei and Zhang, Weixi},
  howpublished = {arXiv:2601.07389},
  year         = {2026}
}

@inproceedings{bansal2026rl,
  title     = {{RL} Excursions during Pre-training: How early is too early for On-policy Learning?},
  author    = {Bansal, Rachit and Mohri, Clara and Qin, Tian and Alvarez-Melis, David and Kakade, Sham M.},
  booktitle = {Workshop on Scaling Post-Training for {LLM}s},
  year      = {2026}
}

@inproceedings{fu2026srft,
  title     = {{SRFT}: A Single-Stage Method with Supervised and Reinforcement Fine-Tuning for Reasoning},
  author    = {Fu, Yuqian and Chen, Tinghong and Chai, Jiajun and Wang, Xihuai and Tu, Songjun and Yin, Guojun and Lin, Wei and Zhang, Qichao and Zhu, Yuanheng and Zhao, Dongbin},
  booktitle = {International Conference on Learning Representations},
  year      = {2026}
}

@inproceedings{huang2026prefixrft,
  title     = {Blending Supervised and Reinforcement Fine-Tuning with Prefix Sampling},
  author    = {Huang, Zeyu and Cheng, Tianhao and Qiu, Zihan and Wang, Zili and Xu, Yinghui and Ponti, Edoardo M. and Titov, Ivan},
  booktitle = {International Conference on Machine Learning},
  year      = {2026}
}

@inproceedings{zhang2026chord,
  title     = {On-Policy {RL} Meets Off-Policy Experts: Harmonizing Supervised Fine-Tuning and Reinforcement Learning via Dynamic Weighting},
  author    = {Zhang, Wenhao and Xie, Yuexiang and Sun, Yuchang and Chen, Yanxi and Wang, Guoyin and Li, Yaliang and Ding, Bolin and Zhou, Jingren},
  booktitle = {International Conference on Learning Representations},
  year      = {2026}
}

@inproceedings{liu2023implicitbias,
  title     = {Same Pre-training Loss, Better Downstream: Implicit Bias Matters for Language Models},
  author    = {Liu, Hong and Xie, Sang Michael and Li, Zhiyuan and Ma, Tengyu},
  booktitle = {International Conference on Machine Learning},
  year      = {2023}
}

@misc{zeng2025indicators,
  title        = {Can Pre-training Indicators Reliably Predict Fine-tuning Outcomes of {LLM}s?},
  author       = {Zeng, Hansi and Hui, Kai and Zhuang, Honglei and Qin, Zhen and Yue, Zhenrui and Zamani, Hamed and Alon, Dana},
  howpublished = {arXiv:2504.12491},
  year         = {2025}
}

@inproceedings{lourie2025scaling,
  title     = {Scaling Laws Are Unreliable for Downstream Tasks: A Reality Check},
  author    = {Lourie, Nicholas and Hu, Michael Y. and Cho, Kyunghyun},
  booktitle = {Findings of the Association for Computational Linguistics},
  year      = {2025}
}

@misc{brown2024monkeys,
  title        = {Large Language Monkeys: Scaling Inference Compute with Repeated Sampling},
  author       = {Brown, Bradley and Juravsky, Jordan and Ehrlich, Ryan and Clark, Ronald and Le, Quoc V. and R{\'e}, Christopher and Mirhoseini, Azalia},
  howpublished = {arXiv:2407.21787},
  year         = {2024}
}

@inproceedings{farahmand2010error,
  title     = {Error Propagation for Approximate Policy and Value Iteration},
  author    = {Farahmand, Amir-massoud and Szepesv{\'a}ri, Csaba and Munos, R{\'e}mi},
  booktitle = {Advances in Neural Information Processing Systems},
  year      = {2010}
}

@inproceedings{chen2019information,
  title     = {Information-Theoretic Considerations in Batch Reinforcement Learning},
  author    = {Chen, Jinglin and Jiang, Nan},
  booktitle = {International Conference on Machine Learning},
  year      = {2019}
}

@inproceedings{xie2020qstar,
  title     = {{Q*} Approximation Schemes for Batch Reinforcement Learning: A Theoretical Comparison},
  author    = {Xie, Tengyang and Jiang, Nan},
  booktitle = {Conference on Uncertainty in Artificial Intelligence},
  year      = {2020}
}

@inproceedings{jin2021pessimism,
  title     = {Is Pessimism Provably Efficient for Offline {RL}?},
  author    = {Jin, Ying and Yang, Zhuoran and Wang, Zhaoran},
  booktitle = {International Conference on Machine Learning},
  year      = {2021}
}

@inproceedings{foster2022offline,
  title     = {Offline Reinforcement Learning: Fundamental Barriers for Value Function Approximation},
  author    = {Foster, Dylan J. and Krishnamurthy, Akshay and Simchi-Levi, David and Xu, Yunzong},
  booktitle = {Conference on Learning Theory},
  year      = {2022}
}

@article{jiang2025offlinerl,
  title   = {Offline Reinforcement Learning in Large State Spaces: Algorithms and Guarantees},
  author  = {Jiang, Nan and Xie, Tengyang},
  journal = {Statistical Science},
  year    = {2025}
}

@inproceedings{xie2023coverage,
  title     = {The Role of Coverage in Online Reinforcement Learning},
  author    = {Xie, Tengyang and Foster, Dylan J. and Bai, Yu and Jiang, Nan and Kakade, Sham M.},
  booktitle = {International Conference on Learning Representations},
  year      = {2023}
}

@inproceedings{jin2025oodforgetting,
  title     = {{RL} Fine-Tuning Heals the {OOD} Forgetting in {SFT}},
  author    = {Jin, Hangzhan and Luan, Sitao and Lyu, Sicheng and Rabusseau, Guillaume and Precup, Doina and Hamdaqa, Mohammad},
  booktitle = {Workshop on Foundations of Reasoning in Language Models},
  year      = {2025}
}

@inproceedings{wei2024selfcodealign,
  title     = {SelfCodeAlign: Self-Alignment for Code Generation},
  author    = {Yuxiang Wei and Federico Cassano and Jiawei Liu and Yifeng Ding and Naman Jain and Zachary Mueller and Harm de Vries and Leandro Von Werra and Arjun Guha and LINGMING ZHANG},
  booktitle = {Advances in Neural Information Processing Systems},
  year      = {2024}
}

@misc{ahmad2025opencodeinstruct,
  title        = {OpenCodeInstruct: A Large-scale Instruction Tuning Dataset for Code LLMs},
  author       = {Wasi Uddin Ahmad and Aleksander Ficek and Mehrzad Samadi and Jocelyn Huang and Vahid Noroozi and Somshubra Majumdar and Boris Ginsburg},
  howpublished = {arXiv:2504.04030},
  year         = {2025}
}

@inproceedings{wei2024magicoder,
  title     = {Magicoder: Empowering Code Generation with {OSS}-Instruct},
  author    = {Yuxiang Wei and Zhe Wang and Jiawei Liu and Yifeng Ding and LINGMING ZHANG},
  booktitle = {International Conference on Machine Learning},
  year      = {2024}
}

@inproceedings{gu2025olmes,
  title     = {{OLMES}: A Standard for Language Model Evaluations},
  author    = {Gu, Yuling and Tafjord, Oyvind and Kuehl, Bailey and Haddad, Dany and Dodge, Jesse and Hajishirzi, Hannaneh},
  booktitle = {Findings of the Association for Computational Linguistics},
  year      = {2025}
}

@inproceedings{toshniwal2024openmathinstruct,
  title     = {OpenMathInstruct-2: Accelerating {AI} for Math with Massive Open-Source Instruction Data},
  author    = {Shubham Toshniwal and Wei Du and Ivan Moshkov and Branislav Kisacanin and Alexan Ayrapetyan and Igor Gitman},
  booktitle = {Workshop on Mathematical Reasoning and AI},
  year      = {2024}
}

@misc{olmo2026olmo3,
  title        = {Olmo 3},
  author       = {{Team Olmo} and {Allyson Ettinger} and {Amanda Bertsch} and {Bailey Kuehl} and {David Graham} and {David Heineman} and {Dirk Groeneveld} and {Faeze Brahman} and {Finbarr Timbers} and {Hamish Ivison} and {Jacob Morrison} and {Jake Poznanski} and {Kyle Lo} and {Luca Soldaini} and {Matt Jordan} and {Mayee Chen} and {Michael Noukhovitch} and {Nathan Lambert} and {Pete Walsh} and {Pradeep Dasigi} and {Robert Berry} and {Saumya Malik} and {Saurabh Shah} and {Scott Geng} and {Shane Arora} and {Shashank Gupta} and {Taira Anderson} and {Teng Xiao} and {Tyler Murray} and {Tyler Romero} and {Victoria Graf} and {Akari Asai} and {Akshita Bhagia} and {Alexander Wettig} and {Alisa Liu} and {Aman Rangapur} and {Chloe Anastasiades} and {Costa Huang} and {Dustin Schwenk} and {Harsh Trivedi} and {Ian Magnusson} and {Jaron Lochner} and {Jiacheng Liu} and {Lester James V. Miranda} and {Maarten Sap} and {Malia Morgan} and {Michael Schmitz} and {Michal Guerquin} and {Michael Wilson} and {Regan Huff} and {Ronan Le Bras} and {Rui Xin} and {Rulin Shao} and {Sam Skjonsberg} and {Shannon Zejiang Shen} and {Shuyue Stella Li} and {Tucker Wilde} and {Valentina Pyatkin} and {Will Merrill} and {Yapei Chang} and {Yuling Gu} and {Zhiyuan Zeng} and {Ashish Sabharwal} and {Luke Zettlemoyer} and {Pang Wei Koh} and {Ali Farhadi} and {Noah A. Smith} and {Hannaneh Hajishirzi}},
  howpublished = {arXiv:2512.13961},
  year         = {2026}
}

@inproceedings{pang2025tokencleaning,
  title     = {Token Cleaning: Fine-Grained Data Selection for {LLM} Supervised Fine-Tuning},
  author    = {Pang, Jinlong and Di, Na and Zhu, Zhaowei and Wei, Jiaheng and Cheng, Hao and Qian, Chen and Liu, Yang},
  booktitle = {International Conference on Machine Learning},
  year      = {2025}
}

@inproceedings{huang2025sharpening,
  title     = {Self-Improvement in Language Models: The Sharpening Mechanism},
  author    = {Huang, Audrey and Block, Adam and Foster, Dylan and Rohatgi, Dhruv and Zhang, Cyril and Simchowitz, Max and Ash, Jordan and Krishnamurthy, Akshay},
  booktitle = {International Conference on Learning Representations},
  year      = {2025}
}

@inproceedings{karan2025reasoning,
  title     = {Reasoning with Sampling: Your Base Model is Smarter Than You Think},
  author    = {Karan, Aayush and Du, Yilun},
  booktitle = {International Conference on Learning Representations},
  year      = {2026}
}

@inproceedings{cheng2026isocompute,
  title     = {{IsoCompute} Playbook: Optimally Scaling Sampling Compute for {LLM} {RL}},
  author    = {Cheng, Zhoujun and Xie, Yutao and Qu, Yuxiao and Setlur, Amrith and Hao, Shibo and Pimpalkhute, Varad and Liang, Tongtong and Yao, Feng and Liu, Zhengzhong and Xing, Eric P. and Smith, Virginia and Salakhutdinov, Ruslan and Hu, Zhiting and Killian, Taylor W. and Kumar, Aviral},
  booktitle = {International Conference on Machine Learning},
  year      = {2026}
}

@inproceedings{he2025rewarding,
  title     = {Rewarding the Unlikely: Lifting {GRPO} Beyond Distribution Sharpening},
  author    = {He, Andre Wang and Fried, Daniel and Welleck, Sean},
  booktitle = {Conference on Empirical Methods in Natural Language Processing},
  year      = {2025}
}

@inproceedings{chen2026rethinking,
  title     = {Rethinking Fine-Tuning when Scaling Test-Time Compute: Limiting Confidence Improves Mathematical Reasoning},
  author    = {Chen, Feng and Ravent{\'o}s, Allan and Cheng, Nan and Ganguli, Surya and Druckmann, Shaul},
  booktitle = {Advances in Neural Information Processing Systems},
  year      = {2025}
}

@inproceedings{katharopoulos2018notall,
  title     = {Not All Samples Are Created Equal: Deep Learning with Importance Sampling},
  author    = {Katharopoulos, Angelos and Fleuret, Francois},
  booktitle = {International Conference on Machine Learning},
  year      = {2018}
}

@inproceedings{swayamdipta2020dataset,
  title     = {Dataset Cartography: Mapping and Diagnosing Datasets with Training Dynamics},
  author    = {Swayamdipta, Swabha and Schwartz, Roy and Lourie, Nicholas and Wang, Yizhong and Hajishirzi, Hannaneh and Smith, Noah A. and Choi, Yejin},
  booktitle = {Conference on Empirical Methods in Natural Language Processing},
  year      = {2020}
}

@inproceedings{brandfonbrener2024colorfilter,
  title     = {{CoLoR-Filter}: Conditional Loss Reduction Filtering for Targeted Language Model Pre-training},
  author    = {Brandfonbrener, David and Zhang, Hanlin and Kirsch, Andreas and Schwarz, Jonathan Richard and Kakade, Sham},
  booktitle = {Advances in Neural Information Processing Systems},
  year      = {2024}
}

@inproceedings{qin2024infobatch,
  title     = {{InfoBatch}: Lossless Training Speed Up by Unbiased Dynamic Data Pruning},
  author    = {Qin, Ziheng and Wang, Kai and Zheng, Zangwei and Gu, Jianyang and Peng, Xiangyu and Xu, Zhaopan and Zhou, Daquan and Shang, Lei and Sun, Baigui and Xie, Xuansong and You, Yang},
  booktitle = {International Conference on Learning Representations},
  year      = {2024}
}

@inproceedings{wang2024greats,
  title     = {{GREATS}: Online Selection of High-Quality Data for {LLM} Training in Every Iteration},
  author    = {Wang, Jiachen T. and Wu, Tong and Song, Dawn and Mittal, Prateek and Jia, Ruoxi},
  booktitle = {Advances in Neural Information Processing Systems},
  year      = {2024}
}

@inproceedings{xia2024less,
  title     = {{LESS}: Selecting Influential Data for Targeted Instruction Tuning},
  author    = {Xia, Mengzhou and Malladi, Sadhika and Gururangan, Suchin and Arora, Sanjeev and Chen, Danqi},
  booktitle = {International Conference on Machine Learning},
  year      = {2024}
}

@inproceedings{li2025preserving,
  title     = {Preserving Diversity in Supervised Fine-Tuning of Large Language Models},
  author    = {Li, Ziniu and Chen, Congliang and Xu, Tian and Qin, Zeyu and Xiao, Jiancong and Luo, Zhi-Quan and Sun, Ruoyu},
  booktitle = {International Conference on Learning Representations},
  year      = {2025}
}

@inproceedings{zhao2026bounded,
  title     = {Don't Force the Fit: Bounded Log-Likelihood Loss for Enhanced Reasoning in Large Language Models},
  author    = {Zhao, Feng and Zhang, Hong and Yang, Yu and Zhao, Ruilin and Xu, Guandong},
  booktitle = {International Conference on Machine Learning},
  year      = {2026}
}

@inproceedings{wang2026curiosft,
  title     = {Learning While Staying Curious: Entropy-Preserving Supervised Fine-Tuning via Adaptive Self-Distillation for Large Reasoning Models},
  author    = {Wang, Hao and Gu, Hao and Piao, Hongming and Gong, Kaixiong and Ye, Yuxiao and Yue, Xiangyu and Han, Sirui and Guo, Yike and Wu, Dapeng},
  booktitle = {Annual Meeting of the Association for Computational Linguistics},
  year      = {2026}
}

@inproceedings{chen2026sedsft,
  title     = {{SED-SFT}: Selectively Encouraging Diversity in Supervised Fine-Tuning},
  author    = {Chen, Yijie and Liu, Yijin and Meng, Fandong},
  booktitle = {Annual Meeting of the Association for Computational Linguistics},
  year      = {2026}
}

@inproceedings{fan2026prefix,
  title     = {Learning Diverse Responses with Prefix-Conditioned Supervised Fine-Tuning},
  author    = {Fan, Zhiyuan and Chen, Guanqiao and Huang, Yanyi and Zhao, Mingkuan and Guo, Dadi and Fung, Yi R.},
  booktitle = {Annual Meeting of the Association for Computational Linguistics},
  year      = {2026}
}

@inproceedings{huang2025bestofn,
  title     = {Is Best-of-N the Best of Them? Coverage, Scaling, and Optimality in Inference-Time Alignment},
  author    = {Huang, Audrey and Block, Adam and Liu, Qinghua and Jiang, Nan and Krishnamurthy, Akshay and Foster, Dylan J.},
  booktitle = {International Conference on Machine Learning},
  year      = {2025}
}

@inproceedings{chow2025inferenceaware,
  title     = {Inference-Aware Fine-Tuning for Best-of-N Sampling in Large Language Models},
  author    = {Chow, Yinlam and Tennenholtz, Guy and Gur, Izzeddin and Zhuang, Vincent and Dai, Bo and Kumar, Aviral and Agarwal, Rishabh and Thiagarajan, Sridhar and Boutilier, Craig and Faust, Aleksandra},
  booktitle = {International Conference on Learning Representations},
  year      = {2025}
}

@inproceedings{snell2025scaling,
  title     = {Scaling {LLM} Test-Time Compute Optimally Can be More Effective than Scaling Parameters for Reasoning},
  author    = {Snell, Charlie and Lee, Jaehoon and Xu, Kelvin and Kumar, Aviral},
  booktitle = {International Conference on Learning Representations},
  year      = {2025}
}

@inproceedings{wu2026modeconditioning,
  title     = {Mode-conditioning unlocks superior test-time compute scaling},
  author    = {Wu, Chen and Goyal, Sachin and Raghunathan, Aditi},
  booktitle = {International Conference on Learning Representations},
  year      = {2026}
}

@inproceedings{liu2025prorl,
  title     = {{ProRL}: Prolonged Reinforcement Learning Expands Reasoning Boundaries in Large Language Models},
  author    = {Liu, Mingjie and Diao, Shizhe and Lu, Ximing and Hu, Jian and Dong, Xin and Choi, Yejin and Kautz, Jan and Dong, Yi},
  booktitle = {Advances in Neural Information Processing Systems},
  year      = {2025}
}

@inproceedings{tuyls2026representation,
  title     = {Representation-Based Exploration for Language Models: From Test-Time to Post-Training},
  author    = {Tuyls, Jens and Foster, Dylan and Krishnamurthy, Akshay and Ash, Jordan},
  booktitle = {International Conference on Learning Representations},
  year      = {2026}
}

@inproceedings{yao2025diversityaware,
  title     = {Diversity-Aware Policy Optimization for Large Language Model Reasoning},
  author    = {Yao, Jian and Cheng, Ran and Wu, Xingyu and Wu, Jibin and Tan, KC},
  booktitle = {Advances in Neural Information Processing Systems},
  year      = {2025}
}

@inproceedings{yano2026pretraining,
  title     = {Pre-training {LLM} without Learning Rate Decay Enhances Supervised Fine-Tuning},
  author    = {Yano, Kazuki and Kiyono, Shun and Kobayashi, Sosuke and Takase, Sho and Suzuki, Jun},
  booktitle = {International Conference on Learning Representations},
  year      = {2026}
}

@inproceedings{chu2025sftmemorizes,
  title     = {{SFT} Memorizes, {RL} Generalizes: A Comparative Study of Foundation Model Post-Training},
  author    = {Chu, Tianzhe and Zhai, Yuexiang and Yang, Jihan and Tong, Shengbang and Xie, Saining and Schuurmans, Dale and Le, Quoc V. and Levine, Sergey and Ma, Yi},
  booktitle = {International Conference on Machine Learning},
  year      = {2025}
}

@inproceedings{mindermann2022prioritized,
  title     = {Prioritized training on points that are learnable, worth learning, and not yet learnt},
  author    = {Mindermann, S{\"o}ren and Brauner, Jan M and Razzak, Muhammed T and Sharma, Mrinank and Kirsch, Andreas and Xu, Winnie and H{\"o}ltgen, Benedikt and Gomez, Aidan N and Morisot, Adrien and Farquhar, Sebastian and others},
  booktitle = {International Conference on Machine Learning},
  year      = {2022}
}

@inproceedings{thirukovalluru2024sequence,
  title     = {Sequence reducible holdout loss for language model pretraining},
  author    = {Thirukovalluru, Raghuveer and Monath, Nicholas and Dhingra, Bhuwan and Wiseman, Sam},
  booktitle = {Joint International Conference on Computational Linguistics, Language Resources and Evaluation},
  year      = {2024}
}

@article{team2026kimi,
  title={Kimi k3: Open frontier intelligence},
  author={Kimi Team and Tongtong Bai and Yifan Bai and Yiping Bao and M. C. and Jianfeng Cai and Xinyuan Cai and Peizhou Cao and Yuxuan Cao and Ziwei Chai and Y. Charles and H. S. Che and Guanduo Chen and Guangyu Chen and Guanzheng Chen and Huarong Chen and Jia Chen and Jianlong Chen and Jun Chen and Kexin Chen and Peng Chen and Ruijue Chen and Wentao Chen and Xin Chen and Yang Chen and Yanru Chen and Yifei Chen and Yingjiang Chen and Yuankun Chen and Yujie Chen and Yutian Chen and Zhirong Chen and Dazhi Cheng and Yean Cheng and Jialei Cui and Jingbing Cui and Anqi Dai and Jiaqi Deng and Hao Ding and Rui Ding and Shaofeng Ding and Mengfan Dong and Mengnan Dong and Yuhao Dong and Yuxin Dong and Angang Du and Chenzhuang Du and Dikang Du and Jusen Du and Yulun Du and Yu Fan and Jing Feng and Qiulin Feng and Yichen Feng and Kelin Fu and Qiang Fu and Fuxuan Gao and Hongcheng Gao and Jingyue Gao and Tong Gao and Weijia Gao and Shangyi Geng and Jie Gong and Linhu Gong and Shengao Gong and Xiaochen Gong and Qizheng Gu and Yicheng Gu and Shuhao Guan and Haiqing Guo and Shiqi Guo and Xiang Guo and Zhengyan Guo and Beixi Hao and Wenxin Hao and Xiaoru Hao and Dailan He and Haotian He and Lehan He and Qi He and Weiran He and Xinran He and Xinyi He and Yibo He and Yunjia He and Chao Hong and Tiange Hong and Hao Hu and Jiaxi Hu and Ruikun Hu and Weiming Hu and Yangyang Hu and Zhenxing Hu and Liang Hua and Jinbin Huang and Ke Huang and Ruiyuan Huang and Siying Huang and Weixiao Huang and Yan Huang and Zhengjie Huang and Zhiqi Huang and Yulong Hui and Chaobo Jia and Yutong Jiang and Zhejun Jiang and Zuoyou Jiang and Wenyi Jin and Xinyi Jin and Yu Jing and Huanjun Kong and Guokun Lai and Aidi Li and Cheng Li and Chengyuan Li and Cong Li and Fang Li and Guanyu Li and Haoyang Li and Jia Li and Junxiong Li and Lei Li and Letian Li and Lincan Li and Weihong Li and Wentao Li and Xintong Li and Yang Li and Yishen Li and Yiwei Li and Yuxiao Li and Zhaowei Li and Zhaoxi Li and Zheming Li and Zhengxiao Li and Zhiyuan Li and Jiawei Lin and Xiaohan Lin and Yibo Lin and Zichao Lin and Ziyan Lin and Bill Liu and Boxiao Liu and Chuan Liu and Liang Liu and Shaowei Liu and Shudong Liu and Shuran Liu and Tianwei Liu and Weizhou Liu and Yangyang Liu and Yanming Liu and Yibo Liu and Yipeng Liu and Zhengying Liu and Zhiheng Liu and Enzhe Lu and Haoyu Lu and Linqiang Lu and Tingzhan Lu and Zhiyuan Lu and Aotian Luo and G. Luo and Junyu Luo and Yifan Luo and B. Lyu and Wenzhou Lyu and Shaoguang Mao and Yuan Mei and Xin Men and Minqing Ni and Yixuan Niu and Siyuan Pan and Shujun Peng and Zhangyang Qi and Ruoyu Qin and ZeChao Qin and Zeyu Qin and Haiquan Qiu and Jianxin Qiu and Jiezhong Qiu and Bowen Qu and Yuhao Qu and Zeyu Shang and Youbo Shao and Han Shen and Jincheng Shi and Juanfeng Shi and Lidong Shi and Shengyuan Shi and Wingchun Siu and Pengwei Song and Xiaoxi Song and Jianlin Su and Yunfeng Su and Zhaochen Su and Lin Sui and Jingsong Sun and Junyao Sun and Shaoning Sun and Shuzhe Sun and Tongyu Sun and Yujun Sun and Yunpeng Tai and Chuning Tang and Heyi Tang and Sirui Tang and Zecheng Tang and Chaoran Tian and Rongpeng Tian and Yu Tian and Wei Tu and Chensi Wang and Chuang Wang and Chunjie Wang and Dinglu Wang and Feng Wang and Hailong Wang and Haiming Wang and Hao Wang and Hao Wang and Huaqing Wang and Hui Wang and Jiayi Wang and Jinglong Wang and Jinhong Wang and Jiuzheng Wang and Linian Wang and Shaobo Wang and Shenzhi Wang and Shuyi Wang and Si Wang and Siyuan Wang and Tianfu Wang and Wenjue Wang and Xingran Wang and Xinmei Wang and Xinyuan Wang and Xusheng Wang and Yalin Wang and Yangkun Wang and Yao Wang and Yaoyu Wang and Yejie Wang and Yiqin Wang and Yucheng Wang and Yuzhi Wang and Zhaoji Wang and Zhaowei Wang and Zhengtao Wang and Zhenhao Wang and Zhongsheng Wang and Zifan Wang and Chu Wei and Ming Wei and Shouxin Wei and Zichen Wen and Fan Wu and Haoning Wu and Rucong Wu and Wenhao Wu and Xiaoxue Wu and Yingcong Wu and Yongqi Wu and Yuxin Wu and Zijian Wu and Xinglang Xian and Chenxuan Xiang and Yuye Xiang and Bocheng Xiao and Chenjun Xiao and Xin Xiao and Jin Xie and Xiaotong Xie and Yifeng Xie and Zhe Xie and Bowei Xing and Yiming Xiong and Baosheng Xu and Boyu Xu and Jiale Xu and Jianfan Xu and Jing Xu and Jinjing Xu and L. H. Xu and Qingtao Xu and Shuyao Xu and Suting Xu and Tiantian Xu and Tianxiang Xu and Weixin Xu and Xinran Xu and Yangchuan Xu and Ye Xu and Yueni Xu and Ziyao Xu and Haonan Xue and Junjie Yan and Yaoyao Yan and Fan Yang and Guangyao Yang and Hao Yang and Junwei Yang and Ruoyu Yang and Wenjie Yang and Xiaofei Yang and Xinyu Yang and Yi Yang and Yiling Yang and Ying Yang and Yuchen Yang and Zhen Yang and Zhilin Yang and Zian Yang and Zuhao Yang and Haotian Yao and Dan Ye and Haoran Ye and Wenjie Ye and Zhanbo Ye and Bohong Yin and Haoxiang Yin and Xietong Yin and Chengzhen Yu and Haozhen Yu and Longhui Yu and Shengnan Yu and Shuying Yu and Tianxiang Yu and Enming Yuan and Mengjie Yuan and Tongtian Yue and Wei Yue and Yang Yue and Dunyuan Zha and Haobing Zhan and B. H. Zhang and Dehao Zhang and Fei Zhang and Hao Zhang and Haoyuan Zhang and Huanyu Zhang and Jiapei Zhang and Jiaxuan Zhang and Jin Zhang and Kaiyi Zhang and Miaozhen Zhang and Puqi Zhang and Qinglei Zhang and Rong Zhang and Rui Zhang and Shaoshuai Zhang and Shiyi Zhang and Xiaobin Zhang and Xiaoyun Zhang and Y. Zhang and Yangkun Zhang and Ye Zhang and Yichi Zhang and Yikun Zhang and Yizhi Zhang and Yongting Zhang and Yu Zhang and Yutao Zhang and Yutong Zhang and Zheng Zhang and Zijing Zhang and Bin Zhao and Chenguang Zhao and Feifan Zhao and Jinglun Zhao and Jinxiang Zhao and Shuai Zhao and Wenshuo Zhao and Xiangyu Zhao and Xuanle Zhao and Yikai Zhao and Zijia Zhao and Haozhi Zheng and Huabin Zheng and Ruihan Zheng and Shaojie Zheng and Tengyang Zheng and Haofeng Zhong and Lei Zhong and Longguang Zhong and M. Zhou and Qiankang Zhou and Runjie Zhou and Ruozhang Zhou and Xinyu Zhou and Yiqiao Zhou and Zaida Zhou and Jinguo Zhu and Liya Zhu and Xinhao Zhu and Yangjunfeng Zhu and Yuxuan Zhu and Zhen Zhu and Chen Zhuang and Weiyu Zhuang and Xinxing Zu},
  journal={arXiv preprint arXiv:2607.24653},
  year={2026}
}

@techreport{microsoft2025mai,
  title={Mai-thinking-1: Building a hill-climbing machine},
  author={{Microsoft AI Team}},
  year={2025},
  institution={Technical report, Microsoft AI, 2026. https://microsoft. ai/pdf/mai-thinking~…}
}

@article{zhao2026absolute,
  title={Absolute zero: Reinforced self-play reasoning with zero data},
  author={Zhao, Andrew and Wu, Yiran and Wu, Tong and Xu, Quentin and Yue, Yang and Lin, Matthieu and Wang, Shenzhi and Wu, Qingyun and Zheng, Zilong and Huang, Gao},
  journal={Advances in Neural Information Processing Systems},
  volume={38},
  pages={105816--105879},
  year={2026}
}


\clearpage

\appendix

\section{Additional Related Work}
\label{app:related_work}

\paragraph{Data selection and diversity-preserving SFT}
Importance sampling and dynamic pruning prioritize high-gradient or
low-information examples to accelerate optimization while preserving the
full-data objective
\citep{katharopoulos2018notall,qin2024infobatch}; TailSFT instead changes the
effective SFT objective to improve coverage and post-RL performance. RHO-LOSS
selects learnable, not-yet-learned examples and Dataset Cartography diagnoses
examples from training dynamics
\citep{mindermann2022prioritized,swayamdipta2020dataset}; TailSFT uses only loss
reduction relative to the initial policy and masks the most-improved sequences
online. Rho-1 and CoLoR use reference-relative losses in pre-training, Token
Cleaning selects SFT tokens, and GREATS and LESS estimate utility or influence
\citep{lin2024rho1,brandfonbrener2024colorfilter,pang2025tokencleaning,
wang2024greats,xia2024less}; TailSFT operates at sequence or document
granularity without target examples or influence computation.
\citet{qin2025curatedsft} cast curated SFT as implicit RL; TailSFT supplies a
dynamic, coverage-motivated rule and evaluates an explicit later RL stage.
Confidence caps, entropy regularization, self-distillation, and token-level
clipping preserve diversity during SFT
\citep{chen2026rethinking,li2025preserving,zhao2026bounded,
wang2026curiosft,chen2026sedsft}, while prefix-conditioned SFT separates
response modes \citep{fan2026prefix}; TailSFT instead masks already-improved
examples relative to the base policy, without auxiliary losses or mode labels.

\paragraph{Coupling supervised and reinforcement training}
\citet{bansal2026rl} apply RL at intermediate pre-training checkpoints and
blend SFT and RL updates, \citet{fu2026srft} jointly weight demonstrations and
on-policy rollouts, \citet{huang2026prefixrft} continue expert prefixes
on-policy, and \citet{zhang2026chord} retain expert SFT as a dynamically
weighted RL auxiliary objective. These methods change when or how the
objectives are coupled; TailSFT leaves the two-stage pipeline and RL algorithm
unchanged and changes only which SFT examples receive gradient.
\citet{yoshihara2025twostage} use prolonged SFT to maximize mathematical
accuracy and then GRPO chiefly to improve token efficiency; TailSFT instead
optimizes the SFT stage for the final post-RL model, even when local SFT
metrics worsen. \citet{niu2026nondecoupling} derive interference between
sequential SFT and RL objectives under their assumptions; TailSFT keeps the
stages sequential but reshapes SFT gradients to better support the later RL
objective.

\paragraph{Coverage and test-time compute}
The coverage principle formalizes response-level likelihood-ratio coverage as
the condition governing Best-of-$N$ recovery
\citep{chen2025coverageprinciple}; TailSFT operationalizes this criterion as an
online SFT rule and proves an advantage for relative clipping. Classical RL
theory uses concentrability and related coverage coefficients to control error
propagation and offline or online learnability
\citep{farahmand2010error,chen2019information,xie2020qstar,
jin2021pessimism,foster2022offline,jiang2025offlinerl,xie2023coverage}; those
works ask whether logged state-action data cover a target policy, whereas
TailSFT asks whether a language model covers expert responses at a finite
sampling budget. Repeated sampling yields predictable coverage gains
\citep{brown2024monkeys}, while compute allocation, Best-of-$N$ analysis,
inference-aware fine-tuning, mode conditioning, and rollout-budget scaling
improve how a fixed policy is sampled or trained for sampling
\citep{snell2025scaling,huang2025bestofn,chow2025inferenceaware,
wu2026modeconditioning,cheng2026isocompute}; TailSFT instead improves the
pre-RL policy without a verifier, test-time mode label, or altered rollout
allocation.

\paragraph{RLVR, sharpening, and exploration}
Sharpening theory formalizes post-training as reallocating mass within covered
behaviors, while sampling analyses show that capabilities can already be
latent in the base policy \citep{huang2025sharpening,karan2025reasoning}; TailSFT
builds on this view by improving useful coverage before RL. Empirically, RLVR
can raise pass@1 without expanding large-$k$ capacity, remain
support-constrained, or amplify earlier behaviors
\citep{yue2025beyondbase,wu2025invisibleleash,zhao2025echochamber}, while
\citet{wen2025correctreasoning} show that verifiable rewards can promote
correct reasoning latent in the base model; TailSFT changes SFT so more
reward-bearing responses remain sampleable. RL-stage methods upweight unlikely
correct trajectories, optimize solution diversity, or add
representation-based exploration
\citep{he2025rewarding,yao2025diversityaware,tuyls2026representation}; these
methods are complementary, whereas TailSFT uses supervised data alone.
\citet{liu2025prorl} show that prolonged, diverse RL can expand the reasoning
boundary, and \citet{foster2025goodfoundation} show that base-policy coverage
controls efficient exploration; TailSFT does not posit an absolute boundary,
but supplies a better foundation for standard on-policy RL.

\paragraph{Stage-aware model development}
\citet{kang2026quagmires} show that high SFT scores need not predict post-RL
performance and identify held-out generalization loss and pass@large-$k$ as
stronger proxies; TailSFT goes beyond diagnosis by modifying SFT to improve
coverage and validating the resulting initialization through matched RL runs.
Models with the same pre-training loss can transfer differently because of
optimization's implicit bias \citep{liu2023implicitbias}; pre-training
perplexity can misrank fine-tuning outcomes \citep{zeng2025indicators}, and
downstream scaling laws can be unstable \citep{lourie2025scaling}. These works
expose failures of local metrics; TailSFT supplies an actionable SFT
intervention and a coverage-based diagnostic for the SFT-to-RL transition.
Extended pre-training can reduce loss while harming adaptability, alternative
schedules can improve SFT despite weaker local metrics, and flatter
pre-training solutions can reduce forgetting after post-training
\citep{springer2025overtraining,yano2026pretraining,watts2026sharpness};
TailSFT makes the analogous intervention inside SFT and targets response
coverage rather than pre-training geometry. \citet{zhang2025interplay} and
\citet{wang2025octothinker} show that pre- or mid-training exposure controls
later RL gains; TailSFT isolates a lightweight SFT intervention and tests it
through matched RL runs. \citet{chu2025sftmemorizes} separate SFT's stabilizing
role from RL's generalization, while \citet{jin2025oodforgetting} find that OOD
performance can peak early in SFT and be partly restored by RL; TailSFT filters
already-fit examples to prevent coverage loss before RL.

\newcommand{\piref}{\pi_{\mathrm{ref}}}
\newcommand{\one}{\mathbf{1}}
\newcommand{\lerm}{L_{\mathrm{ERM}}}
\newcommand{\pierm}{\pi_{\mathrm{ERM}}}
\newcommand{\labs}{L_{\mathrm{ABS},\alpha}}
\newcommand{\piabs}{\pi_{\mathrm{ABS},\alpha}}
\newcommand{\loff}{L_{\mathrm{OFF},\beta}}
\newcommand{\pioff}{\pi_{\mathrm{OFF},\beta}}
\newcommand{\KL}[2]{\mathrm{KL}(#1 \|\| #2)}

\section{Theoretical Analysis: TailSFT in the expert conditioning setting}\label{app:theory}
We consider a stylized setup for supervised fine-tuning called \emph{expert conditioning} where the SFT data distribution and optimal policy for the downstream task are defined by conditioning the pretrained model output distribution on some event. For theoretical analysis, we focus on a simplified setting where there is no context or prompt, and so all policies are elementary distributions over responses $y \in \mathcal{Y}$. Let 
$\piref \in \Delta(\mathcal{Y})$ be the pretrained model and let $S \subset \Ycal$ be some subset of responses. In expert conditioning, we define the expert policy as $\pi^\star(y) \propto \piref(y) \cdot \one\{y \in S\}$. Note that this setup is closely related to formulations of RLHF where the expert is defined as $\piref(y)\cdot \exp(R^\star(y)/\beta)$ for some reward function $R^\star$; formally, expert conditioning is equivalent to this formulation in the limit where  $\beta \to 0$. 

We are given $n$ samples $y_1,\ldots,y_n \sim \pi^\star$ and obtain a policy by solving a certain optimization problem defined in terms of a loss function over the sample. The three loss functions are
\begin{align*}
    \mathrm{ERM}: &~~~~~~~~ \lerm(\pi) := \frac{1}{n}\sum_{i=1}^n - \log (\pi(y_i))\\
    \mathrm{ABS}: &~~~~~~~~ \labs(\pi) := \frac{1}{n}\sum_{i=1}^n \max(-\log (\pi(y_i)) + \log(\alpha),0)\\
    \mathrm{OFF}: &~~~~~~~~ \loff(\pi) := \frac{1}{n}\sum_{i=1}^n \max(-\log (\pi(y_i)) + \log(\beta\cdot \piref(y_i)),0)   
\end{align*}
The first objective is standard empirical risk minimization on the cross entropy loss. The second and third are TailSFT variants with absolute and relative clipping respectively. Indeed for absolute clipping, we ignore sample $y_i$ if $\pi(y_i) \geq \alpha$ and for relative clipping we ignore $y_i$ if $\pi(y_i) \geq \beta \piref(y_i)$. The formal optimization problem for a particular loss function is:
\begin{align}
    \minimize_{\pi\in\Delta(\mathcal{Y})}~ \KL{\pi}{\piref}~ \textrm{ subject to }~ L(\pi) = \min_{\pi'} L(\pi') \label{eq:theory_main_opt}
\end{align}
Thus we minimize the forward Kullback-Leibler (KL) divergence between the optimization variable $\pi$ and the reference policy $\piref$ subject to $\pi$ being among the minimizers of the loss $L$. We use forward KL as simplification of the implicit bias of optimization, and use forward KL primarily due to its mode-covering behavior, so that coverage is preserved to the extent possible subject to minimizing the loss. We take $L$ to be $\lerm,\labs,\loff$ accordingly. 

Formally, fixing the dataset, define $\pierm,\piabs,\pioff$ to be the optimizers of~\cref{eq:theory_main_opt} with loss functions $\lerm, \labs,\loff$ respectively. We are interested in qualitatively understanding  how well these optimizers cover the expert policy $\pi^\star$, where recall that we define coverage as:
\begin{align*}
    \Pcov_N(\pi) := \Pr_{y \sim \pi^\star}\left[\frac{\pi^\star(y)}{\pi(y)} \geq N\right]
\end{align*}

Our main result is as follows:
\begin{theorem}
\label{thm:main}
    We have the following comparisons:
    \begin{enumerate}
        \item \textbf{Offset clipping is always preferred}: For all $\piref, S$ and datasets:
        \begin{align*}
        \inf_\beta \Pcov_N(\pioff) \leq \inf_{\beta: \min_{\pi}\loff(\pi) = 0} \Pcov_N(\pioff) \leq \left(\inf_{\alpha} \Pcov_N(\piabs)\right) \wedge \Pcov_N(\pierm)
        \end{align*}
        \item \textbf{Offset clipping can dominate}: For $|\mathcal{Y}|$ sufficiently large, there exists a reference policy $\piref$ and expert subset $S$ such that with probability at least $1-\mathrm{poly}(|\mathcal{Y}|^{-1})$
        \begin{align*}
        \inf_\beta \Pcov_N(\pioff) < \left(\inf_{\alpha} \Pcov_N(\piabs)\right) \wedge \Pcov_N(\pierm)
        \end{align*}
        \item \textbf{Absolute clipping can be worse than ERM}: There exists a reference policy $\piref$ and expert subset $S$ such that with high probability
        \begin{align*}
        \Pcov_N(\pierm) < \inf_{\alpha: \min_\pi \labs(\pi) = 0} \Pcov_N(\piabs)
        \end{align*}
        That is, in the regime where zero loss is achievable for absolute clipping, we can have that empirical risk minimization strictly dominates absolute loss clipping. 
    \end{enumerate}
\end{theorem}
Note that the first claim implies that offset clipping---even when restricted to the parameter regime where zero loss is achievable--is never worse than empirical risk minimization. 

\subsection{Proof of~\cref{thm:main}}
First we derive a structural characterization of the solutions of~\cref{eq:theory_main_opt}. Next, we use this characterization to understand the coverage of the policies $\pierm,\piabs,\pioff$. Finally we establish the comparisons in the theorem statement. 

\begin{lemma}[Structure of KL projections]
\label{lem:kl_projection}
    Let $\{f_y: y \in \mathcal{Y}\} \subset [0,1]$ satisfy $\sum_y f_y \leq 1$. Consider the optimization problem
    \begin{align*}
        \widehat{\pi} \gets \argmin_{\pi\in\Delta(\mathcal{Y})}~ \KL{\pi}{\piref}~ \mathrm{subject\ to}~ \forall y: \pi(y) \geq f_y.
    \end{align*}
    Then the minimizer $\widehat{\pi}$ is unique and given by
    \begin{align*}
        \widehat{\pi}(y) = \max(f_y,\lambda \piref(y)) \quad \mathrm{where} \quad \sum_y \max(f_y, \lambda \piref(y)) = 1.
    \end{align*}
\end{lemma}
\begin{proof}[Proof of~\cref{lem:kl_projection}]
    Observe that the optimization problem is strictly convex and feasible under the condition that $\sum_y f_y\leq 1$. We write the Lagrangian of the optimization problem as
    \begin{align*}
        \sum_y \pi(y) \log \frac{\pi(y)}{\piref(y)} + \lambda\left(\sum_y \pi(y) - 1\right) + \sum_y \mu_y(f(y) - \pi(y)) \qquad \mu_y \geq 0.
    \end{align*}
    The stationary conditions are
    \begin{align*}
        \forall y:~ \log \frac{\pi(y)}{\piref(y)} + 1 + \lambda - \mu_y = 0.
    \end{align*}
    If $\pi(y) > f_y$, so the constraint is inactive, then by complementary slackness we have $\mu_y=0$, so $\pi(y) = \piref(y) \cdot \exp(-1-\lambda)$ and this must be larger than $f_y$. Otherwise, we must have $\pi(y) = f_y$ solving the stationary condition for $\mu_y$ gives
    \begin{align*}
        \mu_y = \log (\pi(y)/\piref(y)) + 1 + \lambda = \log\frac{f_y}{\piref(y)\cdot\exp(-1-\lambda)}
    \end{align*}
    The constraint that $\mu_y \geq 0$ thus implies that $f_y \geq \piref(y)\cdot\exp(-1-\lambda)$. Taken together, this gives $\pi(y) = \max(f_y,\piref(y)\exp(-1-\lambda))$ and $\lambda$ is chosen to normalize the distribution. 
\end{proof}

\begin{lemma}[Structure of clipped minimizer]
\label{lem:clipped_minimizer}
Let $\{f_y: y \in \mathcal{Y}\}$ let $p \in \Delta(Y)$ and let $\sum_y \mathbf{1}\{p(y)>0\} f_y > 1$. Then the minimizer of
\begin{align*}
    \sum_{y \in \mathcal{Y}} p(y) \max(- \log \pi(y) + \log(f_y), 0)
\end{align*}
is unique and given by
\begin{align*}
    \hat{\pi}(y) = \min(f_y, \lambda p(y)) \quad \mathrm{where} \quad \sum_y \min(f_y, \lambda p(y)) = 1.
\end{align*}
\end{lemma}
\begin{proof}
    First observe that for any $y$ such that $p(y) = 0$ we will have $\pi(y) = 0$ as well, because the conditions on $f_y$ imply that we cannot achieve an objective value of zero, and we can always shift mass from actions with $p(y)=0$ to those with $p(y) \ne 0$ to reduce the loss. 
    Thus, without loss of generality we can assume that $p(y) \ne 0$ which also implies $\pi(y) \ne 0$. Next, we write the Lagrangian:
    \begin{align*}
        \sum_y p(y) \max(-\log \pi(y) + \log (f_y), 0) + \lambda\left(\sum_y \pi(y) - 1\right)
    \end{align*}
    The stationary condition is that $\frac{-p(y)}{\pi(y)} + \lambda = 0$ if $\pi(y) < f_y$. If $\pi(y) = f_y)$ the sub-differential for the loss term is in $[-p(y)/\pi(y), 0]$ and if $\pi(y)>0$ we must have $\lambda=0$ to satisfy stationarity. However, since at least one $y$ must have $\pi(y)< f_y$ we get that $\lambda>0$. Therefore no action can have $\pi(y) > f_y$ and we have that $\pi(y) = \min(f_y,\lambda p(y))$.
\end{proof}

Next we characterize the structure and the coverage of the $\pioff$.
\begin{lemma}[Coverage of $\pioff$]
\label{lem:offset-coverage}
    For any dataset, let $\hat{S} \subseteq S$ be the actions observed in the dataset. We have
    \begin{align*}
        \inf_{\beta \geq 0} \Pcov_N(\pioff) \leq \inf_{\beta: \min_\pi \loff(\pi) = 0} \Pcov_N(\pioff) = \begin{cases}
            0, & N \geq 1/\piref(S)\\
            \frac{\piref(S\setminus\hat{S})}{\piref(S)}, & 1 < N < 1/\piref(S)
        \end{cases}
    \end{align*}
\end{lemma}
\begin{proof}[Proof of~\cref{lem:offset-coverage}]
If $N \geq 1/\piref(S)$ then take $\beta=1$ so that $\piref$ itself is a minimizer of $\loff$, i.e., $\loff(\piref) = 0$. With this choice of $\beta$, clearly $\piref$ is the solution to~\cref{eq:theory_main_opt}. To compute the coverage, observe that for every action $y$, $\pi^\star(y)/\piref(y) = 1/\piref(S) \leq N$. Thus, in this case $\Pcov(\piref) = 0$ and with optimally tuned $\beta$, the KL projection onto the set of offset-loss minimizers preserves this coverage. 

If $1<N<1/\piref(S)$ we first translate the offset loss optimization problem into the form in~\cref{lem:kl_projection} and then compute the normalizing constant $\lambda$ in the distribution. Observe that the offset loss only involves observed actions $y \in \hat{S}$. If $\beta$ is such that $0$ offset loss is feasible, then we have
\begin{align*}
    \loff(\pi) = 0 \Leftrightarrow \forall y \in \hat{S}:~ \pi(y) \geq \beta\piref(y) \textrm{ and } \forall y \notin \hat{S}: \pi(y) \geq 0
\end{align*}
Here the first condition arises from achieving zero loss, while the second condition arises because $\pi$ must be a distribution. Thus the offset loss version of~\cref{eq:theory_main_opt} corresponds to taking $f_y = \beta\piref(y)$ if $y \in \hat{S}$ and $f_y = 0$ otherwise in~\cref{lem:kl_projection}. Next, observing that every action in the minimizing distribution has mass at least $\lambda \piref(y)$ we get that $\lambda \leq 1$. Since $f_y = 0$ for actions $y \notin \hat{S}$, we have that $\pioff(y) = \lambda \piref(y)$ for all $y \notin\hat{S}$. Such actions that are further supported by $\pi^\star$ are uncovered, since
\begin{align*}
    \frac{\pi^\star(y)}{\pioff(y)} \geq \frac{\pi^\star(y)}{\piref(y)} = \frac{1}{\piref(S)}  > N
\end{align*}
The first inequality uses that $\lambda \leq 1$ and the second inequality is by our assumed regime for $N$. Thus we obtain for any feasible $\beta$ we obtain a coverage lower bound
\begin{align*}
    \Pcov_N(\pioff) \geq \frac{\piref(S\setminus\hat{S})}{\piref(S)}.
\end{align*}

To show that this is achievable take $\beta = (\piref(S)N)^{-1} > 1$ and observe that
\begin{align*}
    \beta \piref(S) = 1/N < 1. 
\end{align*}
Thus if we set $\pi(y) = \beta \piref(y)$ for $y\in S$ we achieve 0 loss and have $\sum_y \pi(y) < 1$; we can allocate the remaining mass arbitrarily to preserve 0 offset loss and obtain a distribution. 

Since zero-loss is achievable, the KL projection ensures that $\pioff(y) \geq \beta\piref(y)$ for all $y \in \hat{S}$ as this is required to achieve zero loss. By the choice of $\beta$ this actions are covered. On the other hand, all other actions have $\pioff(y) = \lambda \piref(y) \leq \piref(y)$ using the argument above that $\lambda \leq 1$. These actions are all not covered under the condition that $1/\piref(S)>N$. Thus this choice of $\beta$ achieves coverage exactly $\piref(S\setminus\hat{S})/\piref(S)$. 
\end{proof}

\begin{proof}[Proof of Claim 1 of~\cref{thm:main}]
We first argue about $\pierm$. Note that the ERM is unique and is exactly the empirical distribution over the sample. In particular, we have $\pierm(S \setminus \hat{S}) = 0$, and therefore these actions are always uncovered, so $\Pcov_N(\pierm)\geq \piref(S\setminus\hat{S})/\piref(S)$. 

Next we turn to $\piabs$. Here we consider two cases. First consider the case that $\alpha$ is large enough such that zero loss is not achievable. In this case, if we have a policy $\pi$ that allocates mass to unobserved actions, we can always strictly improve the loss by moving that mass onto some observed action that is not already saturated, and such an action must exist if zero loss is not achievable. Thus the loss minimizer must allocate no mass on unobserved actions. In this case the same reasoning as we applied to analyze the ERM holds. 

Next, consider that $\alpha$ is small enough such that zero loss is achievable. As with the offset loss, we can translate the optimization problem to the form in~\cref{lem:kl_projection} and take $f_y = \alpha\one\{y \in \hat{S}\}$. However, as in the proof of~\cref{lem:offset-coverage}, we still require that $\lambda \leq 1$. As before when $1< N < 1/\piref(S)$, this implies that all unobserved actions in $S$ are uncovered, establishing a lower bound of $\piref(S \setminus\hat{S})/\piref(S)$ in this regime. 
\end{proof}

\begin{proof}[Proof of Claim 2 in~\cref{thm:main}]
    We construct an instance with three types of arms with $|\mathcal{Y}| = 1 + n/4 + 3n$ for some $n \in \mathbb{N}$ such that $\mathcal{Y} = \{o\}\cup H\cup L$ where $|H| = n/4$ and $L = 3n$. Define
    \begin{align*}
        \piref(o) = \frac{1}{2}, \quad \piref(h) = \frac{1}{n}, h \in H, \quad \piref(\ell) = \frac{1}{12n}, \ell \in L.
    \end{align*}
    Set $S = H\cup L$ such that
    \begin{align*}
        \pi^\star(h) = \frac{2}{n}, h \in H, \quad \pi^\star(\ell) = \frac{1}{6n}, \ell \in L
    \end{align*}

    Recall that $\hat{S}$ is the set of actions observed in the sample. The high probability event is the intersection of two events: (1) $|\hat{S}|$ is sufficiently large and (2) $|\hat{H}^{(1)}|$ is sufficiently large, where $\hat{H}^{(1)}$ is the set of actions in $H$ that are observed exactly once. 
    
    To control $|\hat{S}|$, we first calculate:
    \begin{align*}
        \mathbb{E} | \hat{S} \cap H | &= \sum_{h \in H} \Pr[ h \in \hat{S}] = \frac{n}{4} (1 - (1-2/n)^n) \geq \frac{n}{4} (1 - e^{-2})\\
        \mathbb{E} |\hat{S} \cap L| & = \sum_{\ell \in L} \Pr[\ell \in \hat{S}] = 3n (1 - (1-1/(6n))^n) \geq 3n(1 - e^{-1/6})
    \end{align*}
    Observe that $|\hat{S}| = |\hat{S}\cap H| + |\hat{S} \cap L|$ and that the random variable $|\hat{S}|$ satisfies the conditions of McDiarmid's inequality with constant $1$. Thus with probability at least $1-\delta/2$ we have
    \begin{align*}
        |\hat{S}| \geq \mathbb{E}|\hat{S}| - \sqrt{ \frac{n \log(2/\delta)}{2}} \geq \frac{n}{4} (1 - e^{-2}) + 3n(1 - e^{-1/6}) - \sqrt{\frac{n \log(2/\delta)}{2}}.
    \end{align*}
    The constant on the $\Theta(n)$ term is at least $0.676$, and for $\delta = \mathrm{poly}(1/n)$ the $\Theta(n)$ term dominates, and so we have that for $n$ sufficiently large, $|\hat{S}| \geq 2n/3$. 

    We control $|\hat{H}^{(1)}|$ similarly:
    \begin{align*}
        \mathbb{E}|\hat{H}^{(1)}| = \sum_{h \in H} \Pr[B_h = 1] = \frac{n}{4} \left( n \cdot \frac{2}{n} \cdot (1 - 2/n)^{n-1}\right) \geq \frac{n}{2}\cdot(1-2/n)^{n-1}
    \end{align*}
    Here $B_h$ is a Binomially distributed random variable with parameters $(n,2/n)$. Again we apply McDiarmid's inequality, here $|\hat{H}^{(1)}|$ satisfies the conditions with a constant of $2$, so that with probability at least $1-\delta/2$ we have
    \begin{align*}
        |\hat{H}^{(1)}| \geq \mathbb{E}|\hat{H}^{(1)}| - \sqrt{2n \log(2/\delta)} \geq \frac{n}{2}\cdot(1-2/n)^{n-1} - \sqrt{2n \log(2/\delta)}
    \end{align*}
    Observe that the first term asymptotically approaches $n\cdot e^{-2}/2$ while the second term is lower order in $n$. Therefore for $n$ sufficiently large, we can see that $|\hat{H}^{(1)}| \geq e^{-2}/3$ with high probability. 

    Next we turn to the coverage calculations. Since $N=4/3$ satisfies $1 < N < 1/\piref(S) = 2$ we know that $\inf_{\beta\geq 0}\Pcov(\pioff) = \piref(S\setminus\hat{S})/\piref(S)=\pi^\star(S\setminus\hat{S})$. This is exactly the expert mass of the unobserved set. On the other hand $\pierm$ matches exactly the empirical frequencies in the data. In particular, we have $\pierm(h) = 1/n$ for $h \in \hat{H}^{(1)}$. This implies
    \begin{align*}
        \Pcov(\pierm) \geq \sum_{y \in S \setminus \hat{S}} \pi^\star(y) + \sum_{y \in \hat{H}^{(1)}} \frac{2}{n} \mathbf{1}\left\{ \frac{2/n}{1/n} \geq N\right\} = \pi^\star(S \setminus \hat{S}) + \pi^\star(\hat{H}^{(1)})
    \end{align*}
    Finally for $\piabs$ we need to show that no absolute clipping threshold improves on this value. To cover any $h \in \hat{H}^{(1)}$ we need $\piabs(h) \geq 3/(2n)$. Recall that we only see $|\hat{S}|$ actions in the dataset. If we set $\alpha \leq 1/|\hat{S}|$ then we can achieve zero loss but from~\cref{lem:kl_projection} we have
    \begin{align*}
    \piabs(h) = \max(\alpha, \lambda  \piref(h)) \leq \max(1/|\hat{S}|, 1/n) < 3/(2n).
    \end{align*}
    Here recall that $\lambda \leq 1$ (otherwise $\piabs(h)$ will not be a distribution) and that we have $|\hat{S}| > \frac{2n}{3}$ with high probability. In this regime, unobserved actions get mass $\lambda \piref(h)$ but this is insufficient to cover $\pi^\star$ (since $2=\pi^\star(y)/\piref(y)=2 \geq N=4/3$). 
    
    On the other hand, if $\alpha > 1/|\hat{S}|$ then, by~\cref{lem:clipped_minimizer} the unique minimizer of the absolute clipped loss is
    \begin{align*}
        \piabs(y) = \min(\alpha, \lambda \hat{\mu}(y))
    \end{align*}
    where $\hat{\mu}(y)$ is the empirical distribution. The normalizing condition gives,
    \begin{align*}
    1 \geq \sum_{y \in \hat{S}}\min( \alpha, \lambda \hat{\mu}(y)) \geq |\hat{S}| \min(\alpha, \lambda/n),
    \end{align*}
    which follows because every observed action is observed at least once, so it has empirical mass at least $1/n$. Since $\alpha > 1/|\hat{S}|$ we must have $\lambda/n \leq 1/|\hat{S}|$ and therefore actions $h \in \hat{H}^{(1)}$ have $\piabs(h) \leq 1/|\hat{S}| \leq \frac{3}{2n}$ (while unobserved actions have $\piabs(y) = 0$).

    Thus in both cases we have
    \begin{align*}
        \Pcov(\piabs) \geq \pi^\star(S \setminus \hat{S}) + \pi^\star(\hat{H}^{(1)})
    \end{align*}
    The high probability event implies that $|\hat{H}^{(1)}|>0$ and so $\pi^\star(\hat{H}^{(1)}) > 0$, establishing strict separation. 
\end{proof}
\begin{proof}[Proof of Claim 3 in~\cref{thm:main}]
We use the same construction as in the proof of Claim 2. Recall that in that construction, zero absolute clipped loss is achievable only if $\alpha\leq 1/|\hat{S}|$. Let $\hat{H}^{(>1)}$ denote the set of actions in $H$ that are observed at least twice in the sample $\hat{S}$. Since $\alpha \leq 1/|\hat{S}|\leq 3/(2n)$ (since with high probability we have $|\hat{S}| \geq 2n/3$), via~\cref{lem:kl_projection}, the absolute clipped loss solution satisfies
\begin{align*}
\piabs(h) = \max(\alpha, \lambda \piref(h)) \leq \max(1/|\hat{S}|, 1/n) < 3/(2n) = \pi^\star(h)/N,
\end{align*}
if we take $N = 4/3$ as above. Thus all actions in $\hat{H}^{(<1)}$ violate coverage. We already saw that all actions in $\hat{H}^{(1)}$ violate coverage as well as all unobserved actions in $L$. Thus the coverage is
\begin{align*}
    \inf_{\alpha: \min_{\pi} \labs(\pi) = 0} \Pcov(\piabs) = \frac{1}{2} + \pi^\star(L \setminus \hat{S}).
\end{align*}
On the other hand, ERM covers actions in $\hat{H}^{(>1)}$, because for any $h \in \hat{H}^{(>1)}$, we have $\pierm(h) \geq 2/n$ while $\pi^\star(h) = 2/n$. Thus ERM strictly dominates absolute clipping whenever $|\hat{H}^{(>1)}| > 0$. Since
\begin{align*}
    \mathbb{E}|\hat{H}^{(>1)}| = \sum_{h \in H} \Pr[B_h > 1] = \frac{n}{4}\cdot (1 - (1-2/n)^n - 2(1-2/n)^{n-1})
\end{align*}
where $B_h \sim \textrm{Binomial}(n, 2/n)$, and since $|\hat{H}^{(>1)}|$ satisfies the conditions of McDiarmid's inequality with constant 1, we have that with probability at least $1-\delta$
\begin{align*}
|\hat{H}^{(>1)}| \geq \frac{n}{4}\cdot (1 - (1-2/n)^n - 2(1-2/n)^{n-1}) - \sqrt{\frac{n}{2}\log(1/\delta)}.
\end{align*}
For sufficiently large $n$ the first term is larger than $(1 - 3e^{-2})n/5$ and the second term is lower order, so for sufficiently large $n$ we get that $|H^{(>1)}|>1$ with high probability. 
\end{proof}

\section{Details for Graph Navigation Experiments}
\label{app:synthetic}
In this section, we describe the task, training details, and results for the graph navigation experiments presented in~\cref{fig:synthetic_main}. The setting is closely related to that of~\cite{chen2025coverageprinciple}; we present all details for completeness but highlight where our experiments diverge from theirs. 

\subsection{Graph Reasoning Task and Data}
Following~\citet{chen2025coverageprinciple}, we use a path-following task in a directed acyclic graph (DAG) as an abstraction for reasoning problems. The setup builds on a long line of work using synthetic tasks to understand phenomena in language modeling and serves as a minimal, expressive, yet flexible setting to develop interventions like TailSFT. 

The task is path-following a directed acyclic graph. Each prompt $x \in \Xcal$ encodes a graph $G$ along with a source node $s$ and a target node $t$, such that $x = (G,s,t)$, and each response $y \in \Ycal$ ideally encodes an $s\to t$ path via a sequence of vertices $(s, v_1, \ldots, v_n, t)$. In our experiments we fix a single prompt distribution $\mu \in \Delta(\Xcal)$ for both pre-training and SFT stages, however we use different data-collection policies $\pipre, \pisft: Y \to \Delta(\Xcal)$, both of which map prompts to (distributions over) $s\to t$ paths. 

All graphs $G$ are directed layered graphs with $10$ total layers and with edges only between consecutive layers. The source vertex $s$ is the only vertex in layer $L=1$ and the target vertex $t$ is the only vertex in layer $10$. The subsequent 8 layers have 4 vertices each. For each of layers $L \in \{1,\ldots,9\}$, a subset of the vertices in that layer are called \emph{passable}. The edge structure is such that every passable node in layer $\ell$ has directed edges to all nodes in layer $\ell+1$. Nodes that are not passable have no out-going edges. Thus a valid $s\to t$ path consists of $10$ total nodes $(s, v_1,\ldots,v_8, t)$ where $v_\ell$ is in layer $\ell$ and each $v_\ell$ must be passable. 

Graphs are parameterized by a number $k \in \{0,\ldots,8\}$ denoting the number of layers with \emph{two} passable nodes; the remaining layers have exactly \emph{one} passable node. Given $k$ we choose the layers with two passable nodes at random. Choosing the passable nodes is more intricate. We identify every vertex with a number $\in \{0,\ldots,99\}$, such that in each of the intermediate layers, at least one vertex is \emph{odd} and at least one vertex is \emph{even}. When there is just one passable vertex in a layer, it is chosen uniformly at random. When there are two passable vertices, they are chosen at random such that one is odd and one is even. Thus, there are $2^k$ valid $s \to t$ paths in each graph. We use $k=3$ for all experiments.

As described above, pre-training and SFT share the same prompt distribution. These are graphs of the above structure with vertex names assigned randomly (subject to the aforementioned even/odd restrictions in each layer), with passable nodes assigned randomly, and with layers with two passable nodes selected randomly. For pre-training, the data-collection policy $\pipre$ selects an $s \to t$ path uniformly at random from all $2^k$ valid paths. 

For SFT, the data-collection policy is more complex. First the policy computes a certain ``cryptographic'' function of the input graph to determine if the graph belongs to \textsc{shard0} or \textsc{shard1} (specifically, the cryptographic function is the and of the parity of the vertices layers 1-5 and the parity of the vertices in layers 6-10). The shard determines a certain rule for how the policy chooses vertices in the layers where there are two passable nodes. Specifically in \textsc{shard0} the policy alternates between agreeing with the parity of the target vertex $t$ and disagreeing with the parity of target node $t$, such that it agrees with the target vertex parity in the first layer with two passable nodes, disagrees in the second, and so on. In \textsc{shard1} the policy does exactly the opposite, it disagrees with the target vertex parity in the first layer with two passable nodes, agrees in the second, and so on. 

Intuitively, learning the pre-training policy's response distribution is relatively easy as it only requires local rules such as identifying all passable nodes in the subsequent layer and choosing one at random. Pre-training on this distribution teaches the model to understand the input format and the high-level task. On the other hand, learning the SFT policy's response distribution is very challenging, as it requires identifying global structure both to determine the shard and to determine which passable vertex to select in each layer. 

To measure the performance of the trained model, we observe that the SFT policy is deterministic (for any input graph $x$, there is a single path in the support of the SFT policy's distribution $\pisft(\cdot \mid x)$) and set the reward to be $1$ if and only if the path chosen by the trained model matches the path selected by the SFT policy. Since learning the SFT policy's response distribution is challenging, it is corresponding quite challenging to achieve high reward in this task. 

Inputs and outputs are represented as follows. First, all vertices, including the source and target, are assigned a number in $\{0,\ldots,99\}$. Then, the input prompt is represented as an edge list followed by the source and target nodes, formatted as:
\begin{align*}
x: ~~ \texttt{u\_1 v\_1 | u\_2 v\_2 | \ldots | u\_k v\_k / s t =}
\end{align*}
where $u_i,v_i$ are the numerical values assigned to the vertices in the $i$th edge. We use the delimiter characters \texttt{|}, \texttt{/}, and \texttt{=} to separate edges from each other, the edge list from the source and target vertices, and the input from the response, respectively. Responses $y$ are formatted as:
\begin{align*}
    y: ~~ \texttt{v\_1 v\_2 v\_2 v\_3 \ldots v\_{9} v\_10}.
\end{align*}

\subsection{Training Details}

We use the same numerical tokenizer and transformer architecture as~\citet{chen2025coverageprinciple}. For tokenization, each node is $v$ is tokenized as its numerical value and the special characters are tokenized as $100, 101, 102$. The architecture is a GPT2-style transformer model with 4 heads, 6 transformer blocks and a 384 dimensional embeddings. We used absolute positional encodings. We always use the Adam optimizer with a fixed learning($1\times{}10^{-4}$ during pre-training and $5\times{}10^{-6}$ during SFT) and a batch size of 128. Pre-training operates in an offline training setting with a fixed pool of 256,000 graphs while SFT operates in an online training setting with new graphs generated on-the-fly for each batch. We pre-train for 200k training steps and SFT for 50k training steps with evaluation every 200 steps. 

\begin{figure}[t]
    \centering
    \includegraphics[width=\textwidth]{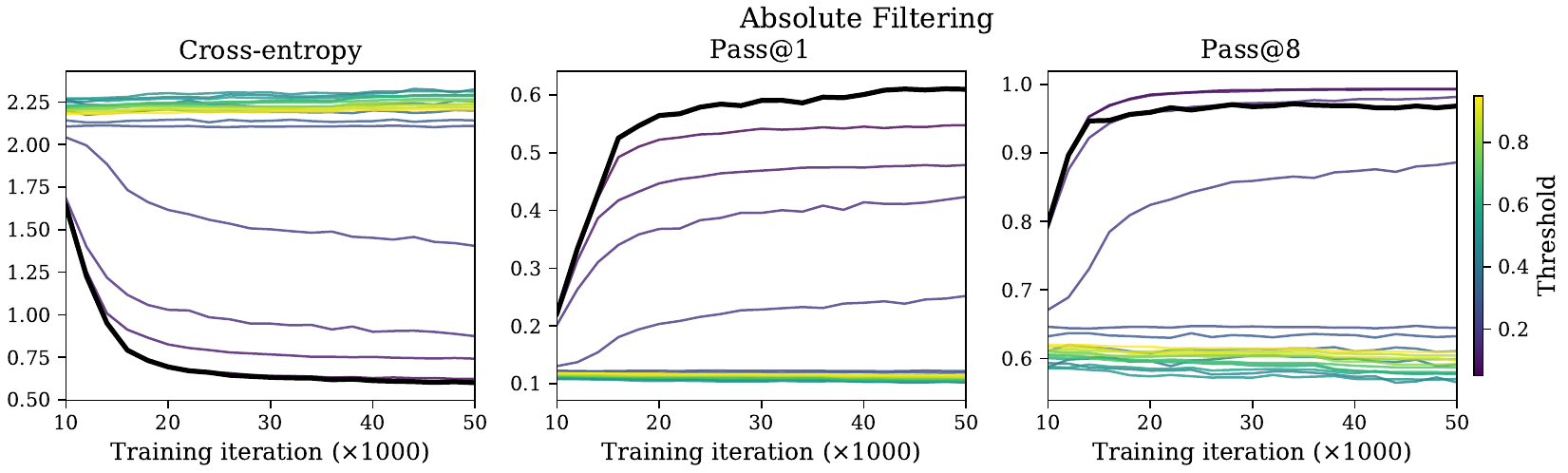}
    \vspace{-0.5cm}
    \caption{Detailed results of hyperparameter sweep for TailSFT with absolute clipping in the synthetic graph navigation experiment.}
    \vspace{-0.5cm}
    \label{fig:synthetic_absolute}
\end{figure}
\begin{figure}[t]
    \centering
    \includegraphics[width=\textwidth]{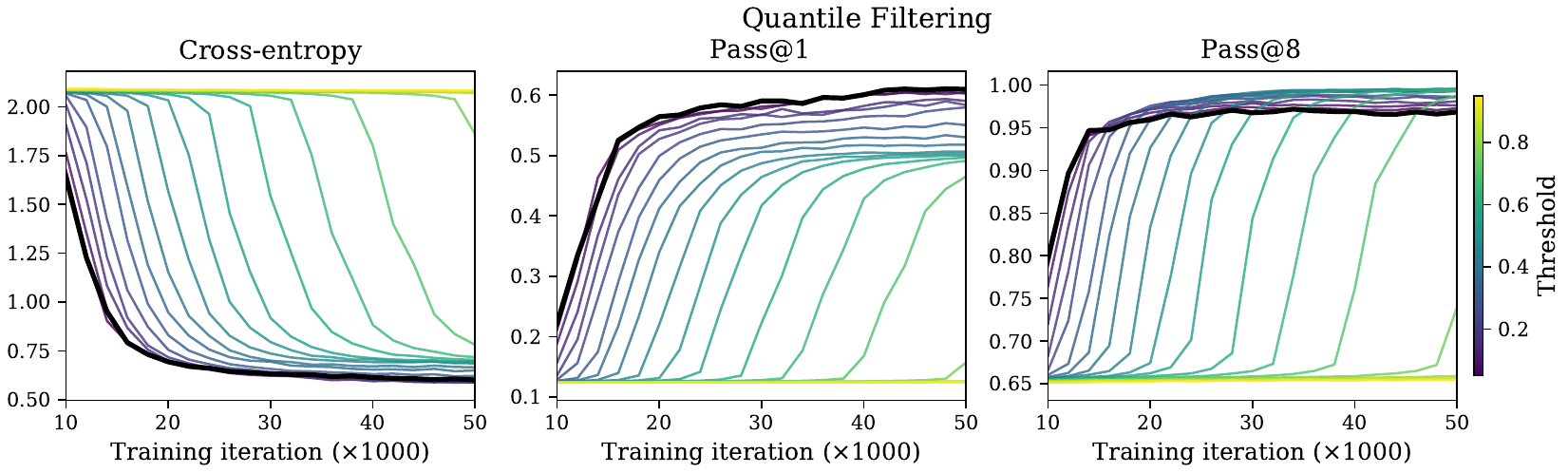}
    \vspace{-0.5cm}
    \caption{Detailed results of hyperparameter sweep for TailSFT with quantile clipping in the synthetic graph navigation experiment.}
    \vspace{-0.5cm}
    \label{fig:synthetic_quantile}
\end{figure}
\begin{figure}[t]
    \centering
    \includegraphics[width=\textwidth]{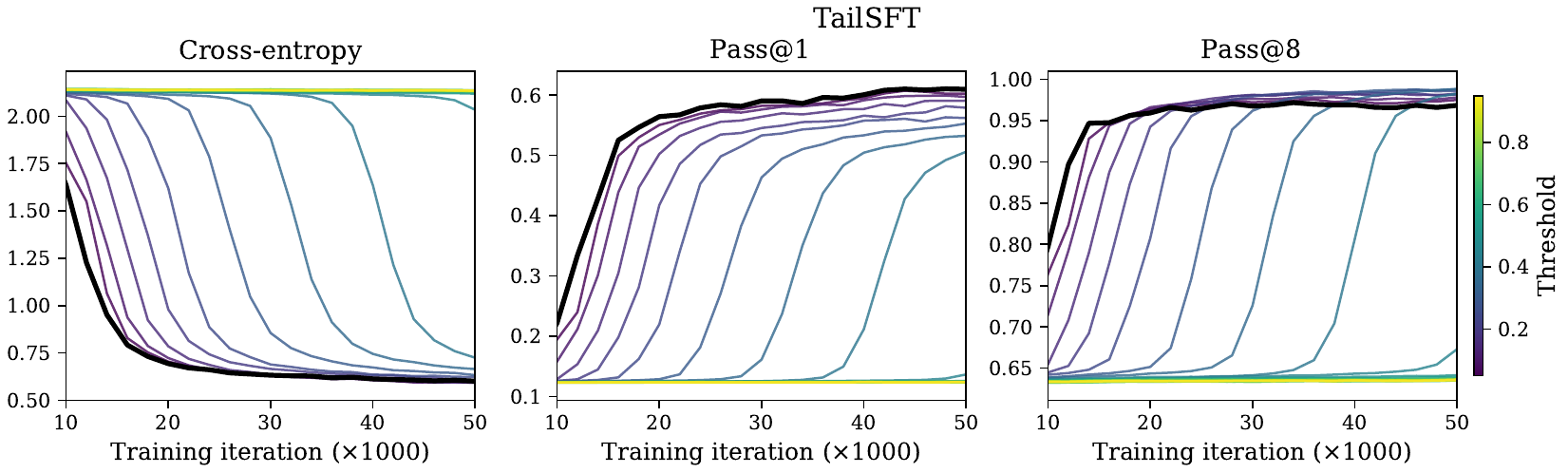}
    \vspace{-0.5cm}
    \caption{Detailed results of hyperparameter sweep for TailSFT with offset quantile clipping in the synthetic graph navigation experiment.}
    \vspace{-0.5cm}
    \label{fig:synthetic_offset}
\end{figure}

\subsection{Additional Experimental Results}
In~\cref{fig:synthetic_main} we show the cross entropy, pass@1, and pass@8 for each of four methods: vanilla SFT, TailSFT with absolute clipping, TailSFT with quantile clipping, and TailSFT with offset quantile clipping using the pre-trained model as the reference. For the TailSFT variants we sweep over clipping hyperparameter and select the best hyperparameter based on pass@8 performance at the end of training. The specific hyperparameter selected are 0.05 for absolute, 0.4 for quantile, and 0.35 for offset quantile. 

In~\cref{fig:synthetic_absolute,fig:synthetic_quantile,fig:synthetic_offset}, we show detailed results of the hyperparameter sweep for the three TailSFT variants. In all cases, we sweep the clipping hyperparameter in the range $0.05\times\{1\ldots,19\}$. For all variants clipping at value $0$ corresponds to vanilla SFT as no sequences are dropped. We highlight two observations. First, absolute clipping is much more sensitive to the clipping hyperparameter than quantile or offset quantile variants. This is rather intuitive, as setting an absolute threshold for the cross-entropy loss requires understanding how the distribution of per-example cross-entropy losses evolves during training. On the other hand, setting the threshold as a per-batch quantile (either on the loss or on the offset loss) is much easier and these methods are much more robust. Indeed, the second observation is that for quantile-variants of TailSFT, setting the threshold too small recovers baseline SFT performance and setting the threshold too large does not impact pass@8 performance, but rather slows training convergence. This is also intuitive, as a large clipping parameter still focuses training on the harder examples but sacrifices computational efficiency as fewer per-sample gradients are utilized. 

\section{Details for Language Modeling Experiments}
\subsection{SFT}
\label{app:sft}

The language modeling experiments use the same \tailsft{} algorithm (\Cref{alg:tailsft}) for both math and code but
differ in data, prompt format, hyperparameter grid, and evaluation.  We
describe the shared procedure here; the code (\cref{app:sft-code}) and math
(\cref{app:sft-math}) subsections give the domain-specific details.

\paragraph{Supervised fine-tuning}
Every run starts from the OLMo-3 7B base model and minimizes the standard
supervised cross-entropy loss on chat-formatted instruction/response pairs.  We
apply the OLMo-3 chat template and mask every prompt token, so only the final
assistant turn and its terminating end-of-turn token count toward the loss.  All
runs use AdamW with $(\beta_1,\beta_2)=(0.9,0.95)$ and $\epsilon=10^{-8}$,
gradient-norm clipping at $1.0$, $8$-way data parallelism, and a learning rate
held constant after a linear warmup over the first $3\%$ of steps.  Sequence
length, weight decay, microbatch size, effective batch, epoch count, and the
learning-rate grid are set per domain and given in each subsection.  The per-step
loss is length-normalized as in OLMo-3: we sum the target-token
cross-entropy over the batch and divide by the number of unmasked target tokens,
so every target token is weighted equally rather than every sequence.

\paragraph{Example filtering (\tailsft{})}
\tailsft{} drops a fraction of the already-fit examples from the loss at each step.
Before training we score every example $x$ once with the base model, using the
same tokenization and assistant-only mask as SFT, to get its initialization loss
$\ell_0(x)$, the mean cross-entropy over the target tokens.  During training we
compare the current per-example loss $\ell_t(x)$ against this baseline through the
signed margin
\[
    m_t(x) = \ell_t(x) - \ell_0(x).
\]
A very negative margin means the loss on $x$ has already dropped well below its
starting value.  We filter out
the examples with the smallest (most negative) margins and train on the rest.
Both $\ell_0$ and $\ell_t$ are length-normalized per example---the mean
cross-entropy over that example's target tokens---so the filtering decision does
not favor longer or shorter responses.

The filtering decision runs on the \emph{selection batch}, the $W\times b$
examples in a single forward pass across the $W=8$ data-parallel ranks with
per-rank microbatch $b$.  The per-example losses are all-gathered so every rank
computes the same mask.  We do not filter over the full gradient-accumulated
optimizer batch or over one rank's local microbatch; each accumulation step is
filtered on its own selection batch.  Within a selection batch we rank the
examples by $m_t$ and zero out the $\lfloor Wb\,f\rceil$ with the smallest
margins, then train the survivors with the usual token-averaged cross-entropy.
The code runs use $b=1$ (an $8$-example selection batch) and the math runs use
$b=2$ (a $16$-example selection batch).  Reproducing the method needs only the
base-model losses $\ell_0$, the filter fraction $f$, and its schedule; everything
else is ordinary distributed SFT.

\paragraph{Filtering notation}
A run is specified by its filter fraction $f\in[0,1]$, the fraction of each
selection batch that is dropped ($\lfloor Wb\,f\rceil$ examples), and a schedule
for how $f$ moves over training.  A \emph{static} schedule holds $f$ fixed.  A
\emph{ramp} schedule written $0\!\to\!f$ raises the fraction linearly from $0$ at
the first step to $f$ at the last, using every example early and filtering
hardest late.  Standard SFT is $f=0$.  We describe every run by its filter
fraction $f$ and schedule in the below results.

\paragraph{Filtering rarely hurts}
The consistent finding across both domains is that filtering preserves pass@16
coverage.  At matched settings it almost always matches or improves the no-filter
pass@16, the gains are sometimes large, and the only cost is a small pass@1
decrease from deprioritizing examples the model already fits.  Each subsection
reports the full per-benchmark grids and tallies how rarely filtering regresses
coverage.  \cref{tab:sft-base-p1,tab:sft-base-p16} collect the base model, Standard
SFT, and \tailsft{} side by side at the main-table settings for pass@1 and pass@16.
In several cases where standard SFT marginally lowers pass@16 below the base model,
\tailsft{} rescues the coverage, recovering it back above the base level.
If we instead select the best pass@16 checkpoint for
each of Standard SFT and \tailsft{}, \tailsft{} still matches or improves
coverage on most benchmarks.

\begin{table}[t]
\centering
\footnotesize
\caption{Per-benchmark pass@1 for the OLMo-3 7B base model (no SFT), Standard SFT ($f{=}0$), and \tailsft{}, at the settings used in the main results table. $\Delta_{\mathrm{Std}-\mathrm{Base}}$ is Standard SFT minus base; $\Delta_{\mathrm{Tail}-\mathrm{Base}}$ is \tailsft{} minus base. Emdashes indicate missing evaluations.}
\label{tab:sft-base-p1}
\begin{tabular}{@{}ll ccc cc@{}}
\toprule
SFT data & Benchmark & Base & Standard SFT & \tailsft{} & $\Delta_{\mathrm{Std}-\mathrm{Base}}$ & $\Delta_{\mathrm{Tail}-\mathrm{Base}}$ \\
\midrule
OMI & AIME & $5.16$ & $3.65_{\pm0.37}$ & $4.03_{\pm0.02}$ & $-1.51$ & $-1.13$ \\
 & MATH-500 L5 & \textemdash & $24.80_{\pm0.40}$ & $25.16_{\pm0.40}$ & \textemdash & \textemdash \\
 & OMEGA-500 & $3.54_{\pm0.30}$ & $6.58_{\pm0.50}$ & $6.51_{\pm0.26}$ & $+3.04$ & $+2.97$ \\
\midrule
BigCode & MBPP+ & $36.82_{\pm0.56}$ & $53.04_{\pm0.93}$ & $51.36_{\pm0.14}$ & $+16.22$ & $+14.54$ \\
 & HumanEval+ & $32.94_{\pm1.11}$ & $45.06_{\pm0.57}$ & $46.33_{\pm1.07}$ & $+12.12$ & $+13.39$ \\
 & CruxEval-I & $25.67_{\pm0.62}$ & $29.79_{\pm0.48}$ & $30.43_{\pm0.15}$ & $+4.12$ & $+4.76$ \\
 & CruxEval-O & $4.22_{\pm0.74}$ & $4.16_{\pm0.44}$ & $13.18_{\pm0.30}$ & $-0.06$ & $+8.96$ \\
 & LiveCodeBench & $5.50_{\pm0.42}$ & $10.74_{\pm0.46}$ & $11.49_{\pm0.38}$ & $+5.24$ & $+5.99$ \\
\midrule
Magicoder & MBPP+ & $36.82_{\pm0.56}$ & $55.35_{\pm0.47}$ & $54.60_{\pm0.55}$ & $+18.53$ & $+17.78$ \\
 & HumanEval+ & $32.94_{\pm1.11}$ & $48.49_{\pm0.58}$ & $47.69_{\pm0.19}$ & $+15.55$ & $+14.75$ \\
 & CruxEval-I & $25.67_{\pm0.62}$ & $28.50_{\pm0.13}$ & $26.48_{\pm0.45}$ & $+2.83$ & $+0.81$ \\
 & CruxEval-O & $4.22_{\pm0.74}$ & $12.86_{\pm0.19}$ & $16.16_{\pm0.53}$ & $+8.64$ & $+11.94$ \\
 & LiveCodeBench & $5.50_{\pm0.42}$ & $14.84_{\pm0.11}$ & $14.32_{\pm0.38}$ & $+9.34$ & $+8.82$ \\
\midrule
OCI & MBPP+ & $36.82_{\pm0.56}$ & $58.26_{\pm0.88}$ & $55.81_{\pm0.06}$ & $+21.44$ & $+18.99$ \\
 & HumanEval+ & $32.94_{\pm1.11}$ & $54.76_{\pm1.16}$ & $51.94_{\pm1.83}$ & $+21.82$ & $+19.00$ \\
 & CruxEval-I & $25.67_{\pm0.62}$ & $29.46_{\pm0.46}$ & $30.17_{\pm0.64}$ & $+3.79$ & $+4.50$ \\
 & CruxEval-O & $4.22_{\pm0.74}$ & $9.30_{\pm0.49}$ & $10.84_{\pm0.37}$ & $+5.08$ & $+6.62$ \\
 & LiveCodeBench & $5.50_{\pm0.42}$ & $12.24_{\pm0.36}$ & $11.49_{\pm0.84}$ & $+6.74$ & $+5.99$ \\
\bottomrule
\end{tabular}
\end{table}

\begin{table}[t]
\centering
\footnotesize
\caption{Per-benchmark pass@16 (coverage) for the OLMo-3 7B base model (no SFT), Standard SFT ($f{=}0$), and \tailsft{}, at the settings used in the main results table. Columns as in \cref{tab:sft-base-p1}. }
\label{tab:sft-base-p16}
\begin{tabular}{@{}ll ccc cc@{}}
\toprule
SFT data & Benchmark & Base & Standard SFT & \tailsft{} & $\Delta_{\mathrm{Std}-\mathrm{Base}}$ & $\Delta_{\mathrm{Tail}-\mathrm{Base}}$ \\
\midrule
OMI & AIME & $20.91$ & $15.24_{\pm2.02}$ & $18.31_{\pm0.57}$ & $-5.67$ & $-2.60$ \\
 & MATH-500 L5 & \textemdash & $66.42_{\pm0.00}$ & $69.15_{\pm0.86}$ & \textemdash & \textemdash \\
 & OMEGA-500 & $26.13_{\pm1.42}$ & $32.80_{\pm2.25}$ & $32.60_{\pm1.56}$ & $+6.67$ & $+6.47$ \\
\midrule
BigCode & MBPP+ & $75.93_{\pm1.06}$ & $74.69_{\pm1.10}$ & $78.84_{\pm1.06}$ & $-1.24$ & $+2.91$ \\
 & HumanEval+ & $79.27_{\pm0.61}$ & $76.63_{\pm0.70}$ & $80.08_{\pm0.93}$ & $-2.64$ & $+0.81$ \\
 & CruxEval-I & $70.50_{\pm1.54}$ & $61.71_{\pm2.27}$ & $65.79_{\pm0.26}$ & $-8.79$ & $-4.71$ \\
 & CruxEval-O & $32.08_{\pm1.98}$ & $24.21_{\pm2.38}$ & $41.00_{\pm1.35}$ & $-7.87$ & $+8.92$ \\
 & LiveCodeBench & $32.30_{\pm0.93}$ & $30.39_{\pm0.86}$ & $33.12_{\pm0.34}$ & $-1.91$ & $+0.82$ \\
\midrule
Magicoder & MBPP+ & $75.93_{\pm1.06}$ & $75.31_{\pm0.93}$ & $78.66_{\pm1.25}$ & $-0.62$ & $+2.73$ \\
 & HumanEval+ & $79.27_{\pm0.61}$ & $77.85_{\pm0.70}$ & $78.66_{\pm1.83}$ & $-1.42$ & $-0.61$ \\
 & CruxEval-I & $70.50_{\pm1.54}$ & $59.75_{\pm0.62}$ & $68.08_{\pm0.94}$ & $-10.75$ & $-2.42$ \\
 & CruxEval-O & $32.08_{\pm1.98}$ & $38.08_{\pm0.94}$ & $47.92_{\pm1.66}$ & $+6.00$ & $+15.84$ \\
 & LiveCodeBench & $32.30_{\pm0.93}$ & $31.81_{\pm0.34}$ & $33.22_{\pm0.34}$ & $-0.49$ & $+0.92$ \\
\midrule
OCI & MBPP+ & $75.93_{\pm1.06}$ & $81.22_{\pm0.75}$ & $82.36_{\pm0.81}$ & $+5.29$ & $+6.43$ \\
 & HumanEval+ & $79.27_{\pm0.61}$ & $85.98_{\pm1.61}$ & $83.23_{\pm2.16}$ & $+6.71$ & $+3.96$ \\
 & CruxEval-I & $70.50_{\pm1.54}$ & $68.08_{\pm1.77}$ & $71.58_{\pm1.91}$ & $-2.42$ & $+1.08$ \\
 & CruxEval-O & $32.08_{\pm1.98}$ & $40.12_{\pm0.00}$ & $44.46_{\pm0.47}$ & $+8.04$ & $+12.38$ \\
 & LiveCodeBench & $32.30_{\pm0.93}$ & $34.59_{\pm0.90}$ & $33.50_{\pm0.16}$ & $+2.29$ & $+1.20$ \\
\bottomrule
\end{tabular}
\end{table}

\paragraph{The gains are broad, not concentrated}
\Cref{tab:perproblem-passk} decomposes each \passat{16} gain into per-problem
paired changes, showing that the improvements are spread across many problems
rather than driven by a handful.  
\begin{table}[t]
\centering
\footnotesize
\setlength{\tabcolsep}{3.5pt}
\begin{tabular}{@{}ll rrr rrr@{}}
\toprule
SFT data & Benchmark & \#gained & \#lost & \#unchanged & median $\Delta$ (pp) & mean $\Delta$ (pp) & 95\% CI \\
\midrule
OMI & AIME & 31 & 14 & 75 & +0.00 & +3.07 & [+0.77, +5.59] \\
OMI & MATH-500 L5 & 25 & 15 & 94 & +0.00 & +2.74 & [-1.24, +6.47] \\
OMI & OMEGA-500 & 75 & 77 & 348 & +0.00 & -0.20 & [-2.33, +1.80] \\
\addlinespace[2pt]
BigCode & MBPP+ & 39 & 19 & 320 & +0.00 & +4.14 & [+2.12, +6.35] \\
BigCode & HumanEval+ & 21 & 11 & 132 & +0.00 & +3.46 & [+0.00, +6.91] \\
BigCode & CruxEval-I & 157 & 89 & 554 & +0.00 & +4.08 & [+2.12, +6.12] \\
BigCode & CruxEval-O & 274 & 44 & 482 & +0.00 & +16.79 & [+14.54, +19.04] \\
BigCode & LiveCodeBench & 72 & 41 & 499 & +0.00 & +2.72 & [+1.09, +4.41] \\
\addlinespace[2pt]
Magicoder & MBPP+ & 42 & 19 & 317 & +0.00 & +3.35 & [+1.06, +5.73] \\
Magicoder & HumanEval+ & 13 & 10 & 141 & +0.00 & +0.81 & [-2.03, +3.86] \\
Magicoder & CruxEval-I & 196 & 82 & 522 & +0.00 & +8.33 & [+6.17, +10.46] \\
Magicoder & CruxEval-O & 183 & 48 & 569 & +0.00 & +9.83 & [+7.92, +11.83] \\
Magicoder & LiveCodeBench & 56 & 41 & 515 & +0.00 & +1.42 & [-0.22, +3.05] \\
\addlinespace[2pt]
OCI & MBPP+ & 22 & 15 & 341 & +0.00 & +1.15 & [-0.35, +2.60] \\
OCI & HumanEval+ & 8 & 17 & 139 & +0.00 & -2.74 & [-5.18, -0.51] \\
OCI & CruxEval-I & 134 & 69 & 597 & +0.00 & +3.50 & [+2.04, +4.92] \\
OCI & CruxEval-O & 132 & 57 & 611 & +0.00 & +4.33 & [+2.83, +5.79] \\
OCI & LiveCodeBench & 45 & 65 & 502 & +0.00 & -1.09 & [-2.61, +0.44] \\
\bottomrule
\end{tabular}
\caption{Per-problem paired changes in \passat{16} for \tailsft{} versus Standard SFT at the main-table settings. A problem counts as gained or lost only when its paired \passat{16} strictly changes; unchanged includes ceiling and floor ties. The final column is a paired bootstrap $95\%$ interval for the mean per-problem change.}
\label{tab:perproblem-passk}
\end{table}

\subsubsection{Code}
\label{app:sft-code}

\paragraph{Model and optimization}
Code SFT trains at sequence length $2048$, which effectively never truncates training targets.  Runs use no weight decay and a per-rank microbatch
of $1$ with $4$ gradient-accumulation steps, for an effective batch of $32$ and an
$8$-example selection batch.  The learning rate is searched over
$\{1,2,3\}\times10^{-5}$.  Each dataset uses a fixed epoch count.

\begin{table}[t]
\centering
\small
\begin{tabular}{p{0.1\linewidth} p{0.45\linewidth} p{0.1\linewidth} }
\toprule
SFT data & Source & Train / val \\
\midrule
Magicoder & \texttt{ise-uiuc/Magicoder-OSS-Instruct-75K} & 33,817 / 342  \\
BigCode & \texttt{bigcode/self-oss-instruct-sc2-exec-filter-50k} & 49,260 / 498 \\
OCI & \texttt{nvidia/OpenCodeInstruct} & 72,635 / 365  \\
\bottomrule
\end{tabular}
\caption{Code SFT datasets. Row counts are after decontamination, format conversion, and initialization-loss annotation.}
\label{tab:code-sft-datasets}
\end{table}

\paragraph{Magicoder}
We keep the Python subset of \texttt{ise-uiuc/Magicoder-OSS-Instruct-75K}
($38{,}284$ of $75{,}197$ rows).  We drop non-Python rows, rows without
extractable code, rows matching an MBPP+ or HumanEval+ signature, and rows
containing unit tests, and discard solutions that fail to parse as an abstract
syntax tree.  The assistant target starts with the prefix
\texttt{Here is the completed function:\textbackslash n\textbackslash n```python\textbackslash n},
followed by the extracted code and a closing fence, matching the MBPP+ and
HumanEval+ answer format.  One further row is removed because it does not fit the
$2048$-token training tokenization, leaving $33{,}817$ train rows.

\paragraph{BigCode}
For \texttt{bigcode/self-oss-instruct-sc2-exec-filter-50k} ($50{,}661$ raw rows)
the user message is the instruction and the target uses the same evalplus-aligned
prefix and fenced Python format as Magicoder.  We remove exact and near MBPP+ and
HumanEval+ contamination, benchmark-signature contamination, rows containing
tests, and non-Python rows, leaving $49{,}260$ train rows.

\paragraph{OCI}
For \texttt{nvidia/OpenCodeInstruct} we load five evenly spaced shards and keep
rows with \texttt{average\_test\_score}$=1.0$.  We remove MBPP+ contamination by
exact prompt match, by test-assertion substring match against inputs, outputs,
and unit tests, and by MBPP+ function-name definitions appearing in outputs, then
shuffle with seed $42$ and hold out a $0.5\%$ validation split.  OCI keeps its
native format, with the user message the OCI \texttt{input} and the target the OCI
\texttt{output}, both verbatim.

\paragraph{Hyperparameter configurations}
We vary the learning rate and filtering configuration as summarized in Table~7. The configurations used in \Cref{tab:sft-main} and the additional evaluated configurations are reported across several tables.
Tables~\ref{tab:code-sft-grid-bigcode} and \ref{tab:code-sft-grid-magicoder} give
MBPP+ pass@1 and pass@16 for every completed BigCode and Magicoder
configuration, with the change against the no-filter run at the same learning
rate.  

\begin{table}[t]
\centering
\small
\begin{tabular}{p{0.32\linewidth} p{0.55\linewidth}}
\toprule
Hyperparameter & Values \\
\midrule
Base model & OLMo-3 7B base \\
Learning rate & $\{1,2,3\}\times10^{-5}$ \\
Filter schedule & none; static; ramp $0\!\to\!\cdot$ \\
Filter fraction $f$ & $0$ (none); static $\{0.125,0.25,0.5\}$; ramp $0\!\to\!0.5$ \\
Epochs & fixed per dataset ($4$ / $5$ / $3$ for BigCode / Magicoder / OCI) \\
Sequence length & $2048$ \\
Effective batch & $32$ ($8\times1\times4$) \\
Warmup / schedule & $3\%$ linear warmup, then constant \\
Optimizer & AdamW $(0.9,0.95)$, wd $0$, clip $1.0$ \\
\bottomrule
\end{tabular}
\caption{Code SFT hyperparameter grid over learning rates and filtering configurations. }
\label{tab:code-sft-hp-grid}
\end{table}

\begin{table}[t]
\centering
\small
\begin{tabular}{llrrrr}
\toprule
Filter & LR & pass@1 & $\Delta$ & pass@16 & $\Delta$ \\
\midrule
base (no SFT) & --- & $36.82\,\pm\,0.56$ & --- & $75.93\,\pm\,1.06$ & --- \\
\addlinespace[2pt]
no filter & $1\times 10^{-5}$ & $47.45\,\pm\,0.82$ & --- & $79.45\,\pm\,0.81$ & --- \\
static $f{=}0.125$ & $1\times 10^{-5}$ & $47.00\,\pm\,0.60$ & $-0.45$ & $79.81\,\pm\,0.81$ & $+0.36$ \\
ramp $0{\to}0.5$ & $1\times 10^{-5}$ & $45.62\,\pm\,0.57$ & $-1.83$ & $79.72\,\pm\,0.85$ & $+0.27$ \\
\addlinespace[2pt]
no filter & $2\times 10^{-5}$ & $52.58\,\pm\,0.11$ & --- & $76.72\,\pm\,1.06$ & --- \\
static $f{=}0.125$ & $2\times 10^{-5}$ & $52.07\,\pm\,0.70$ & $-0.51$ & $77.51\,\pm\,0.46$ & $+0.79$ \\
static $f{=}0.25$ & $2\times 10^{-5}$ & $51.36\,\pm\,0.14$ & $-1.22$ & $78.84\,\pm\,1.06$ & $+2.12$ \\
static $f{=}0.5$ & $2\times 10^{-5}$ & $47.22\,\pm\,0.54$ & $-5.36$ & $77.95\,\pm\,0.76$ & $+1.23$ \\
ramp $0{\to}0.5$ & $2\times 10^{-5}$ & $51.65\,\pm\,0.18$ & $-0.93$ & $78.66\,\pm\,0.40$ & $+1.94$ \\
\addlinespace[2pt]
no filter & $3\times 10^{-5}$ & $53.04\,\pm\,0.93$ & --- & $74.69\,\pm\,1.10$ & --- \\
static $f{=}0.125$ & $3\times 10^{-5}$ & $51.90\,\pm\,0.59$ & $-1.14$ & $75.40\,\pm\,1.06$ & $+0.71$ \\
ramp $0{\to}0.5$ & $3\times 10^{-5}$ & $50.67\,\pm\,0.26$ & $-2.37$ & $75.93\,\pm\,0.46$ & $+1.24$ \\
\bottomrule
\end{tabular}
\caption{BigCode MBPP+ ($n{=}378$) SFT filtering grid, $4$ epochs.  Values are
percent, mean $\pm$ sample standard deviation over three seeds;
$\Delta$ is the change against the no-filter run at the same learning rate.
The \emph{base (no SFT)} row is the untuned OLMo-3 7B base under the same MBPP+
protocol (format-robust EvalPlus stops).}
\label{tab:code-sft-grid-bigcode}
\end{table}

\begin{table}[t]
\centering
\small
\begin{tabular}{llrrrr}
\toprule
Filter & LR & pass@1 & $\Delta$ & pass@16 & $\Delta$ \\
\midrule
base (no SFT) & --- & $36.82\,\pm\,0.56$ & --- & $75.93\,\pm\,1.06$ & --- \\
\addlinespace[2pt]
no filter & $1\times 10^{-5}$ & $53.14\,\pm\,0.52$ & --- & $80.69\,\pm\,0.26$ & --- \\
static $f{=}0.125$ & $1\times 10^{-5}$ & $52.19\,\pm\,0.25$ & $-0.95$ & $79.98\,\pm\,0.61$ & $-0.71$ \\
static $f{=}0.25$ & $1\times 10^{-5}$ & $51.51\,\pm\,0.04$ & $-1.63$ & $80.42\,\pm\,1.40$ & $-0.27$ \\
static $f{=}0.5$ & $1\times 10^{-5}$ & $50.00\,\pm\,0.59$ & $-3.14$ & $79.45\,\pm\,1.55$ & $-1.24$ \\
\addlinespace[2pt]
no filter & $2\times 10^{-5}$ & $55.01\,\pm\,0.36$ & --- & $75.13\,\pm\,1.15$ & --- \\
static $f{=}0.125$ & $2\times 10^{-5}$ & $54.83\,\pm\,0.52$ & $-0.18$ & $77.51\,\pm\,0.53$ & $+2.38$ \\
static $f{=}0.25$ & $2\times 10^{-5}$ & $54.60\,\pm\,0.55$ & $-0.41$ & $78.66\,\pm\,1.25$ & $+3.53$ \\
static $f{=}0.5$ & $2\times 10^{-5}$ & $51.92\,\pm\,0.41$ & $-3.09$ & $79.72\,\pm\,0.31$ & $+4.59$ \\
ramp $0{\to}0.5$ & $2\times 10^{-5}$ & $54.73\,\pm\,0.73$ & $-0.28$ & $78.31\,\pm\,1.47$ & $+3.18$ \\
\addlinespace[2pt]
no filter & $3\times 10^{-5}$ & $55.35\,\pm\,0.47$ & --- & $75.31\,\pm\,0.93$ & --- \\
static $f{=}0.125$ & $3\times 10^{-5}$ & $54.75\,\pm\,0.58$ & $-0.60$ & $76.54\,\pm\,0.15$ & $+1.23$ \\
static $f{=}0.25$ & $3\times 10^{-5}$ & $52.94\,\pm\,0.67$ & $-2.41$ & $75.93\,\pm\,1.06$ & $+0.62$ \\
static $f{=}0.5$ & $3\times 10^{-5}$ & $51.14\,\pm\,0.47$ & $-4.21$ & $77.95\,\pm\,1.36$ & $+2.64$ \\
\bottomrule
\end{tabular}
\caption{Magicoder MBPP+ ($n{=}378$) SFT filtering grid, $5$ epochs.  Values are
percent, mean $\pm$ sample standard deviation over three seeds;
$\Delta$ is the change against the no-filter run at the same learning rate.
The \emph{base (no SFT)} row is the same untuned base as in
\cref{tab:code-sft-grid-bigcode}.}
\label{tab:code-sft-grid-magicoder}
\end{table}

For OCI, Table~10 reports the no-filter Standard configuration and the TailSFT ramp configuration used in Table~1.  MBPP+ coverage
improves under filtering ($+1.15$~pp), while HumanEval+ loses coverage
($-2.74$~pp).  Accordingly, OCI HumanEval+ has $\rho_{16}=0.23$, outside the $\rho_{16}>1$ regime identified by the coverage-ratio diagnostic.

\begin{table}[t]
\centering
\small
\begin{tabular}{llrrrr}
\toprule
Benchmark & Filter & pass@1 & $\Delta$ & pass@16 & $\Delta$ \\
\midrule
MBPP+ ($n{=}378$) & base (no SFT) & $36.82\,\pm\,0.56$ & --- & $75.93\,\pm\,1.06$ & --- \\
MBPP+ ($n{=}378$) & no filter & $58.26\,\pm\,0.88$ & --- & $81.22\,\pm\,0.75$ & --- \\
MBPP+ ($n{=}378$) & ramp $0{\to}0.5$ & $55.81\,\pm\,0.06$ & $-2.44$ & $82.36\,\pm\,0.81$ & $+1.15$ \\
\addlinespace[2pt]
HumanEval+ ($n{=}164$) & base (no SFT) & $32.94\,\pm\,1.11$ & --- & $79.27\,\pm\,0.61$ & --- \\
HumanEval+ ($n{=}164$) & no filter & $54.76\,\pm\,1.16$ & --- & $85.98\,\pm\,1.61$ & --- \\
HumanEval+ ($n{=}164$) & ramp $0{\to}0.5$ & $51.94\,\pm\,1.83$ & $-2.82$ & $83.23\,\pm\,2.16$ & $-2.74$ \\
\bottomrule
\end{tabular}
\caption{OCI MBPP+ and HumanEval+ SFT grid, $3$ epochs, learning rate $2\times10^{-5}$;
Standard is no filtering and \tailsft{} is the ramp $0\!\to\!0.5$.  Values are
percent, mean $\pm$ sample standard deviation; $\Delta$ is the change against no
filtering.  The \emph{base (no SFT)} rows are the untuned base under the same
per-benchmark protocol.}
\label{tab:code-sft-grid-oci}
\end{table}

\paragraph{Coverage holds across benchmarks and can repair coverage collapse}
Tables~\ref{tab:code-sft-grid-heplus}--\ref{tab:code-sft-grid-lcb} report every
filtered configuration we evaluated on HumanEval+, CruxEval-I, CruxEval-O, and
LiveCodeBench, against the no-filter run for the same dataset.  The largest coverage gains land where standard SFT had driven
coverage \emph{below} the base model: on CruxEval-I standard SFT falls
$8$--$11$~pp under the base ($70.5\!\to\!61.7$ for BigCode, $59.8$ for Magicoder), so
much of \tailsft{}'s advantage is repairing a coverage collapse that ordinary SFT
introduced.  

\begin{table}[t]
\centering
\small
\begin{tabular}{llrrrr}
\toprule
Filter & LR & pass@1 & $\Delta$ & pass@16 & $\Delta$ \\
\midrule
base (no SFT) & --- & $32.94\,\pm\,1.11$ & --- & $79.27\,\pm\,0.61$ & --- \\
\addlinespace[2pt]
\multicolumn{6}{l}{\textit{BigCode}} \\
no filter & $3\times 10^{-5}$ & $45.06\,\pm\,0.57$ & --- & $76.63\,\pm\,0.70$ & --- \\
static $f{=}0.125$ & $2\times 10^{-5}$ & $48.08\,\pm\,0.12$ & $+3.02$ & $79.47\,\pm\,1.76$ & $+2.85$ \\
static $f{=}0.25$ & $2\times 10^{-5}$ & $46.33\,\pm\,1.07$ & $+1.27$ & $80.08\,\pm\,0.93$ & $+3.46$ \\
static $f{=}0.5$ & $2\times 10^{-5}$ & $44.08\,\pm\,0.48$ & $-0.98$ & $80.69\,\pm\,1.27$ & $+4.07$ \\
\addlinespace[2pt]
\multicolumn{6}{l}{\textit{Magicoder}} \\
no filter & $3\times 10^{-5}$ & $48.49\,\pm\,0.58$ & --- & $77.85\,\pm\,0.70$ & --- \\
static $f{=}0.25$ & $2\times 10^{-5}$ & $47.69\,\pm\,0.19$ & $-0.80$ & $78.66\,\pm\,1.83$ & $+0.81$ \\
static $f{=}0.5$ & $2\times 10^{-5}$ & $46.60\,\pm\,0.66$ & $-1.89$ & $80.49\,\pm\,1.61$ & $+2.64$ \\
\bottomrule
\end{tabular}
\caption{HumanEval+ ($n{=}164$) SFT filtering grid.  Epochs are $4$ (BigCode) and $5$
(Magicoder); OCI HumanEval+ appears in \cref{tab:code-sft-grid-oci}.  Values are
percent, mean $\pm$ sample standard deviation over three seeds;
$\Delta$ is the change against the no-filter run for the same dataset; the
\emph{base (no SFT)} row is the untuned OLMo-3 7B base under the same protocol.}
\label{tab:code-sft-grid-heplus}
\end{table}

\begin{table}[t]
\centering
\small
\begin{tabular}{llrrrr}
\toprule
Filter & LR & pass@1 & $\Delta$ & pass@16 & $\Delta$ \\
\midrule
base (no SFT) & --- & $25.67\,\pm\,0.62$ & --- & $70.50\,\pm\,1.54$ & --- \\
\addlinespace[2pt]
\multicolumn{6}{l}{\textit{BigCode}} \\
no filter & $3\times 10^{-5}$ & $29.79\,\pm\,0.48$ & --- & $61.71\,\pm\,2.27$ & --- \\
static $f{=}0.125$ & $2\times 10^{-5}$ & $30.65\,\pm\,0.26$ & $+0.86$ & $63.54\,\pm\,0.76$ & $+1.83$ \\
static $f{=}0.25$ & $2\times 10^{-5}$ & $30.43\,\pm\,0.15$ & $+0.64$ & $65.79\,\pm\,0.26$ & $+4.08$ \\
static $f{=}0.5$ & $2\times 10^{-5}$ & $29.76\,\pm\,0.35$ & $-0.03$ & $69.29\,\pm\,1.92$ & $+7.58$ \\
\addlinespace[2pt]
\multicolumn{6}{l}{\textit{Magicoder}} \\
no filter & $3\times 10^{-5}$ & $28.50\,\pm\,0.13$ & --- & $59.75\,\pm\,0.62$ & --- \\
static $f{=}0.25$ & $2\times 10^{-5}$ & $26.48\,\pm\,0.45$ & $-2.02$ & $68.08\,\pm\,0.94$ & $+8.33$ \\
static $f{=}0.5$ & $2\times 10^{-5}$ & $25.63\,\pm\,0.56$ & $-2.88$ & $69.79\,\pm\,0.38$ & $+10.04$ \\
\addlinespace[2pt]
\multicolumn{6}{l}{\textit{OCI}} \\
no filter & $2\times 10^{-5}$ & $29.46\,\pm\,0.46$ & --- & $68.08\,\pm\,1.77$ & --- \\
ramp $0{\to}0.5$ & $2\times 10^{-5}$ & $30.17\,\pm\,0.64$ & $+0.71$ & $71.58\,\pm\,1.91$ & $+3.50$ \\
\bottomrule
\end{tabular}
\caption{CruxEval-I ($n{=}800$) SFT filtering grid.  Epochs are $4$ / $5$ / $3$ for
BigCode / Magicoder / OCI.  Values are percent, mean $\pm$ sample standard
deviation over three seeds; $\Delta$ is the change against the
no-filter run for the same dataset; the \emph{base (no SFT)} row is the untuned
OLMo-3 7B base under the same protocol.}
\label{tab:code-sft-grid-cruxi}
\end{table}

\begin{table}[t]
\centering
\small
\begin{tabular}{llrrrr}
\toprule
Filter & LR & pass@1 & $\Delta$ & pass@16 & $\Delta$ \\
\midrule
base (no SFT) & --- & $4.22\,\pm\,0.74$ & --- & $32.08\,\pm\,1.98$ & --- \\
\addlinespace[2pt]
\multicolumn{6}{l}{\textit{BigCode}} \\
no filter & $3\times 10^{-5}$ & $4.16\,\pm\,0.44$ & --- & $24.21\,\pm\,2.38$ & --- \\
static $f{=}0.125$ & $2\times 10^{-5}$ & $12.02\,\pm\,0.11$ & $+7.86$ & $39.75\,\pm\,0.76$ & $+15.54$ \\
static $f{=}0.25$ & $2\times 10^{-5}$ & $13.18\,\pm\,0.30$ & $+9.02$ & $41.00\,\pm\,1.35$ & $+16.79$ \\
static $f{=}0.5$ & $2\times 10^{-5}$ & $7.35\,\pm\,0.37$ & $+3.20$ & $36.54\,\pm\,0.14$ & $+12.33$ \\
\addlinespace[2pt]
\multicolumn{6}{l}{\textit{Magicoder}} \\
no filter & $3\times 10^{-5}$ & $12.86\,\pm\,0.19$ & --- & $38.08\,\pm\,0.94$ & --- \\
static $f{=}0.25$ & $2\times 10^{-5}$ & $16.16\,\pm\,0.53$ & $+3.30$ & $47.92\,\pm\,1.66$ & $+9.83$ \\
static $f{=}0.5$ & $2\times 10^{-5}$ & $10.12\,\pm\,0.37$ & $-2.74$ & $41.67\,\pm\,0.63$ & $+3.58$ \\
\addlinespace[2pt]
\multicolumn{6}{l}{\textit{OCI}} \\
no filter & $2\times 10^{-5}$ & $9.30\,\pm\,0.49$ & --- & $40.12\,\pm\,0.00$ & --- \\
ramp $0{\to}0.5$ & $2\times 10^{-5}$ & $10.84\,\pm\,0.37$ & $+1.54$ & $44.46\,\pm\,0.47$ & $+4.33$ \\
\bottomrule
\end{tabular}
\caption{CruxEval-O ($n{=}800$) SFT filtering grid.  Epochs are $4$ / $5$ / $3$ for
BigCode / Magicoder / OCI.  Values are percent, mean $\pm$ sample standard
deviation over three seeds; $\Delta$ is the change against the
no-filter run for the same dataset; the \emph{base (no SFT)} row is the untuned
OLMo-3 7B base under the same protocol.}
\label{tab:code-sft-grid-cruxo}
\end{table}

\begin{table}[t]
\centering
\small
\begin{tabular}{llrrrr}
\toprule
Filter & LR & pass@1 & $\Delta$ & pass@16 & $\Delta$ \\
\midrule
base (no SFT) & --- & $5.50\,\pm\,0.42$ & --- & $32.30\,\pm\,0.93$ & --- \\
\addlinespace[2pt]
\multicolumn{6}{l}{\textit{BigCode}} \\
no filter & $3\times 10^{-5}$ & $10.74\,\pm\,0.46$ & --- & $30.39\,\pm\,0.86$ & --- \\
static $f{=}0.125$ & $2\times 10^{-5}$ & $11.77\,\pm\,0.44$ & $+1.03$ & $31.54\,\pm\,0.71$ & $+1.14$ \\
static $f{=}0.25$ & $2\times 10^{-5}$ & $11.49\,\pm\,0.38$ & $+0.75$ & $33.12\,\pm\,0.34$ & $+2.72$ \\
static $f{=}0.5$ & $2\times 10^{-5}$ & $11.45\,\pm\,0.42$ & $+0.71$ & $33.39\,\pm\,0.25$ & $+3.00$ \\
\addlinespace[2pt]
\multicolumn{6}{l}{\textit{Magicoder}} \\
no filter & $3\times 10^{-5}$ & $14.84\,\pm\,0.11$ & --- & $31.81\,\pm\,0.34$ & --- \\
static $f{=}0.25$ & $2\times 10^{-5}$ & $14.32\,\pm\,0.38$ & $-0.51$ & $33.22\,\pm\,0.34$ & $+1.42$ \\
static $f{=}0.5$ & $2\times 10^{-5}$ & $14.81\,\pm\,0.22$ & $-0.02$ & $34.48\,\pm\,1.28$ & $+2.67$ \\
\addlinespace[2pt]
\multicolumn{6}{l}{\textit{OCI}} \\
no filter & $2\times 10^{-5}$ & $12.24\,\pm\,0.36$ & --- & $34.59\,\pm\,0.90$ & --- \\
ramp $0{\to}0.5$ & $2\times 10^{-5}$ & $11.49\,\pm\,0.84$ & $-0.75$ & $33.50\,\pm\,0.16$ & $-1.09$ \\
\bottomrule
\end{tabular}
\caption{LiveCodeBench ($n{=}612$) SFT filtering grid.  Epochs are $4$ / $5$ / $3$ for
BigCode / Magicoder / OCI.  Values are percent, mean $\pm$ sample standard
deviation over three seeds; $\Delta$ is the change against the
no-filter run for the same dataset; the \emph{base (no SFT)} row is the untuned
OLMo-3 7B base under the same protocol.}
\label{tab:code-sft-grid-lcb}
\end{table}

\paragraph{Evaluation protocol}
We evaluate with the OLMES~\citep{gu2025olmes} EvalPlus configuration used across the OLMo-3 code
suite, so the appendix numbers are directly comparable to that standard.  Every
code result is sampled at temperature $1.0$, top-$p=1.0$, with $16$ generations
per problem for each of three independent evaluation seed sets; throughout the
appendix, reported means and standard deviations are computed over these three
seed sets, which we refer to simply as three seeds.  MBPP+ uses the $378$-problem EvalPlus test split with the
assistant prefix
\texttt{Here is the completed function:\textbackslash n\textbackslash n```python\textbackslash n},
maximum generation length $2048$, context length $4096$, and stop strings
\texttt{```}, newline-triple-quote, newline-\texttt{assert}, and newline-comment.
HumanEval+ uses the $164$-problem EvalPlus split with the same prefix, generation
length $1024$, and context $4096$.  CruxEval-I and CruxEval-O use $800$ instances
each with OLMo-3 multiturn few-shot prompts (two shots for input prediction, one
for output), generation length $512$, and context $4096$.  LiveCodeBench uses
$612$ code-generation problems with an expert-Python system prompt, generation
length $8192$, and context $32768$.  We report pass@$k$ with the standard unbiased
estimator implemented in OLMES, $\mathrm{pass}@k=1-\binom{n-c}{k}\big/\binom{n}{k}$
for a document with $c$ correct completions out of $n$ samples, averaged over
documents rather than as a best-of-$k$ maximum.  With $n{=}16$ samples this makes
pass@1 the mean per-document pass rate and pass@16 the fraction of documents with at
least one correct completion.  Each generation length sits above the lengths these
tasks actually decode to, so completions are not cut off.

\begin{table}[t]
\centering
\footnotesize
\setlength{\tabcolsep}{4pt}
\begin{tabular}{llrrrr}
\toprule
SFT data & Benchmark & Standard p@1 & \tailsft{} p@1 & Standard p@16 & \tailsft{} p@16 \\
\midrule
BigCode & MBPP+ ($n=378$) & $53.04_{\pm0.93}$ & $51.36_{\pm0.14}$ & $74.69_{\pm1.10}$ & $78.84_{\pm1.06}$ \\
BigCode & HumanEval+ ($n=164$) & $45.06_{\pm0.57}$ & $46.33_{\pm1.07}$ & $76.63_{\pm0.70}$ & $80.08_{\pm0.93}$ \\
BigCode & CruxEval-I ($n=800$) & $29.79_{\pm0.48}$ & $30.43_{\pm0.15}$ & $61.71_{\pm2.27}$ & $65.79_{\pm0.26}$ \\
BigCode & CruxEval-O ($n=800$) & $4.16_{\pm0.44}$ & $13.18_{\pm0.30}$ & $24.21_{\pm2.38}$ & $41.00_{\pm1.35}$ \\
BigCode & LiveCodeBench ($n=612$) & $10.74_{\pm0.46}$ & $11.49_{\pm0.38}$ & $30.39_{\pm0.86}$ & $33.12_{\pm0.34}$ \\
\addlinespace[1pt]
Magicoder & MBPP+ ($n=378$) & $55.35_{\pm0.47}$ & $54.60_{\pm0.55}$ & $75.31_{\pm0.93}$ & $78.66_{\pm1.25}$ \\
Magicoder & HumanEval+ ($n=164$) & $48.49_{\pm0.58}$ & $47.69_{\pm0.19}$ & $77.85_{\pm0.70}$ & $78.66_{\pm1.83}$ \\
Magicoder & CruxEval-I ($n=800$) & $28.50_{\pm0.13}$ & $26.48_{\pm0.45}$ & $59.75_{\pm0.62}$ & $68.08_{\pm0.94}$ \\
Magicoder & CruxEval-O ($n=800$) & $12.86_{\pm0.19}$ & $16.16_{\pm0.53}$ & $38.08_{\pm0.94}$ & $47.92_{\pm1.66}$ \\
Magicoder & LiveCodeBench ($n=612$) & $14.84_{\pm0.11}$ & $14.32_{\pm0.38}$ & $31.81_{\pm0.34}$ & $33.22_{\pm0.34}$ \\
\addlinespace[1pt]
OCI & CruxEval-I ($n=800$) & $29.46_{\pm0.46}$ & $30.17_{\pm0.64}$ & $68.08_{\pm1.77}$ & $71.58_{\pm1.91}$ \\
OCI & CruxEval-O ($n=800$) & $9.30_{\pm0.49}$ & $10.84_{\pm0.37}$ & $40.12_{\pm0.00}$ & $44.46_{\pm0.47}$ \\
OCI & LiveCodeBench ($n=612$) & $12.24_{\pm0.36}$ & $11.49_{\pm0.84}$ & $34.59_{\pm0.90}$ & $33.50_{\pm0.16}$ \\
\bottomrule
\end{tabular}
\caption{Post-SFT code results, matching \cref{tab:sft-main}.  Values are
percentages, mean $\pm$ sample standard deviation over three seeds.}
\label{tab:code-sft-results}
\end{table}

\paragraph{The coverage advantage tends to widen with $k$}
\Cref{tab:sft-passk-delta-code} reports the \tailsft{} minus Standard SFT
pass@$k$ gap at $k\in\{1,2,4,8,16\}$ for every benchmark.  On most cells the gap
starts near zero (or slightly negative) at $k{=}1$ and tends to increase with $k$---most
strikingly CruxEval-O, where BigCode grows from $+9.02$ to $+16.79$---so the
advantage is a coverage effect rather than a pass@1 shift.  The two OCI transfer
benchmarks (HumanEval+ and LiveCodeBench) are the main exception, where \tailsft{}
trails at every $k$.

\begin{table}[t]
\centering
\footnotesize
\caption{SFT pass@k deltas for code benchmarks, reported as \tailsft{} minus Standard SFT in percentage points at the main-table settings. Deltas are computed with the unbiased OLMES~\citep{gu2025olmes} pass@k estimator.}
\label{tab:sft-passk-delta-code}
\begin{tabular}{@{}llrrrrr@{}}
\toprule
SFT data & Benchmark & $\Delta$p@1 & $\Delta$p@2 & $\Delta$p@4 & $\Delta$p@8 & $\Delta$p@16 \\
\midrule
BigCode & MBPP+ ($n{=}378$) & $-1.68$ & $+0.64$ & $+2.19$ & $+3.19$ & $+4.14$ \\
 & HumanEval+ ($n{=}164$) & $+1.27$ & $+3.37$ & $+4.34$ & $+4.44$ & $+3.46$ \\
 & CruxEval-I ($n{=}800$) & $+0.64$ & $+1.61$ & $+2.57$ & $+3.25$ & $+4.08$ \\
 & CruxEval-O ($n{=}800$) & $+9.02$ & $+12.58$ & $+15.45$ & $+16.95$ & $+16.79$ \\
 & LiveCodeBench ($n{=}612$) & $+0.75$ & $+1.39$ & $+1.86$ & $+2.21$ & $+2.72$ \\
\midrule
Magicoder & MBPP+ ($n{=}378$) & $-0.75$ & $+1.73$ & $+2.95$ & $+3.25$ & $+3.35$ \\
 & HumanEval+ ($n{=}164$) & $-0.80$ & $+0.81$ & $+1.55$ & $+1.64$ & $+0.81$ \\
 & CruxEval-I ($n{=}800$) & $-2.02$ & $-0.10$ & $+2.36$ & $+5.28$ & $+8.33$ \\
 & CruxEval-O ($n{=}800$) & $+3.30$ & $+4.97$ & $+6.70$ & $+8.38$ & $+9.83$ \\
 & LiveCodeBench ($n{=}612$) & $-0.51$ & $+0.30$ & $+0.80$ & $+1.08$ & $+1.42$ \\
\midrule
OCI & MBPP+ ($n{=}378$) & $-2.44$ & $-1.21$ & $-0.17$ & $+0.65$ & $+1.15$ \\
 & HumanEval+ ($n{=}164$) & $-2.82$ & $-2.38$ & $-1.89$ & $-2.02$ & $-2.74$ \\
 & CruxEval-I ($n{=}800$) & $+0.71$ & $+1.54$ & $+2.46$ & $+3.19$ & $+3.50$ \\
 & CruxEval-O ($n{=}800$) & $+1.54$ & $+2.33$ & $+3.14$ & $+3.86$ & $+4.33$ \\
 & LiveCodeBench ($n{=}612$) & $-0.75$ & $-0.86$ & $-0.91$ & $-0.91$ & $-1.09$ \\
\bottomrule
\end{tabular}
\end{table}

\subsubsection{Math}
\label{app:sft-math}

\paragraph{Model and optimization}
Math SFT trains at sequence length $4096$ with weight decay $0.1$, a per-rank
microbatch of $2$ with $4$ gradient-accumulation steps, for an effective batch of
$64$ and a $16$-example selection batch.  The learning rate is fixed at
$3\times10^{-5}$.  We train for $2$ epochs and evaluate the final checkpoint.  Two
passes fit the $\sim\!350$k-example subset without the overfitting seen at higher
counts, and the count is held identical across every arm so differences come only
from filtering.  Each supervised example is a two-turn chat.  The user message is
the problem followed by the instruction ``\texttt{Present the answer in LaTex
format: \textbackslash boxed\{Your answer\}}'', and the assistant target is the
provided chain-of-thought solution.

\paragraph{Data}
The OpenMathInstruct-2 (OMI) subset is built from the OMI \texttt{train\_1M}
shards by
keeping \texttt{problem\_source} in \{\texttt{math}, \texttt{augmented\_math}\},
exact-match decontaminating against MATH-500 problem strings, dropping rows whose
problem plus solution exceeds $2048$ tokens under the OLMo-3 tokenizer, and
sampling a $350$k-example stratified subset with a $95/5$ train/validation split.
No MATH-500 rows were found during decontamination.  Since rows over $2048$ tokens
were removed during data construction, the $4096$-token training sequence length
does not truncate supervised targets.  The final SFT parquets add
\texttt{init\_ce}, the base-model per-example cross-entropy used by the filtering
rule.

\begin{table}[h]
\centering
\small
\begin{tabular}{@{}llrr@{}}
\toprule
Dataset & Source/filter & Train rows & Validation rows \\
\midrule
OMI math subset & OpenMathInstruct-2 \texttt{math}/\texttt{augmented\_math}, MATH-500 decontaminated & 332{,}500 & 17{,}500 \\
\bottomrule
\end{tabular}
\caption{In-domain math SFT data.  The annotated training split contains
$326{,}725$ \texttt{augmented\_math} and $5{,}775$ \texttt{math} rows; the
validation split contains $17{,}192$ and $308$, respectively.}
\label{tab:app-sft-math-data}
\end{table}

\paragraph{Hyperparameter search and selection}
For math we fixed the learning rate at $3\times10^{-5}$ and swept the filtering
configuration over no filtering, static $f\in\{0.0625,0.125,0.1875,0.25,0.5\}$,
and the linear ramps $0\!\to\!f$ for the same endpoints.  The optimizer, sequence
length, batch size, and $2$-epoch schedule are held fixed as shown in
\cref{tab:app-sft-math-hp-grid}.  
\Cref{tab:sft-main} then reports the no-filter Standard checkpoint against the
selected \tailsft{} checkpoint.  We use the ramp $0\!\to\!0.5$ arm.

\begin{table}[t]
\centering
\small
\begin{tabular}{p{0.32\linewidth} p{0.55\linewidth}}
\toprule
Hyperparameter & Values \\
\midrule
Base model & OLMo-3 7B base \\
Learning rate & fixed at $3\times10^{-5}$ \\
Filter schedule & none; static; ramp $0\!\to\!\cdot$ \\
Filter fraction $f$ & $0$ (none); static $\{0.0625,0.125,0.1875,0.25,0.5\}$; ramp $0\!\to\!\{0.0625,0.125,0.1875,0.25,0.5\}$ \\
Epochs & $2$ \\
Sequence length & $4096$ \\
Effective batch & $64$ ($8\times2\times4$) \\
Warmup schedule & $3\%$ linear warmup, then constant \\
Optimizer & AdamW $(0.9,0.95)$, wd $0.1$, clip $1.0$ \\
\bottomrule
\end{tabular}
\caption{Math SFT hyperparameter grid.}
\label{tab:app-sft-math-hp-grid}
\end{table}

\paragraph{Per-benchmark grids}
Tables~\ref{tab:math-sft-grid}--\ref{tab:math-sft-grid-aime} report every
completed math SFT filtering evaluation in-domain math
benchmarks.  The broadest grid is MATH-500.  Its $500$-problem aggregate
(\cref{tab:math-sft-grid}) sits near ceiling, so
\cref{tab:math-sft-grid-p16-levels,tab:math-sft-grid-p1-levels} break the same
sweep out by the five MATH difficulty levels.  No-filter pass@16 is $100\%$ at
Level~1 and stays above $96\%$ through Level~3, so the easy levels dominate the
aggregate and hide where filtering acts.  The unreached coverage, and the gains
from filtering, concentrate at the hardest levels: at Level~5 filtering matches
or improves no-filter pass@16 in $8$ of $10$ arms and strictly improves it in
$7$, while the near-saturated levels move little in either direction.  No
MATH-500 pass@16 cell drops by more than $1.8$~pp, inside the seed noise.
The completed AIME rows all improve pass@16, while the frozen OMEGA-500 arm is
essentially tied with no filtering.

\begin{table}[t]
\centering
\small
\begin{tabular}{lrrrr}
\toprule
Filter & pass@1 & $\Delta$ & pass@16 & $\Delta$ \\
\midrule
no filter & $53.55\,\pm\,0.49$ & --- & $86.27\,\pm\,0.61$ & --- \\
static $f{=}0.0625$ & $53.66\,\pm\,0.28$ & $+0.11$ & $87.00\,\pm\,0.87$ & $+0.73$ \\
static $f{=}0.125$ & $53.60\,\pm\,0.39$ & $+0.05$ & $87.20\,\pm\,0.40$ & $+0.93$ \\
static $f{=}0.1875$ & $53.15\,\pm\,0.33$ & $-0.41$ & $86.20\,\pm\,0.69$ & $-0.07$ \\
static $f{=}0.25$ & $53.51\,\pm\,0.41$ & $-0.04$ & $86.33\,\pm\,1.29$ & $+0.07$ \\
static $f{=}0.5$ & $52.02\,\pm\,0.13$ & $-1.54$ & $85.60\,\pm\,0.72$ & $-0.67$ \\
ramp $0{\to}0.0625$ & $53.80\,\pm\,0.29$ & $+0.25$ & $85.87\,\pm\,0.31$ & $-0.40$ \\
ramp $0{\to}0.125$ & $53.37\,\pm\,0.73$ & $-0.19$ & $86.47\,\pm\,0.50$ & $+0.20$ \\
ramp $0{\to}0.1875$ & $54.13\,\pm\,0.50$ & $+0.57$ & $87.27\,\pm\,0.95$ & $+1.00$ \\
ramp $0{\to}0.25$ & $53.91\,\pm\,0.19$ & $+0.35$ & $86.13\,\pm\,0.46$ & $-0.13$ \\
ramp $0{\to}0.5$ & $53.59\,\pm\,0.41$ & $+0.03$ & $86.27\,\pm\,1.01$ & $+0.00$ \\
\bottomrule
\end{tabular}
\caption{MATH-500 full ($n{=}500$) SFT filtering grid, $4096$-token generations.
Values are percent, mean $\pm$ sample standard deviation over three generation
seeds; $\Delta$ is the change against no filtering.}
\label{tab:math-sft-grid}
\end{table}

\begin{table}[t]
\centering
\small
\setlength{\tabcolsep}{5pt}
\begin{tabular}{lrrrrr}
\toprule
Filter & Level 1 ($n{=}43$) & Level 2 ($n{=}90$) & Level 3 ($n{=}105$) & Level 4 ($n{=}128$) & Level 5 ($n{=}134$) \\
\midrule
no filter & $100.00\,\pm\,0.00$ & $97.78\,\pm\,1.11$ & $96.51\,\pm\,1.98$ & $85.68\,\pm\,1.63$ & $66.67\,\pm\,1.88$ \\
static $f{=}0.0625$ & $100.00\,\pm\,0.00$ & $98.15\,\pm\,0.64$ & $96.19\,\pm\,0.95$ & $85.68\,\pm\,0.90$ & $69.40\,\pm\,3.25$ \\
static $f{=}0.125$ & $100.00\,\pm\,0.00$ & $97.41\,\pm\,0.64$ & $96.83\,\pm\,1.10$ & $85.68\,\pm\,2.51$ & $70.15\,\pm\,1.29$ \\
static $f{=}0.1875$ & $100.00\,\pm\,0.00$ & $97.78\,\pm\,0.00$ & $95.56\,\pm\,1.46$ & $84.38\,\pm\,0.78$ & $68.41\,\pm\,1.88$ \\
static $f{=}0.25$ & $100.00\,\pm\,0.00$ & $97.41\,\pm\,0.64$ & $96.51\,\pm\,0.55$ & $85.94\,\pm\,3.40$ & $66.92\,\pm\,2.62$ \\
static $f{=}0.5$ & $100.00\,\pm\,0.00$ & $96.30\,\pm\,1.70$ & $95.56\,\pm\,1.98$ & $85.68\,\pm\,3.25$ & $65.92\,\pm\,4.24$ \\
ramp $0{\to}0.0625$ & $99.22\,\pm\,1.34$ & $97.04\,\pm\,0.64$ & $97.14\,\pm\,0.95$ & $84.38\,\pm\,0.78$ & $66.67\,\pm\,1.55$ \\
ramp $0{\to}0.125$ & $100.00\,\pm\,0.00$ & $97.04\,\pm\,1.70$ & $96.83\,\pm\,1.46$ & $83.85\,\pm\,1.19$ & $69.40\,\pm\,2.69$ \\
ramp $0{\to}0.1875$ & $100.00\,\pm\,0.00$ & $97.41\,\pm\,0.64$ & $96.51\,\pm\,0.55$ & $85.68\,\pm\,1.63$ & $70.65\,\pm\,4.97$ \\
ramp $0{\to}0.25$ & $100.00\,\pm\,0.00$ & $97.78\,\pm\,1.11$ & $97.78\,\pm\,1.46$ & $84.90\,\pm\,0.45$ & $65.92\,\pm\,1.14$ \\
ramp $0{\to}0.5$ & $100.00\,\pm\,0.00$ & $96.67\,\pm\,0.00$ & $97.14\,\pm\,0.00$ & $85.42\,\pm\,1.97$ & $67.16\,\pm\,1.97$ \\
\bottomrule
\end{tabular}
\caption{MATH-500 SFT \textbf{pass@16} by difficulty level, from the same $4096$-token
SFT filtering sweep as \cref{tab:math-sft-grid}.  Columns are the five MATH difficulty
levels, with the number of problems in each header.  Values are percent, mean
$\pm$ sample standard deviation over three seeds.  pass@16 is at or
near ceiling through Level~3, so the aggregate is dominated by the easy levels;
the gains from filtering concentrate at Levels~4--5.  The frozen main table
reports the selected ramp arm on Level~5, re-evaluated with $8192$-token
generations.}
\label{tab:math-sft-grid-p16-levels}
\end{table}

\begin{table}[t]
\centering
\small
\setlength{\tabcolsep}{5pt}
\begin{tabular}{lrrrrr}
\toprule
Filter & Level 1 ($n{=}43$) & Level 2 ($n{=}90$) & Level 3 ($n{=}105$) & Level 4 ($n{=}128$) & Level 5 ($n{=}134$) \\
\midrule
no filter & $85.27\,\pm\,1.46$ & $76.25\,\pm\,0.43$ & $65.95\,\pm\,1.59$ & $47.02\,\pm\,0.65$ & $24.66\,\pm\,0.63$ \\
static $f{=}0.0625$ & $85.66\,\pm\,1.05$ & $76.44\,\pm\,1.15$ & $64.98\,\pm\,1.93$ & $46.92\,\pm\,0.21$ & $25.67\,\pm\,0.72$ \\
static $f{=}0.125$ & $84.98\,\pm\,0.89$ & $76.69\,\pm\,0.70$ & $64.46\,\pm\,0.77$ & $47.28\,\pm\,0.47$ & $25.56\,\pm\,0.73$ \\
static $f{=}0.1875$ & $84.79\,\pm\,0.75$ & $75.83\,\pm\,1.05$ & $65.06\,\pm\,0.51$ & $46.14\,\pm\,1.65$ & $25.11\,\pm\,0.66$ \\
static $f{=}0.25$ & $85.08\,\pm\,0.69$ & $75.62\,\pm\,0.77$ & $65.08\,\pm\,0.34$ & $47.18\,\pm\,0.86$ & $25.51\,\pm\,0.33$ \\
static $f{=}0.5$ & $83.58\,\pm\,0.58$ & $74.31\,\pm\,0.90$ & $62.38\,\pm\,1.11$ & $46.48\,\pm\,0.13$ & $24.08\,\pm\,0.03$ \\
ramp $0{\to}0.0625$ & $85.61\,\pm\,1.54$ & $76.06\,\pm\,0.51$ & $66.03\,\pm\,0.72$ & $47.27\,\pm\,0.30$ & $25.30\,\pm\,0.56$ \\
ramp $0{\to}0.125$ & $84.20\,\pm\,0.67$ & $76.41\,\pm\,1.28$ & $65.24\,\pm\,1.17$ & $47.20\,\pm\,1.10$ & $24.58\,\pm\,0.82$ \\
ramp $0{\to}0.1875$ & $85.90\,\pm\,0.52$ & $76.39\,\pm\,0.84$ & $65.85\,\pm\,0.30$ & $48.00\,\pm\,0.09$ & $25.65\,\pm\,1.10$ \\
ramp $0{\to}0.25$ & $85.85\,\pm\,1.24$ & $76.44\,\pm\,0.26$ & $65.66\,\pm\,0.53$ & $47.58\,\pm\,0.31$ & $25.37\,\pm\,0.69$ \\
ramp $0{\to}0.5$ & $85.27\,\pm\,0.22$ & $75.93\,\pm\,0.54$ & $65.28\,\pm\,0.61$ & $47.84\,\pm\,0.79$ & $24.75\,\pm\,1.02$ \\
\bottomrule
\end{tabular}
\caption{MATH-500 SFT \textbf{pass@1} by difficulty level, from the same sweep as
\cref{tab:math-sft-grid-p16-levels}.  Values are percent, mean $\pm$ sample
standard deviation over three seeds.}
\label{tab:math-sft-grid-p1-levels}
\end{table}

\paragraph{Reading the Level-4 numbers}
Level~4 is the one subset where filtering neither clearly helps nor hurts, and
this is an estimator artifact rather than a real regression.  pass@16 is scored
as a binary hit per seed and averaged over three seeds, so every
per-problem value is quantized to $\{0,\tfrac13,\tfrac23,1\}$ and a problem's
filtered-minus-no-filter difference can only be a multiple of $33$~pp.
The Level-4 mean is then set by a few knife's-edge problems that the base model
already solves on two of three seeds, where a single seed flip moves the
subset mean by a point or two.  Two such problems account for the whole gap:
under the arms that most help Level~5, the Level-4 pass@16 change moves from
$+0.0$~pp on the full subset to about $+0.8$~pp once those two problems are
excluded, in the same direction as Level~5.  Level~5 shows a larger and steadier
gain because it contains many more genuinely contested problems, so this
per-problem quantization averages out.

\begin{table}[t]
\centering
\small
\begin{tabular}{llrrrr}
\toprule
Eval family & Filter & pass@1 & $\Delta$ & pass@16 & $\Delta$ \\
\midrule
\multicolumn{2}{@{}l}{base (no SFT)$^{\dagger}$} & $5.16$ & --- & $20.91$ & --- \\
\addlinespace[2pt]
ramp sweep & no filter & $3.65\,\pm\,0.37$ & --- & $15.24\,\pm\,2.02$ & --- \\
ramp sweep & ramp $0{\to}0.1875$ & $3.82\,\pm\,0.20$ & $+0.16$ & $16.06\,\pm\,1.42$ & $+0.82$ \\
ramp sweep & ramp $0{\to}0.25$ & $3.78\,\pm\,0.20$ & $+0.13$ & $15.37\,\pm\,0.23$ & $+0.13$ \\
ramp sweep & ramp $0{\to}0.5$ & $4.03\,\pm\,0.02$ & $+0.37$ & $18.31\,\pm\,0.57$ & $+3.07$ \\
\addlinespace[2pt]
static & no filter & $3.65\,\pm\,0.34$ & --- & $15.52\,\pm\,2.00$ & --- \\
static & static $f{=}0.125$ & $3.75\,\pm\,0.18$ & $+0.10$ & $16.14\,\pm\,1.73$ & $+0.63$ \\
\bottomrule
\end{tabular}
\caption{AIME 2022--2025 ($n{=}120$) completed SFT filtering evaluations.  The
ramp sweep and the static pair were evaluated in separate runs, each against its
own matched no-filter baseline.  Values are percent, mean $\pm$ sample standard
deviation over three seeds; $\Delta$ is the change against the
matching no-filter row.  $^{\dagger}$The \emph{base (no SFT)} row is the untuned
OLMo-3 7B base under the same AIME protocol.}
\label{tab:math-sft-grid-aime}
\end{table}

\paragraph{Evaluation protocol}
We evaluate with the standard OLMES~\citep{gu2025olmes} math configuration used across the OLMo-3
math suite.  AIME contains $30$ problems from each year $2022$--$2025$ ($n=120$),
chat prompting with the boxed-answer instruction, temperature $0.6$, top-$p=0.95$,
maximum generation length and model context $8192$, and $32$ generations per
problem for each of three seeds.  MATH-500 uses the $500$-problem test split;
the main table reports the Level-5 subset ($n=134$), since the full benchmark is
near ceiling, with temperature $1.0$, top-$p=1.0$, maximum generation length
$8192$, model context $16384$, and $16$ generations per problem for each of three
seeds.  OMEGA-500 (the \texttt{saumyamalik/omega-500} $500$-problem set) uses
temperature $1.0$, top-$p=1.0$,
maximum generation length $4096$, model context $8192$, and $16$ generations per
problem for each of three seeds.  For AIME and MATH-500 we use the stored
per-document pass@$k$ metrics.  For OMEGA-500, whose aggregate files do not store
pass@$k$, we recompute pass@1 and pass@16 from the $16$ completions using the
extracted flexible boxed answer, falling back to the standard extracted answer,
and match it against the label.  In every case pass@$k$ is the same unbiased
estimator implemented in OLMES ($\mathrm{pass}@k=1-\binom{n-c}{k}\big/\binom{n}{k}$
for $c$ correct of $n$ samples), so AIME's $32$ samples are combined at $k{=}16$
rather than by best-of-$k$.  The sweep of \cref{tab:math-sft-grid}
uses MATH-500 full ($n=500$) with $4096$-token generations.

\begin{table}[h]
\centering
\small
\begin{tabular}{@{}llrrrr@{}}
\toprule
SFT data & Benchmark & Standard p@1 & \tailsft{} p@1 & Standard p@16 & \tailsft{} p@16 \\
\midrule
OMI & AIME 2022--2025 ($n=120$) & $3.65_{\pm0.37}$ & $4.03_{\pm0.02}$ & $15.24_{\pm2.02}$ & $18.31_{\pm0.57}$ \\
OMI & MATH-500 Level 5 ($n=134$) & $24.80_{\pm0.40}$ & $25.16_{\pm0.40}$ & $66.42_{\pm0.00}$ & $69.15_{\pm0.86}$ \\
OMI & OMEGA-500 ($n=500$) & $6.58_{\pm0.50}$ & $6.51_{\pm0.26}$ & $32.80_{\pm2.25}$ & $32.60_{\pm1.56}$ \\
\bottomrule
\end{tabular}
\caption{Post-SFT math results, matching \cref{tab:sft-main}.  Values are
percentages, mean $\pm$ sample standard deviation over three seeds.}
\label{tab:app-sft-math-results}
\end{table}

\paragraph{The coverage advantage tends to widen with $k$}
\Cref{tab:sft-passk-delta-math} reports the \tailsft{} minus Standard SFT
pass@$k$ gap at $k\in\{1,2,4,8,16\}$.  On AIME and MATH-500 the gap is small at
$k{=}1$ and tends to grow with $k$ (AIME $+0.37\!\to\!+3.07$; MATH-500 Level~5
$+0.36\!\to\!+2.74$), while OMEGA-500, where the two models are already matched,
stays flat.  The math gains therefore come from broader coverage that surfaces at
larger $k$ rather than from a shift in greedy accuracy.

\begin{table}[t]
\centering
\footnotesize
\caption{SFT pass@k deltas for math benchmarks, reported as \tailsft{} minus Standard SFT in percentage points at the main-table settings. Deltas are computed with the unbiased OLMES~\citep{gu2025olmes} pass@k estimator.}
\label{tab:sft-passk-delta-math}
\begin{tabular}{@{}lrrrrr@{}}
\toprule
Benchmark & $\Delta$p@1 & $\Delta$p@2 & $\Delta$p@4 & $\Delta$p@8 & $\Delta$p@16 \\
\midrule
AIME 2022--2025 ($n{=}120$) & $+0.37$ & $+0.72$ & $+1.20$ & $+1.93$ & $+3.07$ \\
MATH-500 Level 5 ($n{=}134$) & $+0.36$ & $+0.71$ & $+0.83$ & $+0.87$ & $+2.74$ \\
OMEGA-500 ($n{=}500$) & $-0.06$ & $-0.09$ & $-0.12$ & $-0.23$ & $-0.20$ \\
\bottomrule
\end{tabular}
\end{table}

\subsection{GRPO}
\label{app:grpo}

The in-domain math and code GRPO runs share one training procedure and differ only
in data, reward, hyperparameter grid, and evaluation.  We describe the shared
procedure here; the math (\cref{app:grpo-math}) and code (\cref{app:grpo-code})
subsections give the domain-specific details.  Every run is initialized from one of
the matched SFT checkpoints evaluated in \cref{tab:sft-main}: the Standard-SFT run
supplies the no-filter start and the paired \tailsft{} run supplies the filtered
start, so the two policies differ only in how their initialization was trained.
\Cref{tab:grpo-results} reports a single tuned configuration per domain; the
per-domain grids below give the tuned-configuration results in full.

\paragraph{Reinforcement learning with GRPO}
Each run optimizes the initialized policy against a verifiable reward with
GRPO~\citep{shao2024deepseekmath}.  For every prompt we sample a group of $n$
responses at temperature $1.0$, score each response with the domain reward, and use
the group-relative advantage---each response's reward centered and scaled within its
group---as the policy-gradient signal.  We use no KL regularization and no learned
reward model.  Optimization uses a batch of $128$ prompts, a mini-batch of $128$, and
$8$-way data parallelism, with the rollout group size $n$ and the actor learning rate
searched per domain and all other settings held fixed within a domain.  Checkpoints
are exported periodically during training and the reported checkpoint is chosen on a
held-out validation split, as described per domain.  Post-GRPO evaluation uses OLMES~\citep{gu2025olmes}
and reuses the same per-benchmark coverage protocols as \cref{app:sft}, with pass@$k$
estimated over three seeds; the exact per-benchmark sampling settings are
given below.

\subsubsection{Math}
\label{app:grpo-math}

\paragraph{Initialization}
The Standard start is the no-filter OMI SFT checkpoint and the \tailsft{} start is the
matched ramp $0\!\to\!0.5$ checkpoint (\cref{app:sft-math}); these are the two OMI
checkpoints reported in \cref{tab:sft-main}.

\paragraph{Data and reward}
GRPO trains on the MATH training split (the MATH-lighteval release) with the MATH-500
test problems removed by exact-match decontamination, leaving $11{,}396$ training and
$600$ validation prompts, so the evaluation set is never seen during RL.  Each prompt
uses the same boxed-answer instruction as math SFT, and the reward is $1$ when a
response's boxed final answer matches the reference under the verifiable math grader
and $0$ otherwise.

\paragraph{Optimization}
We use a maximum prompt length of $512$, a maximum response length of $4096$, and
train for $3$ epochs.  At rollout group size $n{=}4$ we search the actor learning
rate; \cref{tab:grpo-grid-math} reports the tuned configuration (actor learning rate
$2\times10^{-5}$).  Doubling the rollout group size to $n{=}8$ leaves the ordering
unchanged---the \tailsft{} initialization retains a post-GRPO \passat{1} advantage of
about $+2$ points on MATH-500 Level 5---so the gain is not an artifact of the group
size.  \Cref{tab:grpo-math-hp-grid} summarizes the configuration.

\begin{table}[t]
\centering
\small
\begin{tabular}{p{0.34\linewidth} p{0.55\linewidth}}
\toprule
Setting & Value \\
\midrule
Initialization & matched Standard / \tailsft{} OMI SFT checkpoints (\cref{tab:sft-main}) \\
Algorithm & GRPO, group-relative advantage, no KL \\
Training prompts & MATH-lighteval train, MATH-500 removed ($11{,}396$ train / $600$ val) \\
Reward & verifiable boxed-answer exact match ($1$/$0$) \\
Rollout group size $n$ & $4$ \\
Actor learning rate & searched $\{1,2\}\times10^{-5}$; reported $2\times10^{-5}$ \\
Max prompt / response length & $512$ / $4096$ \\
Epochs & $3$ \\
Prompt batch / mini-batch & $128$ / $128$ \\
Checkpoint selection & mean-best (highest mean validation reward) \\
Evaluation & MATH-500 L5 (temp $1.0$, $16$ samples) \& AIME (temp $0.6$, $32$ samples); $8192$-token, $3$ seeds \\
\bottomrule
\end{tabular}
\caption{Math GRPO configuration.  All runs share the fixed settings; only the actor
learning rate is searched, and \cref{tab:grpo-grid-math} reports the tuned
configuration.}
\label{tab:grpo-math-hp-grid}
\end{table}

\paragraph{Checkpoint selection and evaluation}
We select the mean-best checkpoint---the one with the highest mean validation reward
over the sampled group---which averages over the small held-out validation split to
reduce selection noise.  We evaluate on the MATH-500 Level-5 subset ($n{=}134$) and
AIME 2022--2025 ($n{=}120$).  Because these are long-form reasoning tasks we use a
generation budget of $8192$ tokens and report only these long-generation evaluations,
since shorter budgets truncate reasoning traces and understate pass rates; at $8192$
tokens truncation is not a concern, as no AIME generation and fewer than $4\%$ of
MATH-500 generations reach the limit.  Matching the math SFT protocol
(\cref{app:sft-math}), MATH-500 Level~5 is sampled at temperature $1.0$,
top-$p{=}1.0$ with $16$ samples per problem, and AIME at temperature $0.6$,
top-$p{=}0.95$ with $32$ samples per problem, each over three seeds.  At the tuned
configuration the \tailsft{}
initialization improves both pass@1 and pass@16 on both benchmarks
(\cref{tab:grpo-grid-math}).

\begin{table}[t]
\centering
\small
\providecommand{\dgain}[1]{\textcolor{ForestGreen}{#1}}
\providecommand{\dnull}[1]{\textcolor{yellow!65!black}{#1}}
\begin{tabular}{l rrr rrr}
\toprule
 & \multicolumn{3}{c}{\textbf{pass@1}} & \multicolumn{3}{c}{\textbf{pass@16}} \\
\cmidrule(lr){2-4}\cmidrule(lr){5-7}
Benchmark & Standard & \tailsft{} & $\Delta$ & Standard & \tailsft{} & $\Delta$ \\
\midrule
MATH-500 Level 5 ($n{=}134$) & $57.70\,\pm\,0.44$ & $60.26\,\pm\,1.02$ & \dgain{$+2.56$} & $83.83\,\pm\,1.72$ & $87.06\,\pm\,0.86$ & \dgain{$+3.23$} \\
AIME 2022--2025 ($n{=}120$)  & $14.40\,\pm\,0.26$ & $15.61\,\pm\,0.44$ & \dgain{$+1.21$} & $32.76\,\pm\,0.88$ & $36.06\,\pm\,0.98$ & \dgain{$+3.30$} \\
\bottomrule
\end{tabular}
\caption{Math GRPO at the tuned configuration (rollout $n{=}4$, actor learning rate
$2\times10^{-5}$, mean-best checkpoint, $8192$-token generation), evaluated from the
matched Standard and \tailsft{} initializations of \cref{tab:sft-main}.  Values are
percent, mean $\pm$ standard deviation over three seeds; $\Delta$ is the
\tailsft{} improvement.}
\label{tab:grpo-grid-math}
\end{table}

\subsubsection{Code}
\label{app:grpo-code}

\paragraph{Initialization}
For each dataset the Standard start is the no-filter SFT checkpoint and the \tailsft{}
start is the matched filtered checkpoint from \cref{app:sft-code} (static $f{=}0.25$
for BigCode and Magicoder, the ramp $0\!\to\!0.5$ for OCI); these are the checkpoints
reported in \cref{tab:sft-main}.

\paragraph{Data and reward}
GRPO trains on the MBPP+ training split with a held-out $100$-problem validation
subset.  The reward is $1$ when a sampled program passes the full MBPP+ unit-test
suite for its problem and $0$ otherwise.

\paragraph{Optimization}
We use a maximum prompt and response length of $1024$ each and train for $10$ epochs.
We search the rollout group size $n\in\{2,4\}$ and the actor learning rate
$\in\{5\times10^{-6},1\times10^{-5},2\times10^{-5}\}$; \cref{tab:grpo-grid-code}
reports the tuned configuration for each dataset.  \Cref{tab:grpo-code-hp-grid}
summarizes the configuration.

\begin{table}[t]
\centering
\small
\begin{tabular}{p{0.34\linewidth} p{0.55\linewidth}}
\toprule
Setting & Value \\
\midrule
Initialization & matched Standard / \tailsft{} SFT checkpoints per dataset (\cref{tab:sft-main}) \\
Algorithm & GRPO, group-relative advantage, no KL \\
Training prompts & MBPP+ train ($100$-problem held-out validation) \\
Reward & full MBPP+ unit-test pass ($1$/$0$) \\
Rollout group size $n$ & searched $\{2,4\}$; reported $4$ \\
Actor learning rate & searched $\{5\times10^{-6},1\times10^{-5},2\times10^{-5}\}$; reported $2\times10^{-5}$ \\
Max prompt / response length & $1024$ / $1024$ \\
Epochs & $10$ \\
Prompt batch / mini-batch & $128$ / $128$ \\
Checkpoint selection & highest validation pass rate \\
Evaluation & EvalPlus MBPP+ ($n{=}378$); $2048$-token, temp $1.0$, top-$p\,1.0$, $16$ samples $\times$ $3$ seeds \\
\bottomrule
\end{tabular}
\caption{Code GRPO configuration.  All runs share the fixed settings; the rollout
group size and actor learning rate are searched, and \cref{tab:grpo-grid-code}
reports the tuned configuration for each dataset.}
\label{tab:grpo-code-hp-grid}
\end{table}

\paragraph{Checkpoint selection and evaluation}
We select the checkpoint with the highest validation pass rate over the small
$100$-problem held-out split, applying the same rule to both initializations.  Evaluation uses
EvalPlus MBPP+ ($378$-problem test split) at temperature $1.0$, top-$p{=}1.0$, with
$16$ samples per problem over three seeds and a $2048$-token generation budget;
MBPP+ solutions are short relative to this budget, so generations terminate well
within the limit and truncation does not affect the reported rates.  At the tuned
configurations the \tailsft{} initialization improves post-GRPO pass@1 while generally retaining
its pass@16 coverage advantage (\cref{tab:grpo-grid-code}).

\begin{table}[t]
\centering
\small
\providecommand{\dgain}[1]{\textcolor{ForestGreen}{#1}}
\providecommand{\dnull}[1]{\textcolor{yellow!65!black}{#1}}
\begin{tabular}{l cc rrr rrr}
\toprule
 & & & \multicolumn{3}{c}{\textbf{pass@1}} & \multicolumn{3}{c}{\textbf{pass@16}} \\
\cmidrule(lr){4-6}\cmidrule(lr){7-9}
SFT data (\tailsft{} filter) & $n$ & LR & Standard & \tailsft{} & $\Delta$ & Standard & \tailsft{} & $\Delta$ \\
\midrule
BigCode (static $f{=}0.25$)   & $4$ & $2\times10^{-5}$ & $69.57\,\pm\,0.29$ & $73.50\,\pm\,0.70$ & \dgain{$+3.93$} & $75.93\,\pm\,0.00$ & $78.66\,\pm\,0.15$ & \dgain{$+2.73$} \\
Magicoder (static $f{=}0.25$) & $2$ & $1\times10^{-5}$ & $64.08\,\pm\,0.32$ & $66.02\,\pm\,0.55$ & \dgain{$+1.94$} & $75.84\,\pm\,0.31$ & $79.28\,\pm\,0.55$ & \dgain{$+3.44$} \\
Magicoder (static $f{=}0.25$) & $4$ & $2\times10^{-5}$ & $70.52\,\pm\,0.25$ & $73.24\,\pm\,0.09$ & \dgain{$+2.72$} & $78.22\,\pm\,0.61$ & $80.60\,\pm\,0.55$ & \dgain{$+2.38$} \\
OCI (ramp $0{\to}0.5$)        & $4$ & $2\times10^{-5}$ & $74.67\,\pm\,0.08$ & $76.30\,\pm\,0.05$ & \dgain{$+1.62$} & $84.22\,\pm\,0.40$ & $84.04\,\pm\,0.15$ & \dnull{$-0.18$} \\
\bottomrule
\end{tabular}
\caption{Code GRPO (MBPP+, $n{=}378$, $2048$-token generation) from the matched
Standard and \tailsft{} initializations of \cref{tab:sft-main}, at tuned
configurations.  $n$ is the GRPO rollout group size and LR the actor learning rate.
Values are percent, mean $\pm$ standard deviation over three seeds; $\Delta$ is the
\tailsft{} improvement, with yellow marking a change within one standard deviation.}
\label{tab:grpo-grid-code}
\end{table}

\paragraph{Training efficiency}
\Cref{fig:grpo-reward-trajectories} tracks the mean training reward per GRPO step for the Standard and TailSFT initializations. The TailSFT reward rises at least as fast as Standard early in training, and in some settings the early reward increases up to roughly $2.5\times$ faster.


\begin{figure}[t]
\centering
\includegraphics[width=\textwidth]{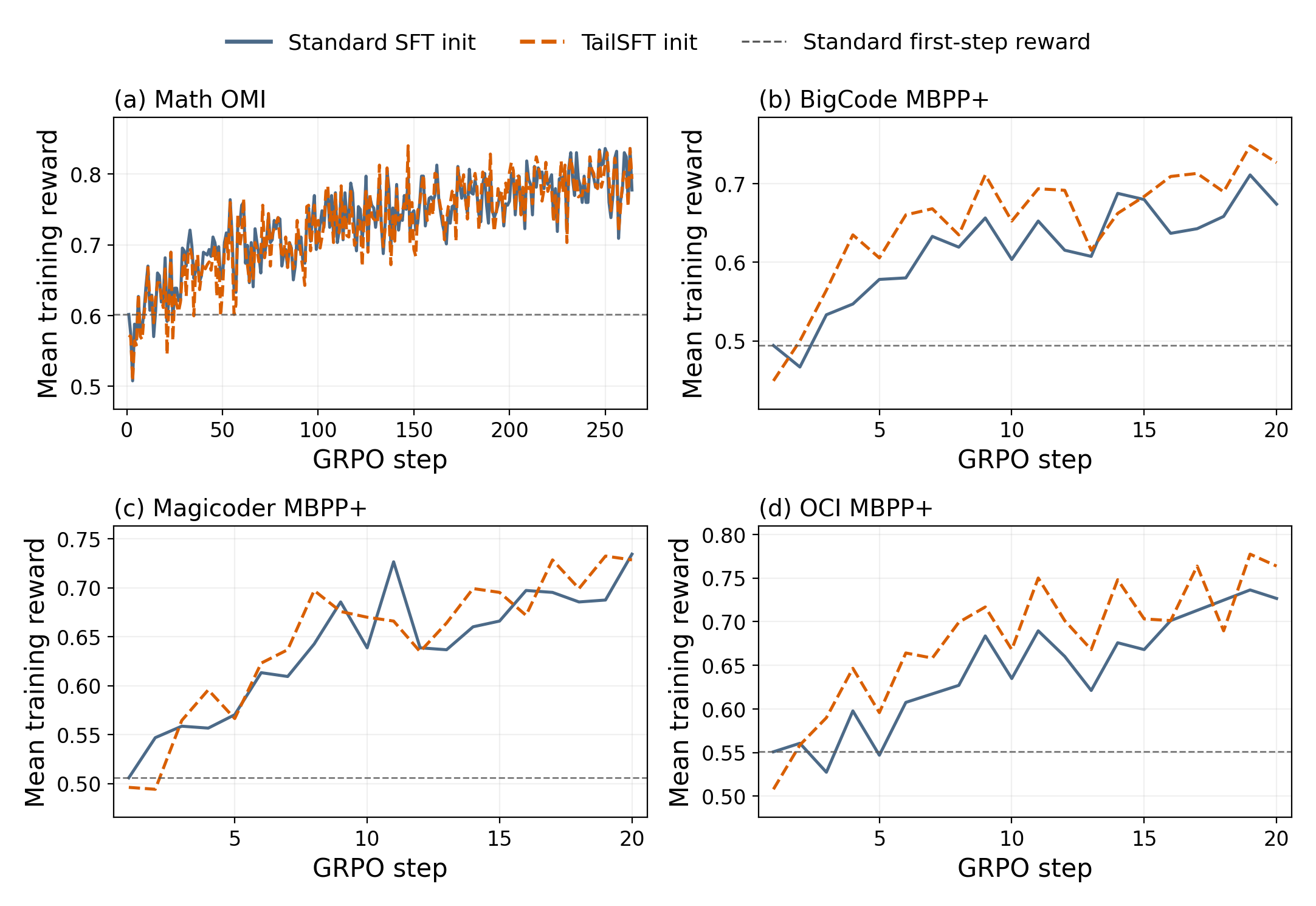}
\caption{\textbf{GRPO training reward from Standard and \tailsft{} initializations.}
Mean training reward (\texttt{critic/rewards/mean}) versus GRPO step for the four
main-table settings.  The horizontal dashed line marks the Standard
initialization's first-step reward.  \tailsft{} begins below Standard  but recovers Standard's starting reward within a small fraction of
training and rises at least as fast thereafter.}
\label{fig:grpo-reward-trajectories}
\end{figure}

\subsection{Coverage ratio diagnostic: computation}
\label{app:coverage-ratio-details}

This subsection records exactly how each point in \Cref{fig:lm-clip-predictor}
was produced from the evaluation runs, filling in the aggregation and
run-selection choices that \Cref{def:coverage-ratio} leaves implicit.  The generation settings (temperature,
top-$p$, generation length, and sample count) for each benchmark are exactly
those documented in \Cref{app:sft-code,app:sft-math}; we do not re-specify them
here.

\paragraph{Per-problem scoring and seed aggregation}
For each problem $i$ and model $\pi$ we draw at least $16$ samples per evaluation
seed, compute the empirical \passat{1} and \passat{16} within each seed,
and average those per-problem values across the available seeds to form
$P_{i,1}(\pi)$ and $P_{i,16}(\pi)$; this is the $P_{i,K}(\pi)$ of
\Cref{def:coverage-ratio}.  Every run uses three
seeds for base, standard, and \tailsft{} on every benchmark.  

\paragraph{Horizontal axis ($\rho_{16}$)}
The coverage ratio is computed exactly as in \Cref{def:coverage-ratio} from
$\pi_0$ and $\pi_{\mathrm{SFT}}$: the base-reachable set
$\mathcal R_0=\{i:0.05<P_{i,16}(\pi_0)<0.95\}$ uses the \emph{empirical}
base \passat{16}, while $L$ and $G$ use the plug-in
$f_{16}\bigl(P_{i,1}(\cdot)\bigr)$ of the empirical \passat{1}.  Standard SFT is
the no-filter ($f{=}0$) run.  The size $\lvert\mathcal R_0\rvert$ of the
base-reachable set varies from $16$ (AIME) to $237$ (CruxEval-O) and is listed
per point in \Cref{tab:coverage-ratio-points}.

\paragraph{Vertical axis (coverage gain $\Delta$)}
The plotted gain is the mean over $\mathcal R_0$ of the empirical \passat{16}
difference $P_{i,16}(\pi_{\mathrm{Tail}})-P_{i,16}(\pi_{\mathrm{SFT}})$, in
percentage points.  The \tailsft{} and standard runs use the same clip settings
surfaced in \Cref{tab:sft-main}: static $f{=}0.25$ for BigCode and Magicoder,
ramp $0\!\to\!0.5$ for OCI, ramp $0\!\to\!0.5$ for AIME, OMEGA-500, and MATH-500 Level~5.

\paragraph{Special cases}
Two points require extra care.  For \emph{MATH-500 Level~5}, the base model was
only evaluated at a $4096$-token generation length, so to keep $\pi_0$ and
$\pi_{\mathrm{SFT}}$ length-matched inside $\rho_{16}$ we compute
$\mathcal R_0$, $L$, and $G$ from a $4096$-token base run and a $4096$-token
no-filter run, restricted to the $134$ Level-5 problems; the plotted gain
$\Delta$ still uses the $8192$-token standard and \tailsft{} runs of
\Cref{tab:sft-main}.  Recomputing $\rho_{16}$ from the $8192$-token no-filter run
instead leaves it essentially unchanged ($2.28$ vs.\ $2.04$, both $>1$), so the
point's placement is robust to this choice.  For \emph{OMEGA-500}, per-problem
\passat{1} and \passat{16} are recomputed from the stored generations with the
same exact-match (flexible) grader used for \Cref{tab:sft-main}, rather than read
from cached metric files, so the diagnostic and the main table score the runs
identically.
Each point is thus a triple $(\rho_{16},\Delta,\lvert\mathcal R_0\rvert)$;
\Cref{tab:coverage-ratio-points} lists all eighteen. 

\begin{table}[t]
\centering
\footnotesize
\begin{tabular}{llrrr}
\toprule
SFT dataset & Benchmark & $\lvert\mathcal R_0\rvert$ & $\rho_{16}$ & $\Delta$ (pp) \\
\midrule
OMI & AIME & 16 & $7.60$ & $+9.92$ \\
 & MATH-500 L5 & 31 & $2.28$ & $+3.23$ \\
 & OMEGA-500 & 140 & $0.90$ & $-0.95$ \\
\midrule
BigCode & CruxEval-I & 199 & $1.63$ & $+3.02$ \\
 & CruxEval-O & 237 & $2.87$ & $+28.69$ \\
 & HumanEval+ & 29 & $1.39$ & $+9.20$ \\
 & LiveCodeBench & 105 & $1.37$ & $+3.81$ \\
 & MBPP+ & 50 & $1.30$ & $+14.00$ \\
\midrule
Magicoder & CruxEval-I & 199 & $2.26$ & $+12.40$ \\
 & CruxEval-O & 237 & $0.73$ & $+18.42$ \\
 & HumanEval+ & 29 & $1.04$ & $-0.00$ \\
 & LiveCodeBench & 105 & $1.22$ & $+1.59$ \\
 & MBPP+ & 50 & $0.77$ & $+11.33$ \\
\midrule
OCI & CruxEval-I & 199 & $1.35$ & $+6.37$ \\
 & CruxEval-O & 237 & $0.55$ & $+9.00$ \\
 & HumanEval+ & 29 & $0.23$ & $-6.32$ \\
 & LiveCodeBench & 105 & $0.97$ & $-1.27$ \\
 & MBPP+ & 50 & $0.54$ & $+2.00$ \\
\bottomrule
\end{tabular}
\caption{Exact values plotted in \Cref{fig:lm-clip-predictor}: the base-reachable
set size $\lvert\mathcal R_0\rvert$, the coverage ratio $\rho_{16}$
(\Cref{def:coverage-ratio}), and the coverage gain $\Delta$ (mean empirical
\passat{16} difference of \tailsft{} over standard SFT on $\mathcal R_0$, in
percentage points), for each (SFT dataset, benchmark) pair.}
\label{tab:coverage-ratio-points}
\end{table}


\end{document}